\documentclass[11pt]{article}
\usepackage[letterpaper, margin=1in]{geometry}

\usepackage[utf8]{inputenc} 
\usepackage[T1]{fontenc}    
\usepackage{hyperref}       
\usepackage{url}            
\usepackage{booktabs}       
\usepackage{amsfonts}       
\usepackage{nicefrac}       
\usepackage{microtype}      
\usepackage{xcolor}         
\usepackage{graphicx}

\usepackage{algorithm}
\usepackage{algorithmic}

\usepackage{amsmath}
\usepackage{amssymb}
\usepackage{mathtools}
\usepackage{amsthm}

\usepackage{times}
\usepackage[capitalize,noabbrev]{cleveref}
\usepackage{natbib}

\usepackage{dsfont}
\usepackage[mathscr]{euscript}
\usepackage{expl3}
\usepackage{xcolor}
\usepackage{colortbl}
\usepackage{thmtools}
\usepackage{thm-restate}
\usepackage{mathtools}
\usepackage{multirow}
\usepackage{wrapfig}
\usepackage{lipsum}
\usepackage{accents}

\usepackage{pifont}

\let\vec\undefined
\usepackage{MnSymbol}
\DeclareMathAlphabet\mathbb{U}{msb}{m}{n}
\usepackage{xpatch}

\def\Rset{\mathbb{R}}

\def\Hset{\mathbb{H}}

\let\Pr\undefined
\let\P\undefined
\DeclareMathOperator*{\Pr}{\mathbb{P}}
\DeclareMathOperator*{\P}{\mathbb{P}}
\DeclareMathOperator*{\E}{\mathbb E}

\DeclareMathOperator*{\argmin}{argmin}

\DeclareMathOperator*{\dom}{dom}

\DeclareMathOperator*{\conv}{conv}
\DeclareMathOperator*{\diag}{diag}

\DeclareMathOperator{\supp}{supp}

\DeclareMathOperator{\Improv}{Improv}
 
\DeclareMathOperator{\Int}{int} 
\DeclareMathOperator{\relint}{relint} 
 
\DeclareMathOperator{\Proj}{Proj} 

\DeclarePairedDelimiter{\abs}{\lvert}{\rvert} 
\DeclarePairedDelimiter{\bracket}{[}{]}
\DeclarePairedDelimiter{\curl}{\{}{\}}
\DeclarePairedDelimiter{\paren}{(}{)}
\DeclarePairedDelimiter{\norm}{\|}{\|}
\DeclarePairedDelimiter{\tri}{\langle}{\rangle}

\ExplSyntaxOn
\tl_const:Nn \c_my_uc_alphabet_tl { ABCDEFGHIJKLMNOPQRSTUVWXYZ }
\tl_const:Nn \c_my_full_alphabet_tl { ABCDEFGHIJKLMNOPQRSTUVWXYZ
  abcdefghijklmnopqrstuvwxyz }

\tl_map_inline:Nn \c_my_uc_alphabet_tl
 { \cs_gset:cpn { c#1 } { \mathcal{#1} } }

\tl_map_inline:Nn \c_my_uc_alphabet_tl
 { \cs_gset:cpn { s#1 } { \mathscr{#1} } }

\tl_map_inline:Nn \c_my_full_alphabet_tl
 {
  \cs_gset:cpn { b#1 } { \mathbf{#1} }
  \cs_gset:cpn { sf#1 } { \mathsf{#1} }
 }
\ExplSyntaxOff

\newcommand{\h}{\widehat}
\newcommand{\ov}{\overline}
\newcommand{\wt}{\widetilde}
\newcommand{\e}{\epsilon}
\newcommand{\ignore}[1]{}

\newlength{\dhatheight}
\newcommand{\doublehat}[1]{%
    \settoheight{\dhatheight}{\ensuremath{\widehat{#1}}}%
    \addtolength{\dhatheight}{-0.25ex}%
    \widehat{\vphantom{\rule{1pt}{\dhatheight}}%
    \smash{\widehat{#1}}}}

\newcommand{\Piall}{\Pi_{\mathrm{all}}}
\newcommand{\Pialldag}{\Piall^\dagger}
\newcommand{\hh}{\doublehat}
\newcommand{\Ccoh}{\sC_{\mathrm{coh}}}
\newcommand{\Ccohdag}{\Ccoh^\dagger}
\newcommand{\Ccohddag}{\Ccoh^{\dagger\dagger}}
\newcommand{\KL}{\sfD_{\mathrm{KL}}}
\newcommand{\JS}{\sfD_{\mathrm{JS}}}
\newcommand{\Hell}{\sfD_{\mathrm{Hell}}}
\newcommand{\orbit}{\mathrm{orbit}}

\newcommand{\proj}{\Pi}

\hypersetup{
  breaklinks   = true, 
  colorlinks   = true, 
  urlcolor     = blue, 
  linkcolor    = blue, 
  citecolor   = blue 
}

\usepackage[toc, page, header]{appendix}
\declaretheorem{theorem}
\newtheorem{lemma}[theorem]{Lemma}

\newtheorem{corollary}[theorem]{Corollary}
\newtheorem{definition}[theorem]{Definition}

\title{Coherence Mechanisms for Provable Self-Improvement}
\author{Mehryar Mohri \and Jon Schneider \and Yifan Wu}
\date{}

\allowdisplaybreaks

\begin{document}

\maketitle

\tableofcontents
\clearpage

\begin{abstract}
Self-improvement is a critical capability for large language models
and other intelligent systems, enabling them to refine their
behavior and internal consistency without external
supervision. Despite its importance, prior approaches largely rely
on empirical heuristics and lack formal guarantees. In this paper,
we propose a principled framework for self-improvement based on the
concept of \emph{coherence}, which requires that a model's outputs
remain consistent under task-preserving transformations of the
input.

We formalize this concept using projection-based mechanisms that
update a baseline model to be coherent while remaining as close as
possible to its original behavior. We provide rigorous theoretical
guarantees that these mechanisms achieve \emph{monotonic
improvement}, measured by a reduction in expected Bregman
divergence. Our analysis is comprehensive, covering both \emph{direct}
and \emph{two-step} projection methods, and robustly extends these
guarantees to non-realizable settings, empirical (finite-sample)
distributions, and relaxed coherence constraints.

Furthermore, we establish a general \emph{characterization
theorem}, showing that any mechanism with similar provable
improvement guarantees must inherently conform to a coherence-based
structure. This culminates in rigidity results under the demand for
universal improvement, establishing coherence as a fundamental and,
in a formal sense, necessary principle for provable self-improvement.
\ignore{  
  While the main focus of the paper is theoretical and algorithmic, we
  also report empirical results validating that coherence-based
  projections yield measurable performance gains in practice. Overall,
  our framework provides a unified, mathematically grounded approach
  for designing self-improving models with provable guarantees.
  }
\end{abstract}

\section{Introduction}

The ability to self-improve is a central and defining challenge in the
design of large language models (LLMs) and other advanced intelligent
systems. Beyond being a desirable feature, self-improvement is
essential for developing models that are internally consistent, robust,
and capable of adapting to complex or evolving tasks. In domains
ranging from robotics to scientific discovery and natural language
processing, the capacity to autonomously enhance performance without
constant, granular human supervision can dramatically improve
reliability, safety, and overall task effectiveness.

Despite its critical importance, the field has largely been guided by
empirical heuristics. This creates a significant gap: without a formal
understanding of \emph{why} a given self-improvement mechanism works,
there is no rigorous guarantee that it will not fail, degrade
performance in subtle ways, or introduce new, unforeseen
inconsistencies. This paper aims to bridge this gap by proposing a
principled, mathematical framework for self-improvement grounded in
the fundamental concept of \emph{coherence}.

Our core thesis is that a model's internal consistency is a powerful
and measurable proxy for its quality. We define coherence as the
requirement that a model produce consistent outputs for inputs that
are equivalent under a task-preserving transformation. For example, a
model should ideally output the same answer distribution for a
question $x$ and its semantic paraphrase $\Phi(x)$. An
\emph{incoherent} model, one that contradicts itself by outputting
different answers for equivalent inputs, is demonstrably
flawed. We posit that the process of actively reducing this internal
incoherence provides a provable, non-heuristic path to model
improvement.

We formalize this intuition by defining the set of all coherent
models, $\Ccoh$. Our self-improvement mechanisms are then formulated
as \emph{projections}: given a baseline model $\pi_0$, we project it
onto this coherent set to find an improved model $\h\pi \in \Ccoh$
that is as close as possible to $\pi_0$. To measure this "closeness,"
we use the geometric language of \emph{Bregman divergences}, a broad
class of measures that includes the KL divergence and squared
Euclidean distance. This geometric approach allows us to provide
formal, non-asymptotic guarantees of monotonic improvement.

This principled, projection-based framework stands in contrast to
many existing approaches that, while often empirically successful, are
guided by different principles. Recent work has explored diverse
strategies, from generating synthetic instruction data
\citep{wang2023self, wangself, chen2024universal} to inference-time
verification like the \emph{sharpening mechanism}
\citep{huang2024sharpening}. Other approaches focus on data
augmentation through paraphrasing \citep{xu2023wizardlm}, iterative
self-feedback loops \citep{madaan2023self}, and adversarial debates
\citep{chen2023playing}. While these studies illustrate a spectrum of
powerful strategies, our work complements this line of research by
providing a rigorous theoretical foundation for a class of mechanisms
with provable guarantees.

This paper provides a comprehensive theoretical analysis of
coherence-based self-improvement. The argument is structured to build
from foundational concepts to our main theoretical results.

\textbf{Foundations (\cref{sec:prelim,sec:coherence}).}
We first establish our mathematical language. \cref{sec:prelim}
introduces the notation and probabilistic setting, and then provides a
detailed review of Bregman divergences and Legendre functions, which
form the core geometric toolkit for our entire analysis.
\cref{sec:coherence} then formalizes our central concept. We introduce
\emph{invariance mappings} ($\Phi$) and use them to define the sets of
coherent models ($\Ccoh$, $\Ccohdag$, $\Ccohddag$), establishing their
key properties (e.g., as convex cones or linear subspaces). This
section culminates in the formal problem formulation: self-improvement
as a constrained optimization problem.

\textbf{Direct Coherence Projection (\cref{sec:direct-projection}).}
This section introduces and analyzes our first mechanism, the
\emph{direct projection} of $\pi_0$ onto the target set
$\Ccoh \cap \Pi$.
\begin{itemize}

\item We begin by establishing our main \emph{improvement guarantee}
  (\cref{sec:direct-improvement-guarantees}), proving via the
  Pythagorean theorem that this projection monotonically decreases the
  divergence to an ideal coherent model $\pi^*$. This uses a key lemma
  that proves several important properties of the expected Bregman
  divergence between conditional probabilities under the input
  distribution (\cref{sec:Bregman-divergence-expectation}).

\item We then extend this central result to more realistic scenarios,
  providing guarantees for the \emph{non-realizable setting} (where
  $\pi^*$ is not perfectly coherent,
  \cref{sec:non-realizable-setting}), the \emph{empirical setting}
  (where projections are computed on a finite sample,
  \cref{sec:empirical-projection}), and for \emph{relaxed constraints}
  (where strict coherence is softened to an inequality,
  \cref{sec:relaxed-case}).

\item The section concludes with a study of guarantees for a convex
  family of Bregman divergence generators. We first analyze a natural
  robust solution based on a minimax optimization and show that in
  general it cannot yield a general improvement
  \cref{sec:failure-minimax}. Next, we analyze universal improvement
  and present impossibility results suggesting that it can only be
  achieved under very restrictive assumptions
  \cref{sec:uniform-improvement-guarantee}.

\end{itemize}

\textbf{Two-Step Coherence Projection (\cref{sec:two-step-projection}).}
This section develops an alternative \emph{two-step} mechanism,
which first projects $\pi_0$ onto the unconstrained coherence set
$\Ccohdag$ and then onto the final target set.
\begin{itemize}

\item We first prove a remarkable \emph{equivalence result}
  (\cref{sec:pyth-equality,sec:single=two}), showing that for a broad
  and common class of Bregman divergences (including Euclidean, KL,
  and all separable or quadratic generators), this two-step projection
  is \emph{identical} to the direct projection.

\item We then show the utility of this alternative view by using it to
  derive a new set of theoretical guarantees. This includes an
  alternative improvement bound (\cref{sec:two-step-improvement}),
  explicit bounds in terms of Hellinger distance for relative entropy
  (\cref{sec:relative-entropy-guarantees}), and a striking
  \emph{maximin property} showing our solution is optimal against the
  worst-case ideal model $\pi^*$ (\cref{sec:maximin-properties}).

\end{itemize}

\textbf{Characterization and Rigidity (\cref{sec:characterization}).}
Finally, this section provides the theoretical capstone of the paper,
proving that our coherence-based framework is not merely one possible
method, but is foundational to provable self-improvement.
\begin{itemize}

\item We first show that \emph{any} mechanism that satisfies a similar
  improvement guarantee must, in effect, be a Bregman projection onto
  a coherence-like set defined by its own level-sets
  (\cref{sec:characterization-single-F}).

\item We then prove a series of powerful \emph{rigidity theorems}
  (\cref{sec:characterization-rigidity}). These show that under the
  strong requirement of universal improvement across \emph{all}
  Legendre divergences in a family of divergences generated by a
  convex family, the mechanism's output structure is forced to be
  rigid: it must be block-constant on a \emph{single, universal
    partition}, independent of the specific divergence. This reveals
  that any such mechanism is inherently aligned with the
  coherence-projection framework we introduce.
  
\end{itemize}

Taken together, these results build a comprehensive theoretical case
for coherence as a principled and provably effective mechanism for
self-improvement.

\ignore{
While the primary focus of this paper is theoretical and algorithmic,
we also provide an experimental section that reports empirical results
validating our framework. These experiments demonstrate that
coherence-based projections translate into measurable performance
gains in practice, confirming the practical relevance of our
theoretical analysis.
}

\subsection{Related Work}
Our work sits at the intersection of several active research areas:
heuristic self-improvement, invariance in machine learning, and the
geometric theory of projections. We situate our contribution relative
to these fields.

\textbf{Heuristic Self-Improvement and Self-Correction.}
A dominant paradigm for improving LLMs without human labels is to
fine-tune them on their own outputs, a process often called
self-instruction \citep{wang2023self, wangself}. This typically
involves a "judge" model (which can be the model itself) to filter,
score, or refine generations \citep{madaan2023self, huang2023large,
  zheng2023judging}. This approach has been formalized as
Reinforcement Learning from AI Feedback (RLAIF), where the model
learns from its own self-generated rewards \citep{yuan2024self}, or
through principled self-critique, as in Constitutional AI
\citep{bai2022constitutional}.

While empirically powerful, these methods rely on the model's own,
often biased \citep{xu2024pride}, assessment of "quality" or
"harmlessness." This can lead to instability, "model collapse"
\citep{hataya2023will, martinez2023combining, bertrandstability,
  shumailov2024ai}, or reinforcement of the model's own blind
spots. Our framework differs fundamentally. Instead of relying on a
subjective "quality" score, we use the objective, formal criterion of
\emph{coherence} as the improvement signal. We do not ask the model to
\emph{judge} its outputs, but rather enforce that its outputs be
\emph{mathematically consistent} under transformation, providing a
more robust, non-heuristic foundation for improvement.

\textbf{Inference-Time Mechanisms and Self-Consistency.}
A separate line of research modifies the decoding process at inference
time. This includes generating candidate sets via sampling
\citep{ji2023language, finlaysonclosing} or search
\citep{zhangplanning}, and then re-weighting or selecting outputs
based on model-internal signals \citep{fu2025deep, chuang2024dola} or
external constraints \citep{lu2022neurologic}.

Perhaps the most relevant work in this area is \emph{self-consistency}
\citep{wang2023selfconsistency}, which improves reasoning by sampling
multiple outputs from varied prompts (e.g., chain-of-thought variants
\citep{wei2022chainofthought}) and taking a majority
vote. Self-consistency \emph{leverages} a model's incoherence, that
is, its tendency to give different answers to equivalent prompts, to
find a more robust aggregate answer. Our work is the conceptual
inverse: we seek to \emph{eliminate} this incoherence from the model's
parameters directly, not merely average it out at inference time. Our
coherence projection can be seen as a training-time mechanism to "bake
in" self-consistency, yielding a model that is inherently more
reliable. The theoretical work of \citet{huang2024sharpening} on
tuning a model to learn its best-of-N (BoN) generations is also
related, though our focus is on enforcing coherence across model
outputs, rather than amplifying a model-defined "best" output.

\textbf{Invariance, Data Augmentation, and Self-Distillation.}
The use of label-preserving transformations, or \emph{data
  augmentation}, is a cornerstone of modern machine learning,
especially in vision \citep{krizhevsky2012imagenet,
  cubuk2019autoaugment, cubuk2020randaugment, chen2020simple,
  devries2017improved}. In NLP, this includes paraphrasing,
back-translation, and token-level perturbations \citep{wei2019eda,
  sennrich2016improving, wang2018switchout, xie2020unsupervised,
  yangsynthetic}. These methods all heuristically enforce invariance
by showing the model examples of inputs ($x$) and transformations
($\Phi(x)$) that share a label. Our work provides a formal, geometric
basis for this intuition. We frame the goal not as "generating more
data", but as a principled \emph{projection} onto the set of coherent
models $\Ccoh$. The work of \citet{xie2020unsupervised}, which adds an
unsupervised coherence loss, and \citet{maini2024rephrasing}, which
studies pre-training with paraphrases, are particularly close
practical realizations of this concept, for which we now provide a
general theoretical guarantee.

This perspective also connects to \emph{self-distillation}
\citep{hinton2015distilling, sanh2019distilbert, gerstgrassermodel,
  kim2025pre, guminillm, agarwal2024policy}. The theory of
\citet{allentowards} suggests that gains from self-distillation can
arise from an "implicit ensemble" effect. Our coherence projection
offers a precise formulation of this: it finds a single model $\h\pi$
that is the Bregman centroid of the "ensemble" of incoherent views
(e.g., $\pi_0(x)$ and $\pi_0(\Phi(x))$), thereby distilling their
shared knowledge into one consistent model.

\textbf{Bregman Projections in Machine Learning.}
The mathematical tools we use are grounded in a rich history of
Bregman divergences in machine learning. Bregman divergences are
central to maximum entropy models
\citep{CollinsSchapireSinger2002,DudikPhillipsSchapire2007,AltunSmola2006,
  BanerjeeDhillonGhoshMerugu2007,
  FrongilloReid2013,Harremoes2016,NockMenonOng2016,
  CortesKuznetsovMohriSyed2015, MohriRostamizadehTalwalkar2018},
boosting algorithms
\citep{CollinsSchapireSinger2002,KivinenWarmuth1999,
  WarmuthGlocerVishwanathan2008, TsudaRatschWarmuth2005,
  Lafferty1999,Murata2004,LiuVemuri2011, Rezaei2021}, clustering
\citep{banerjee2005clustering}, and Bayesian estimation
\citep{frigyik2008functional}. Our contribution is to apply this
powerful geometric tool to the problem of model self-improvement,
using projections onto coherence-defined sets as the core mechanism
for provably enhancing model consistency and performance.

\section{Preliminaries}
\label{sec:prelim}

This section introduces notation, probabilistic setting, and key
properties of Bregman divergences, which will be used throughout.

\subsection{Notation and probabilistic setting}

Let $\Sigma$ be a finite alphabet (set of tokens). Prompts are
sequences $x \in \sX \subseteq \Sigma^*$, and outputs lie in
$\sY \subseteq \Sigma^*$. A language model is a conditional
probability distribution
\[
\pi\colon \sX \to \Delta(\sY),
\]
where $\Delta(\sY)$ denotes the probability simplex over $\sY$.
For any prompt $x$, we write $\pi(x)$ for the corresponding
distribution over outputs.

To simplify exposition, we assume throughout that $\sX$ and $\sY$
are finite subsets of $\Sigma^*$. The arguments extend to the common
setting where $\sX$ is countable and $\sY$ is finite, and further
to countable $\sY$ under standard compactness assumptions on the
model class.

We write $\Piall = \big(\Delta(\sY)\big)^{\sX}$ for the set of all
conditional distribution models. Since $\sY$ is finite, $\Delta(\sY)$
is a compact convex subset of a finite-dimensional Euclidean
space. When $\sX$ is finite, $\Piall$ is therefore a finite product of
compact sets (where all standard topologies such as pointwise,
uniform, etc., are equivalent), and hence a compact metric space. For
countable $\sX$, Tychonoff’s theorem guarantees that $\Piall$ remains
compact under the product topology (equivalently, the topology of
pointwise convergence).

We also define the extended space of unnormalized models
$\Pialldag = \big(\Rset_{+}^{\sY}\big)^{\sX}$,
where $\Rset_{+}^{\sY} \cong \Rset_+^d$ with $d = |\sY|$. When
$n = |\sX|$ is finite, this gives
$\Pialldag \cong (\Rset_+^d)^n = \Rset_+^{n \times d}$, a closed
and convex cone in a finite-dimensional vector space. Since it is
unbounded, $\Pialldag$ is not compact.

Throughout, we work with a closed and convex family of models
$\Pi \subseteq \Piall$. Compactness of $\Piall$ implies that any
closed subset $\Pi$ is also compact.

Finally, expectations are taken with respect to a fixed prompt
distribution $\sD_{\sX}$ over $\sX$. We will assume, without loss
of generality that $\supp(\sD_{\sX}) = \sX$.

\subsection{Bregman divergences and Legendre functions}

In this section, we review the definitions and key properties of
Bregman divergences and Legendre-type functions. These tools provide
the foundational geometry for our self-improvement mechanisms, in
particular for computing centroids, projections, and bounding errors
via Fenchel-Bregman inequalities. We focus on the differentiable case
corresponding to Legendre-type generators, while noting the general
subgradient-based definitions for completeness.

\subsubsection{Bregman divergences}

We adopt the standard definition of Bregman divergence following
\citet{Rockafellar1996} (see also \citep{Bregman1967,CensorLent1982}).
\begin{definition}[Bregman divergences]
Let $F\colon \sK \subseteq \Rset^d \to \Rset$ be a convex function
defined on a convex set $\sK$. The \emph{Bregman divergence}
$\sfB_F(\sfp, \sfq)$ is defined for any point $\sfp \in \sK$ and
$\sfq$ in the interior of $\dom(F)$ where $F$ admits a sub-gradient,
as the gap between $F(\sfp)$ and the first-order approximation of $F$
at $\sfq$, evaluated at $\sfp$:
\[
\sfB_F(\sfp \parallel \sfq)
= F(\sfp) - F(\sfq) - \tri*{g_\sfq, \sfp - \sfq},
\]
where $g_\sfq$ is any element of $\partial F(\sfq)$, the
sub-differential of $F$ at $\sfq$.
\end{definition}
This definition can be further extended, for example by considering
directional derivatives instead of sub-gradients.
Throughout most of this work, we restrict attention to the common case
where $F$ is differentiable on the interior of its domain. Then,
$\partial F(q) = \curl*{\nabla F(q)}$ and the divergence takes the
familiar simpler form:
\[
\sfB_F(\sfp \parallel \sfq)
= F(\sfp) - F(\sfq) - \tri*{\nabla F(\sfq), \sfp - \sfq}.
\]

Our primary setting is that of probability distributions over a finite
label set $\sY$, represented by elements of the probability simplex
$\Delta(\sY)$ (or $\Delta_d$ where $d = |\sY|$). These distributions
arise as outputs of conditional models
$\pi\colon \sX \to \Delta(\sY)$.
To ensure Bregman divergences are well-defined for such outputs, we
only need to make a minimal set of assumptions: We assume the domain
of $F$, $\sK$, contains the probability simplex $\Delta(\sY)$, and
that its topological interior, $\Int(\sK)$, contains the relative
interior of the simplex, $\relint(\Delta(\sY))$. A simple domain that
satisfies this is the unit hypercube $\sK = [0, 1]^d$, since the
simplex $\Delta(\sY)$ is fully contained within $\sK$. The relative
interior of the simplex (where $\sfq_i > 0$ and $\sum \sfq_i = 1$) is
contained within the topological interior of the hypercube,
$\text{int}(\sK) = (0, 1)^d$. Under these minimal conditions,
$\sfB_F(\sfp \parallel \sfq)$ is well-defined for any
$\sfp \in \Delta(\sY)$ and any $\sfq$ in the relative interior of the
simplex.

Under these standard assumptions, Bregman divergences admit several
well-known properties (see \citet{Rockafellar1996} or Appendix~E of
\citep{MohriRostamizadehTalwalkar2018}, and \citep{Nielsen2020} for an
introductory text):

\begin{itemize}
\item Non-Negativity: $B_F \geq 0$.

\item Convexity in the first argument: For a fixed $\sfq$, the
  function $\sfp \mapsto B_F(\sfp \parallel \sfq)$ is convex in
  $\sfp$.

\item Asymmetry: In general,
  $B_F(\sfp \parallel \sfq) \neq B_F(\sfq \parallel \sfp)$.

\item Three-Point Identity (or triangular-type identity): For any
  $\sfp \in \sK$ and any $\sfq, \sfr \in \text{int}(\sK) = \Omega$,
  the following identity holds:
\[
\sfB_F(\sfp \parallel \sfr) + \sfB_F(\sfr \parallel \sfq)
= \sfB_F(\sfp \parallel \sfq) + \tri{\nabla F(\sfq) - \nabla F(\sfr), \sfp - \sfr}.
\]

\item Projection: Let $\sC \subseteq \sK $ be a non-empty, closed, and
  convex set, then the projection
  $\Proj_{\sC}^F(\sfq) = \argmin_{\sfp \in \sC} \sfB(\sfp \parallel
  \sfq)$ is well defined for any $\sfq$ allowed in the definition of
  $\sfB_F$.  In our setting, where $\sC$ is a subset of $\Piall$,
  $\sC$ is a compact set as a closed subset of $\Piall$ and continuity
  in $\sfp$ ensures existence of a minimizer.
      
\item Generalized Pythagorean Theorem: Let $\sC$ be a non-empty,
  closed, and convex set, and let $\ov \sfq = \Proj_{\sC}^F(\sfq)$. Then,
  for any point $\sfp \in \sC$, the following inequality holds:
  \[
    \sfB_F(\sfp \parallel \sfq) \geq \sfB_F(\sfp \parallel \ov \sfq)
    + \sfB_F(\ov \sfq \parallel \sfq).
  \]
  The theorem holds with equality if the set $\sC$ is an affine space.

\item Positive linearity in the generator: if $F$ and $G$ are two
  convex functions and $\alpha, \beta > 0$, then:
  $\sfB_{\alpha F + \beta G} = \alpha \sfB_F + \beta \sfB_G$.

\end{itemize}

\subsubsection{Legendre functions}

We will often consider the special case where our generating function
$F \colon \sK \to \Rset$ is of \emph{Legendre type} (or simply
\emph{Legendre}). The following definition is based on \citep{Rockafellar1996}.

\begin{definition}[Functions of Legendre Type]
  A proper, convex, and closed function $F \colon \sK \to \Rset$ is of
  \emph{Legendre type} if its behavior on its open interior
  $\Omega = \text{int}(\sK)$ is as follows:
\begin{enumerate}
\item $F$ is differentiable on $\Omega$.
  
\item The gradient mapping
  $\nabla F \colon \Omega \to \Int(\dom(F^*))$ is a bijection.
\end{enumerate}
\end{definition}
The bijection property is equivalent to $F$ being strictly convex on
$\Omega$ and \emph{steep} (that is for any sequence
$\sfp_k \in \Omega$ converging to a boundary point
$\sfp_{\partial} \in \partial \sK$, the norm of the gradient explodes:
$\|\nabla F(\sfp_k)\| \to +\infty$), which ensures $\nabla F$ is
surjective.

When $F$ is of Legendre type, its Fenchel conjugate $F^*$ is also of
Legendre type, and a number of standard properties hold:
\begin{itemize}
\item Uniqueness of Projection: For any non-empty, closed, and convex
  set $\sC \subseteq \sK$, the Bregman projection of a point
  $\sfq \in \Omega$ onto $\sC$, $\Proj_{\sC}^F(\sfq)$, exists and is
  unique.
    
    \item Fenchel-Young Equality: For any $\sfp \in \Omega$,
    \[
      F(\sfp) + F^*(\nabla F(\sfp)) = \tri{\sfp, \nabla F(\sfp)}.
    \]
    
    \item Divergence Duality: For any $\sfp, \sfq \in \Omega$,
    \[
      \sfB_F(\sfp \parallel \sfq)
      = \sfB_{F^*}\paren*{\nabla F(\sfq) \parallel \nabla F(\sfp)}.
    \]
\end{itemize}
The negative entropy function and the norm squared functions are
examples of Legendre functions generating the (unnormalized) $\KL$
and the Euclidean distance squared Bregman divergences.

\subsubsection{Centroid properties}

Bregman divergences admit closed-form expressions for centroids and
projections, which are essential for our self-improvement mechanisms,
in particular for our two-step projection method. These results allow
us to characterize the minimizers of weighted sums of divergences.

\begin{restatable}{lemma}{BregmanIdentity}
\label{lemma:bregman-identity}
Let $\lambda \in \Delta_p$ and $\sfp_1, \ldots, \sfp_p \in \sK$, and
$\sfq \in \Int(\dom(F))$. Then, the following holds:
\[
\sum_{k=1}^p \lambda_k \sfB_F(\sfp_k \parallel \sfq)
- \sfB_F\paren*{\sum_{k = 1}^p \lambda_k \sfp_k \parallel \sfq}
= \sum_{k=1}^p \lambda_k F(\sfp_k) - F\paren*{\sum_{k=1}^p \lambda_k \sfp_k}.
\]
\end{restatable}
This identity shows that the Bregman divergence from a weighted
combination of points to a reference $\sfq$ decomposes in a way
independent of $\sfq$. This property underlies the computation of
centroids and justifies that Bregman projections can be expressed in
terms of gradient averages.

\begin{restatable}{lemma}{RightBregmanCentroid}
\label{lemma:right-bregman-centroid}
Let $F$ be Legendre and $\sfq_1, \ldots, \sfq_p \in \Omega$ with
weights $\lambda \in \Delta_p$. Then, the minimizer of
$\sfp \mapsto \sum_{k = 1}^p \lambda_k \sfB_F(\sfp \parallel \sfq_k)$
over $\Omega$ satisfies
\[
  \nabla F(\sfp^*) = \sum_{k = 1}^p \lambda_k \nabla F(\sfq_k),
  \quad \text{or equivalently }
  \sfp^* = (\nabla F)^{-1}\paren*{\sum_{k = 1}^p \lambda_k \nabla F(\sfq_k)}.
\]
If $\sC \subseteq \sK$ is closed and convex, the unique minimizer
over $\sC$ is the Bregman projection of $\sfp^*$ onto $\sC$:
\[
\sfp^*_\sC = \argmin_{\sfp \in \sC} \sfB_F(\sfp \parallel \sfp^*).
\]
\end{restatable}
This lemma provides an explicit formula for the Bregman centroid and
its projection onto convex sets, which is fundamental for constructing
improved models that satisfy coherence or other convex constraints.

\begin{restatable}{corollary}{MinBregmanDiv}
\label{corollary: min bregman div}
Let $\sfp^*$ and $\sfp^*_\sC$ be as above. Then,
\[
\min_{\sfp \in K} \sum_{k = 1}^p \lambda_k \sfB_F(\sfp \parallel \sfq_k)
= \sum_{k = 1}^p \lambda_k F^*(\nabla F(\sfp_k))
- F^*\paren*{\sum_{k = 1}^p \lambda_k \nabla F(\sfq_k)}
+ \sfB_F(\sfp^*_\sC \parallel \sfp^*).
\]
\end{restatable}
This corollary provides a decomposition of the minimum Bregman
divergence over a convex set into a centroid term and a projection
term, which is particularly useful for establishing theoretical
guarantees of projection-based self-improvement, especially in the
context of our two-step projection method.

\subsubsection{Fenchel-Bregman inequality}

This general inequality bounds the inner product of differences by
Bregman divergences and is widely used to control error terms in
proofs.

\begin{restatable}{lemma}{BregmanFenchel}
\label{lemma:BF}
For any Legendre function $F$ and $u, v \in \Omega = \Int(\dom(F))$
with $\alpha \in \Int(\dom(F^*)) = \nabla F(\Omega)$ and
$\beta = \nabla F(v)$, the following inequality holds:
\[
  \tri{u - v, \alpha - \beta} \leq \sfB_F(u \parallel v)
  + \sfB_{F^*}(\alpha \parallel \beta).
\]
\end{restatable}
This inequality is a key tool for bounding inner-product terms in
analysis of projections and improvement mechanisms. It generalizes the
familiar Cauchy-Schwarz-type control using the geometry of Bregman
divergences. The proofs for all these results are given in
Appendix~\ref{app:appendix-bregman-proofs}.

\section{Coherence as a Principle of Self-Improvement}
\label{sec:coherence}

We now introduce the key conceptual framework underlying our
self-improvement methodology. At its core is the principle of
\emph{coherence}, which formalizes the intuition that models should
produce consistent outputs for inputs that are equivalent under
task-preserving transformations.

\subsection{Invariance mappings}

Let $\Phi \colon \sX \to \sX$ be a mapping on inputs (e.g.,
prompts). We call $\Phi$ an \emph{invariance mapping} if, under an
ideal model $\pi^*$, outputs are preserved:
\[
\pi^*(x) = \pi^*(\Phi(x)) \quad \text{for all } x \in \sX,
\]
where equality is interpreted under a task-preserving equivalence
relation (syntactic, semantic, answer-based, or distributional).

For example, if $x$ is a well-formed English sentence, its passive
form $x'$ should lead to the same model output. Semantically
equivalent paraphrases are similarly expected to produce consistent
outputs. This idea extends beyond text: two images differing only by a
90-degree rotation typically depict the same object and should yield
similar outputs in a vision model. Such examples illustrate the
generality of invariance mappings as a way to capture task-preserving
symmetries.

Conceptually, these mappings are close to Zellig Harris's
\emph{transformations} \citep{Harris1968} in his mathematical theory
of language, in that $\Phi$ is a formal operation on linguistic
objects that preserves meaning. This connection emphasizes that
invariance mappings are not arbitrary but grounded in a formal
structure of equivalence.

Throughout this paper, we assume that $\Phi$ is an \emph{involution}, i.e.,
\[
\Phi(\Phi(x)) = x \quad \forall x \in \sX,
\]
which simplifies notation and analysis. More generally, our results
extend to the case where $\Phi$ has a \emph{finite orbit}, i.e., there
exists $k \in \mathbb{N}$ such that $\Phi^k(x) = x$ for each
$x \in \sX$. An even broader formulation arises by partitioning $\sX$
into equivalence classes $\sX_i$, $i \in I$, such that all inputs in
the same class are treated as equivalent. In this view, invariance
requires consistency of model outputs within each class. Beyond
textual data, equivalence classes can represent semantically or
functionally equivalent inputs in images, audio, or structured data,
providing a unifying framework across modalities.

Invariance mappings are not limited to passive-active
transformations. They can also capture other structural or semantic
transformations such as negation, interrogative forms, stylistic
variations, or domain-specific invariances, making them a flexible
foundation for enforcing coherence.

\subsection{Coherent models}

A model $\pi$ is said to be \emph{coherent} if it satisfies
\[
  \pi(x) = \pi(\Phi(x)) \quad \forall x \in \sX, \
  \forall \Phi \text{ invariance mapping.}
\]
Coherence formalizes a natural desideratum for models: inputs that are
semantically or structurally equivalent should yield consistent
outputs. By enforcing coherence, we can improve model reliability,
robustness, and internal consistency, which are essential for
self-improvement mechanisms.

We fix an invariance mapping $\Phi$ throughout.
We first define the family of all coherent functions mapping
from $\sX$ to $\Rset^d$
\[
  \Ccohddag = \curl*{\pi \colon \sX \to \Rset^d
    \colon \pi(x) = \pi(\Phi(x))
  \text{ for all } x \in \sX}.
\]
Equivalently, $\Ccohddag$ consists of functions that are constant on
the orbits of $\Phi$. Each equality constraint $\pi(x) = \pi(\Phi(x))$
is linear, hence $\Ccohddag$ is a closed linear subspace of
$(\Rset^d)^{\sX}$, and therefore convex.
The set of coherent (unnormalized) models is given by
\[
\Ccohdag = \Ccohddag \cap \Pialldag.
\]
Since $\Pialldag$ is a closed convex cone and $\Ccohddag$ is a closed
linear subspace, $\Ccohdag$ is a closed convex cone. It is not an
affine set unless $\Phi$ is the identity mapping, in which case all
models are coherent.
Finally, we define the family of coherent conditional probability
models as
\[
  \Ccoh = \Ccohdag \cap \Piall.
\]
Since $\Piall$ is compact and $\Ccoh$ is a closed subset of $\Piall$,
$\Ccoh$ is compact. Convexity follows since both $\Piall$ and
$\Ccohdag$ are convex.

\subsection{Problem formulation}

Let $\pi_0 \in \Piall$ be a base model and $\Phi \colon \sX \to \sX$
an invariance mapping. Our goal is to construct a model $\pi \in \Pi$
that simultaneously satisfies two desiderata:

\begin{enumerate}
\item \textbf{Closeness to the base model:} $\pi$ should remain close
  to $\pi_0$ under a chosen Bregman divergence $\sfB_F$, ensuring that
  improvements preserve the original model's knowledge.

\item \textbf{Coherence:} $\pi$ should be (approximately) invariant
  under $\Phi$, that is,
\[
\pi(x) \approx \pi(\Phi(x)) \quad \text{for all } x \in \sX.
\]
\end{enumerate}

Formally, this can be expressed as a constrained optimization problem:
\[
  \min_{\pi \in \Pi} \sfB_F(\pi \parallel \pi_0)
  \quad \text{s.t.}
  \quad \pi(x) = \pi(\Phi(x)) \text{ (or a relaxed version) for all } x.
\]

This formulation provides the foundation for our analysis of
projection-based improvement mechanisms. By minimizing a Bregman
divergence while enforcing (or relaxing) coherence constraints, we
obtain models that are guaranteed to improve in a principled manner
while respecting the structure of the original model and task-specific
invariances.

\section{Direct Coherence Projection}
\label{sec:direct-projection}

In this section, we study Bregman projections as a mechanism for
enforcing coherence in conditional models. We begin with a key lemma
that characterizes the expected Bregman divergence between conditional
probabilities under the input distribution
(Subsection~\ref{sec:Bregman-divergence-expectation}). Building on
this, we establish general improvement guarantees for Bregman
projections onto coherence sets and, more generally, onto arbitrary
convex sets, under the assumption that the reference distribution
$\pi^*$ is coherent
(Subsection~\ref{sec:direct-improvement-guarantees}). We then relax
this assumption and analyze the setting where $\pi^*$ may not be
coherent but lies close to the coherent set
(Subsection~\ref{sec:non-realizable-setting}).

In practical applications, projection algorithms operate on empirical
distributions. Accordingly, we provide guarantees for empirical
Bregman projections, where the divergence is estimated from finite
data (Subsection~\ref{sec:empirical-projection}). We give an extensive
discussion of these results and present refined guarantees under
strong-convexity assumptions. We also extend the analysis to
\emph{approximate coherence}, in which the coherence constraint is
softened and incorporated into the objective
(Subsection~\ref{sec:relaxed-case}). This setting encompasses
unconstrained optimization formulations with a relaxed notion of
coherence.

Finally, we study projections with respect to a \emph{family} of
Bregman divergences. We first show that a natural robust solution
based on a minimax objective, where the maximum is taken over a convex
set of Legendre generators, does not in general guarantee improvement
(Subsection~\ref{sec:failure-minimax}). We then consider the stronger
requirement of finding a coherent Bregman projection that is
simultaneously valid for the entire family of divergences. We provide
several impossibility results for this setting, showing that a
universal improvement model can exist only under restrictive
conditions (Subsection~\ref{sec:uniform-improvement-guarantee}).

Throughout, we sometimes work with a general closed convex set $\sC$
(not necessarily the coherent set $\Ccoh$). The intersection
$\ov \Pi = \sC \cap \Pi$, which we assume non-empty, is then also
closed and convex, since $\Pi$ itself is closed and convex.

For most of our analysis, we assume that the symmetry operator $\Phi$
is an involution. Many results, however, extend directly to the more
general case where $\Phi$ admits arbitrary finite orbits.

\subsection{Bregman divergence expectation}
\label{sec:Bregman-divergence-expectation}

Our analysis is often based on the expectation over $\sX$ of a Bregman
divergence. Here, we prove some key properties of that
expectation, which will play an important role in several of our
proofs and analyses.

Let $L^2(\sD_\sX; \Omega) = \curl*{\pi \colon \sX \to \Omega \text{
    measurable}, \ \E\bracket*{\norm*{\pi(x)}^2} < \infty}$ be the
space of measurable functions from $\sX$ to $\Omega$.  We equip this
space with the $L^2$ inner product
$\tri*{\pi, \pi'} = \E_{x \sim \sD_\sX}[\tri*{\pi(x), \pi'(x)}]$.

\begin{lemma}
\label{lemma:Bregman-expectation}
Let $F\colon \Omega \to \Rset$ be convex and differentiable, and
define $\sfF(\pi) = \E_{x \sim \sD_\sX}[F(\pi(x))]$,
$\pi \in L^2(\sD_\sX; \Omega)$.
Then $\sfF$ is convex, and its Bregman divergence satisfies
$\sfB_\sfF(\pi \parallel \pi') = \E_{x \sim \sD_\sX}[\sfB_F(\pi(x)
\parallel \pi'(x))]$.
If further $F$ is a Legendre function and bounded below, then $\sfF$ is
also a Legendre function on $L^2(\sD_\sX; \Omega)$, with gradient
$(\nabla \sfF(\pi))(x) = \nabla F(\pi(x))$.
\end{lemma}

\begin{proof}
Convexity of $\sfF$ follows immediately from the convexity of $F$.  
The gradient of $\sfF$ in $L^2(\sD_\sX; \Omega)$ is pointwise:
\[
(\nabla \sfF(\pi))(x) = \nabla F(\pi(x)), \quad x \in \sX,
\]
which gives the Bregman divergence formula
\[
  \sfB_\sfF(\pi \parallel \pi')
  = \sfF(\pi) - \sfF(\pi') - \tri*{\nabla \sfF(\pi'), \pi - \pi'}
= \E_{x \sim \sD_\sX}[\sfB_F(\pi(x) \parallel \pi'(x))].
\]
Now, assume $F$ is Legendre and bounded below. Since $F$ is proper,
closed, and differentiable with bijective gradient, consider the
constant function $\pi_0(x) = p_0 \in \dom(F)$ with
$F(p_0) < +\infty$. Then $\sfF(\pi_0) = F(p_0) < +\infty$, so $\sfF$ is
proper.

Let $(\pi_n)_{n \ge 0}$ be a sequence in $L^2(\sD_\sX; \Omega)$
converging pointwise almost surely to $\pi$. Let $m$ be a finite lower
bound of $F$. Applying Fatou's lemma to $F - m \ge 0$ gives
\[
  \sfF(\pi)
  = \E[\liminf_{n \to \infty} F(\pi_n(x))]
  \leq \liminf_{n \to \infty} \E[F(\pi_n(x))]
  = \liminf_{n \to \infty} \sfF(\pi_n),
\]
so $\sfF$ is lower semicontinuous (closed).

Finally, since $\nabla F$ is bijective, the pointwise mapping
$(\nabla \sfF)(\pi)(x) = \nabla F(\pi(x))$ is a bijection on the space
of functions $\pi \colon \sX \to \Omega$ in $L^2$: for any function
$\theta \colon \sX \to \nabla F(\Omega)$ in $L^2$, there exists a
unique $\pi$ with $(\nabla \sfF(\pi))(x) = \theta(x)$ for all $x$.
Thus, $\sfF$ is proper, closed, differentiable, and has bijective
gradient, and hence is Legendre on $L^2(\sD_\sX; \Omega)$.
\end{proof}
The extension of this lemma to a countably infinite $\sX$ requires a
more detailed functional analysis, as $L^2(\sD_\sX; \Omega)$ becomes
an infinite-dimensional Hilbert space. The proofs for convexity and
the form of the gradient and divergence remain valid. However, two key
properties of Legendre functions require significant additional
justification.
(1) Closed property: The proof that $\sfF$ is closed (lower semicontinuous)
must be formalized for the $L^2$ topology. The provided proof relies
on pointwise convergence. To show $\sfF$ is closed in $L^2$, one must
show that for any sequence $\pi_n \to \pi$ in $L^2$,
$\sfF(\pi) \le \liminf \sfF(\pi_n)$. This can be done by taking a
subsequence $\pi_{n_k}$ that converges pointwise almost everywhere and
then applying Fatou's lemma.
(2) Bijective gradient: This is the most critical extension. The
current proof's pointwise argument is insufficient. One must prove
that the gradient operator $\nabla \sfF$ is a bijection between the
Hilbert space $L^2(\sD_\sX; \Omega)$ and its dual space. This involves
imposing growth conditions on $F$ (and its conjugate $F^*$) to ensure
that $\pi \in L^2$ implies the function $x \mapsto \nabla F(\pi(x))$
is in the dual space, and conversely, that the inverse map
$(\nabla F)^{-1}$ also maps the dual space back to $L^2$. This is
often established using Rockafellar-type results on integral
functionals.

\subsection{Improvement guarantees}
\label{sec:direct-improvement-guarantees}

The following result shows that Bregman projection with a closed
convex set $\sC$, in particular a coherence set $\Ccoh$,
improves the baseline, when the optimal policy $\pi^*$ is in
$(\sC \cap \Pi)$.

\begin{theorem}[Convex set Bregman–projection improves the baseline]
\label{th:direct-bregman-projection}
Let $\sC$ be a closed convex set. Assume that $\sC \cap \Pi$ contains
the optimal conditional distribution function $\pi^*
\colon \sX \to \Delta(\sY)$.
Then, a solution $\h \pi$ of
the following convex optimization problem:
\begin{align}
\label{opt:direct-bregman-projection}
  \min_{\pi \in \Pi} \quad & \E_{x \sim \sD_\sX}
                       \bracket*{\sfB_F\paren*{\pi(x) \parallel \pi_0(x)}} \\
  \text{subject to} \quad & \pi \in \sC \cap \Pi, \nonumber
\end{align}
satisfies:
\begin{align*}
\E_{x \sim \sD_\sX} \bracket*{\sfB_F \paren*{\pi^*(x) \parallel \h \pi(x)}}
& \leq
\E_{x \sim \sD_\sX} \bracket*{\sfB_F\paren*{\pi^*(x) \parallel \pi_0(x)}}
- \E_{x \sim \sD_\sX}[\sfB_F(\h \pi(x) \parallel \pi_0(x))]\\
& \leq
\E_{x \sim \sD_\sX} \bracket*{\sfB_F\paren*{\pi^*(x) \parallel \pi_0(x)}}.
\end{align*}
Assume that $\Phi$ is an involution and that $F$ is $\mu$-strongly
convex with respect to the norm $\norm{\cdot}$. Then, the following
holds:
\[
\E_{x \sim \sD_\sX} \bracket*{\sfB_F \paren*{\pi^*(x) \parallel \h \pi(x)}}
 \leq
\E_{x \sim \sD_\sX} \bracket*{\sfB_F\paren*{\pi^*(x) \parallel \pi_0(x)}}
- \frac{\mu}{2} \E_{x \sim \sD_X} \bracket*{ \lambda(x)(1 - \lambda(x))
  \norm*{\pi_0(x) - \pi_0(\Phi(x))}^2 },
\]
where $\lambda(x) = \frac{\P(x)}{\P(x) + \P(\Phi(x))}$.
\end{theorem}

\begin{proof}
  By Lemma~\ref{lemma:Bregman-expectation}, the expectation of the
  Bregman divergence $\sfB_F$, denoted $\sfB_\sfF$, is a Bregman
  divergence associated to the expectation of $F$,
  $\sfF(\pi) = \E_{x \sim \sD_\sX}[F(\pi(x))]$.  Thus, by definition
  of Bregman divergence, $\h \pi$ is the $\sfF$-Bregman projection of
  $\pi_0$ on $\sC \cap \Pi$.  By the Pythagorean theorem for Bregman
  divergences, we can write for $\pi^* \in \sC \cap \Pi$:
\[
  \sfB_\sfF(\pi^* \parallel \pi_0) \geq \sfB_\sfF(\pi^* \parallel \h \pi)
  + \sfB_\sfF(\h \pi \parallel \pi_0).
\]
Thus, we have
\begin{align*}
\E_{x \sim \sD_\sX}[\sfB_F(\pi^*(x) \parallel \h \pi(x))]
= \sfB_\sfF(\pi^* \parallel \h \pi)
& \leq \sfB_\sfF(\pi^* \parallel \pi_0) - \sfB_\sfF(\h \pi \parallel \pi_0) \\
& = \E_{x \sim \sD_\sX}[\sfB_F(\pi^*(x) \parallel \pi_0(x))] - \E_{x \sim \sD_\sX}[\sfB_F(\h \pi(x) \parallel \pi_0(x))]\\
& \leq \E_{x \sim \sD_\sX}[\sfB_F(\pi^*(x) \parallel \pi_0(x))].
\end{align*}
This proves the first inequalities.
For the explicit bound when $\sC = \Ccoh$, since $\h \pi$ is coherent
and $\Phi$ is an involution, we can write:
\begin{align*}
  & \E_{x \sim \sD_\sX}[\sfB_F(\h \pi(x) \parallel \pi_0(x))]\\
  & = \frac{1}{2} \sum_{x \in \sX} \bracket*{ \P(x) \sfB_F(\h \pi(x) \parallel \pi_0(x)) + \P(\Phi(x)) \sfB_F(\h \pi(\Phi(x)) \parallel \pi_0(\Phi(x))) } \\
  & = \frac{1}{2} \sum_{x \in \sX} \bracket*{ \P(x) \sfB_F(\h \pi(x) \parallel \pi_0(x)) + \P(\Phi(x)) \sfB_F(\h \pi(x) \parallel \pi_0(\Phi(x))) } \\
  & = \frac{1}{2} \sum_{x \in \sX} (\P(x) + \P(\Phi(x)))
    \bracket[\big]{ \lambda(x) \sfB_F(\h \pi(x) \parallel \pi_0(x))
    + (1 - \lambda(x)) \sfB_F(\h \pi(x) \parallel \pi_0(\Phi(x))) }.
\end{align*}
Let the term inside the square brackets be denoted $W(x, \h \pi)$. By
the $\mu$-strong convexity of $F$ with respect to $\norm{\cdot}$, we
have $\sfB_F(\sfp \parallel \sfq) \ge \frac{\mu}{2} \norm{\sfp - \sfq}^2$, thus
\[
 W(x, \h \pi) \ge \frac{\mu}{2} \bracket*{ \lambda(x) \norm{\h \pi(x) - \pi_0(x)}^2 + (1 - \lambda(x)) \norm{\h \pi(x) - \pi_0(\Phi(x))}^2 }.
\]
Using the identity
$\lambda \norm{\sfr - \sfp}^2 + (1 - \lambda) \norm{\sfr - \sfq}^2 =
\norm{\sfr - (\lambda \sfp + (1 - \lambda) \sfq)}^2 + \lambda(1 -
\lambda) \norm{\sfp - \sfq}^2 \ge \lambda(1 - \lambda) \norm{\sfp -
  \sfq}^2$, we obtain:
\[
  W(x, \h \pi) \geq \frac{\mu}{2} \lambda(x)(1 - \lambda(x))
  \norm{\pi_0(x) - \pi_0(\Phi(x))}^2.
\]
Replacing the term inside the square brackets with this lower bound
yields
\begin{align*}
  \E_{x \sim \sD_\sX}[\sfB_F(\h \pi(x) \parallel \pi_0(x))]
  & \geq \frac{1}{2} \sum_{x \in \sX} (\P(x)
    + \P(\Phi(x))) \frac{\mu}{2} \lambda(x)(1 - \lambda(x))
    \norm{\pi_0(x) - \pi_0(\Phi(x))}^2 \\
  & = \frac{\mu}{2} \E_{x \sim \sD_X} \bracket*{ \lambda(x)(1 - \lambda(x))
    \norm{\pi_0(x) - \pi_0(\Phi(x))}^2 }.
\end{align*}
Combining this with the main inequality yields the explicit bound
stated in the theorem.
\end{proof}
The theorem shows that the coherent conditional distribution
$\h \pi$ with minimal $\sfB_F$-divergence from $\pi_0$
is closer to $\pi^*$ than $\pi_0$. In particular, if we choose
$F$ to be the negative entropy, and thus $\sfB_F$ the unnormalized
relative entropy $\KL$, the theorem guarantees
that $\h \pi$ admits a 
$\log$-loss no larger than $\pi_0$, since:
\[
\E_{x \sim \sD_\sX} \bracket*{\KL \paren*{\pi^*(x) \parallel \h \pi(x)}}
= \E_{(x, y)} \bracket*{\bracket*{-\log (\h \pi(x, y))} - \bracket*{- \log (\pi^*(x, y))}}.
\]
Note also that the negative entropy is $1$-strongly convex with respect
to norm $1$ (Sch\"utzenberger-Pinsker inequality \citep{Rioul2023}),
thus $\mu = 1$ in that case.

In the special case where $\Pi \cap \sC$ is an affine set, the
Pythagorean inequality in the first statement of the theorem becomes
an equality (Theorem~\ref{th:pythagorean-equality}). In our setting,
however, where $\Pi$ is bounded, this can only occur in the trivial
case where the intersection reduces to a single point (we are assuming
it is non empty).

\subsection{Improvement guarantees in non-realizable setting}
\label{sec:non-realizable-setting}

For the following result, we no longer assume that $\pi^*$ is in
$\Ccoh \cap \Pi$.

\begin{theorem}[Coherent Bregman–projection with non-realizable $\pi^*$]
\label{th:direct-bregman-projection-non-realizable}
Assume that $\Phi$ is an involution. Let
$\e = \E_{x \sim \sD_\sX} \bracket*{\sfB_F(\pi^*(x) \parallel \ov
  \pi(x))}$, where $\ov \pi$ is the Bregman-projection of $\pi^*$ onto
$\Ccoh \cap \Pi$.
Then, the solution $\h \pi$ of the following convex optimization
problem:
\begin{align*}
  \min_{\pi \in \Pi} \quad & \E_{x \sim \sD_\sX} \bracket*{\sfB_F\paren*{\pi(x) \parallel \pi_0(x)}} \\
  \text{subject to} \quad & \pi \in \sC \cap \Pi,
\end{align*}
satisfies:
\begin{align*}
    & \E\bracket*{\sfB_F(\pi^*(x) \parallel \h \pi(x))}
    - \E\bracket*{\sfB_F(\pi^*(x) \parallel \pi_0(x))}\\
    & \leq - \E\bracket*{\sfB_F(\h \pi(x) \parallel \pi_0(x))} 
    + \E\bracket*{\tri{\pi^*(x) - \ov \pi(x), \nabla F(\pi_0(x)) - \nabla F(\h \pi(x))}}.
\end{align*}
If that $F$ is $\mu$-strongly convex with respect to the norm
$\norm{\cdot}$, then, the following more explicit bound holds:
\begin{align*}
  & \E_{x \sim \sD_\sX} \bracket*{\sfB_F \paren*{\pi^*(x) \parallel \h \pi(x)}}
  - \E_{x \sim \sD_\sX} \bracket*{\sfB_F \paren*{\pi^*(x) \parallel \pi_0(x)}}\\
  & \leq
    - \E_{x \sim \sD_\sX}\bracket*{\sfB_F(\h \pi(x) \parallel \pi_0(x))}
    + \sqrt{\frac{2 \e}{\mu}} \sqrt{\E_{x \sim \sD_\sX}\bracket*{\norm{\nabla F(\pi_0(x))
    - \nabla F(\h \pi(x))}_*^2}}.
\end{align*}
Furthermore, $F$ is also $L$-smooth, then, we have
\[
  \E\bracket*{\sfB_F(\pi^*(x) \parallel \h \pi(x))}
  - \E\bracket*{\sfB_F(\pi^*(x) \parallel \pi_0(x))}
  \leq - D  + \frac{2L}{\mu} \sqrt{\e D},
\]
where $D = \E\bracket*{\sfB_F(\h \pi(x) \parallel \pi_0(x))}$. The
right-hand side is strictly negative whenever
$\e < \bracket*{\frac{\mu}{2L}}^2 D$.
\end{theorem}
\begin{proof}
  By the triangle inequality-type identity for Bregman divergences, we can write:
  \begin{align*}
    \sfB_F(\pi^*(x) \parallel \h \pi(x))
    & = \sfB_F(\pi^*(x) \parallel \ov \pi(x)) + \sfB_F(\ov \pi(x) \parallel \h \pi(x))
    + \tri{\pi^*(x) - \ov \pi(x), \nabla F(\ov \pi(x)) - \nabla F(\h \pi(x))}\\
    \sfB_F(\pi^*(x) \parallel \pi_0(x))
    & = \sfB_F(\pi^*(x) \parallel \ov \pi(x)) + \sfB_F(\ov \pi(x) \parallel \pi_0(x))
    + \tri{\pi^*(x) - \ov \pi(x), \nabla F(\ov \pi(x)) - \nabla F(\pi_0(x))}.
  \end{align*}
  Subtracting the second inequality from the first one and taking
  expectation yields:
  \begin{align*}
    \E\bracket*{\sfB_F(\pi^*(x) \parallel \h \pi(x))}
    - \E\bracket*{\sfB_F(\pi^*(x) \parallel \pi_0(x))}
    & = \E\bracket*{\sfB_F(\ov \pi(x) \parallel \h \pi(x)) - \sfB_F(\ov \pi(x) \parallel \pi_0(x))}\\
    & \quad + \tri{\pi^*(x) - \ov \pi(x), \nabla F(\pi_0(x)) - \nabla F(\h \pi(x))}.
  \end{align*}
  Now, by the Pythagorean theorem, since $\h \pi$ is the projection of $\pi_0$
  onto $\Ccoh \cap \Pi$ and since $\ov \pi$ is in $\Ccoh \cap \Pi$
  we have
  \[
    \E\bracket*{\sfB_F(\ov \pi(x) \parallel \h \pi(x))}
    \leq \E\bracket*{\sfB_F(\ov \pi(x) \parallel \pi_0(x))}
    - \E\bracket*{\sfB_F(\h \pi(x) \parallel \pi_0(x))}.
  \]
  Using this to bound the first term of right-hand side equality above yields
  \begin{align*}
    & \E\bracket*{\sfB_F(\pi^*(x) \parallel \h \pi(x))}
    - \E\bracket*{\sfB_F(\pi^*(x) \parallel \pi_0(x))}\\
    & \leq - \E\bracket*{\sfB_F(\h \pi(x) \parallel \pi_0(x))} 
    + \E\bracket*{\tri{\pi^*(x) - \ov \pi(x), \nabla F(\pi_0(x)) - \nabla F(\h \pi(x))}}.
  \end{align*}
  By the strong-convexity of $F$, we have
  \[
    \e
    = \E_{x \sim \sD_\sX} \bracket*{\sfB_F(\pi^*(x) \parallel \ov \pi(x))}
    \geq  \E_{x \sim \sD_\sX} \bracket*{\frac{\mu}{2} \norm{\pi^*(x) - \ov \pi(x)}^2}
  \]
  Thus, we have
  \begin{align*}
    & \E\bracket*{\sfB_F(\pi^*(x) \parallel \h \pi(x))}
      - \E\bracket*{\sfB_F(\pi^*(x) \parallel \pi_0(x))}\\
    & \leq - \E\bracket*{\sfB_F(\h \pi(x) \parallel \pi_0(x))}
      + \sqrt{\frac{2 \e}{\mu}} \sqrt{\E\bracket*{\norm{\nabla F(\pi_0(x))
        - \nabla F(\h \pi(x))}_*^2}}.
  \end{align*}
  If $F$ is $L$-smooth, then
  \[
    \E \bracket*{\norm{\nabla F(\pi_0(x)) - \nabla F(\h \pi(x))}_*^2}
    \leq L^2 \E\bracket*{\norm{\pi_0(x) - \h \pi(x)}^2}
    \leq \frac{2 L^2}{\mu} \E\bracket*{\sfB_F(\h \pi(x) \parallel \pi_0(x))}.
\]
Define $D = \E\bracket*{\sfB_F(\h \pi(x) \parallel \pi_0(x))}$, then, we can write:
\begin{align*}
    \E\bracket*{\sfB_F(\pi^*(x) \parallel \h \pi(x))}
      - \E\bracket*{\sfB_F(\pi^*(x) \parallel \pi_0(x))}
    \leq - D  + \frac{2L}{\mu} \sqrt{\e D}.
\end{align*}
The right-hand is negative iff $\frac{2L}{\mu} \sqrt{\e D} < D$, that
is $\e < \bracket*{\frac{\mu}{2L}}^2 D$.
\end{proof}
The theorem shows that when $F$ is $\mu$-strongly convex, for $\pi^*$
approximately realizable ($\e$ small), the coherent projection still
guarantees an improvement over the baseline $\pi_0$. Moreover, if $F$ is
also $L$-smooth, a strict improvement is
guaranteed whenever $\e < \bracket*{\frac{\mu}{2L}}^2 D$.

In the special case of the unnormalized relative entropy where $F$ is
the negative entropy, $F$ is $1$-strongly convex with respect to
$\norm{\cdot}_1$, thus $\mu = 1$. In general, however, $F$ is not
globally smooth, since $(\nabla F)_i (\sfp) = 1 + \log \sfp_i$ and the
Hessian is diagonal with entries $1/p_i$. Nevertheless, if we can
restrict attention to
$\sS_\alpha = \curl*{\sfp \colon \sfp_i \geq \alpha}$, then the
Hessian operator norm is bounded by $1/\alpha$ and $\nabla F$ is
$L$-Lipschitz with $L \leq 1/\alpha$.

\subsection{Improvement guarantees for empirical Bregman–projection}
\label{sec:empirical-projection}

We now analyze the properties of the solution $\h \pi_S$ of the
empirical optimization problem based on an i.i.d.\ sample
$S$. Specifically, we seek to compare $\sfB_F(\pi^* \parallel \h \pi_S)$
with $\sfB_F(\pi^* \parallel \h \pi)$.

\subsubsection{General guarantees}

This section establishes performance guarantees for the empirical
solution $\h \pi_S$ that hold in the most general case. Crucially,
these results do not rely on structural assumptions about the
objective functional, such as strong convexity. The primary result,
Theorem~\ref{th:empirical-direct-bregman-projection}, provides a
general bound applicable to any valid Bregman projection problem.

\begin{theorem}[Guarantees for empirical Bregman–projection]
  \label{th:empirical-direct-bregman-projection}
  Let $\sC$ be a closed convex set and assume that the intersection
  $\ov \Pi = \sC \cap \Pi$ is also closed and convex.  Let
  $S = (x_1, \ldots, x_m)$ be a sample of size $m$ drawn i.i.d.\ from
  $\sD_\sX$. Let $\h \sD_\sX$ denote the corresponding empirical
  distribution.
Define the function class
\[
  \cF = \curl*{x \mapsto \sfB_F(\pi(x) \parallel \pi_0(x)) \colon
    \pi \in \ov \Pi} \cup \curl*{x \mapsto \sfB_F(\pi'(x) \parallel \pi(x)) \colon
    \pi, \pi' \in \ov \Pi},
\]
and set
$\e_m = \sup_{f \in \cF} \abs*{\E_{x \sim \sD_\sX}[f(x)] - \E_{x \sim \h \sD_\sX}[f(x)]}$.
Let $\h \pi$ denote a (population) minimizer of
$\pi \mapsto \E_{x \sim \sD_\sX} \bracket*{\sfB_F\paren*{\pi(x)
    \parallel \pi_0(x)}}$ over $\ov \Pi$ and let
$\h \pi_S$ denote the corresponding empirical minimizer of
$\pi \mapsto \E_{x \sim \h \sD_\sX} \bracket*{\sfB_F\paren*{\pi(x)
    \parallel \pi_0(x)}}$ over $\ov \Pi$.
Assume that $\pi^*$ is in $\ov \Pi$. Then, the following holds:
\begin{align*}
\E_{x \sim \sD_\sX} \bracket*{\sfB_F(\pi^*(x) \parallel \h \pi_S(x))}
  & \leq \E_{x \sim \sD_\sX} \bracket*{\sfB_F(\pi^*(x) \parallel \h \pi(x))}\\
  & \quad + 6 \e_m
  + \E_{x \sim \sD_\sX} \bracket*{\sfB_F(\pi^*(x) \parallel \pi_0(x)) - \sfB_F(\h \pi(x) \parallel \pi_0(x))}.
\end{align*}
In particular, if $\pi^* = \h \pi$, then, the bound simplifies to:
\[
\E_{x \sim \sD_\sX}\bracket*{\sfB_F(\pi^*(x) \parallel \h \pi_S(x))}
  \leq \E_{x \sim \sD_\sX} \bracket*{\sfB_F(\pi^*(x) \parallel \h \pi(x))} + 6\e_m.
\]
\end{theorem}
\begin{proof}  
  The proof consists of combining the empirical and population
  optimality of $\h \pi$ and $\h \pi_S$, applying the triangular-type
  Bregman identity to relate their gradients via $\pi^*$, and bounding
  the resulting difference using standard generalization inequalities.
We will seek a bound on the expectation of
$\Delta(x) = \sfB_F(\pi^*(x) \parallel \h \pi_S(x)) - \sfB_F(\pi^*(x)
\parallel \h \pi(x))$. By the triangular identity for Bregman
divergences, for any $x \in \sX$, the following equality holds:
\begin{multline*}
  \sfB_F(\pi^*(x) \parallel \h \pi_S(x))\\
  = \sfB_F(\pi^*(x) \parallel \h \pi(x)) + \sfB_F(\h \pi(x) \parallel \h \pi_S(x))
  + \tri*{\nabla F(\h \pi(x)) - \nabla F(\h \pi_S(x)), \pi^*(x) - \h \pi(x)}.
\end{multline*}
In view of that, we can rewrite the expression of
$\Delta(x)$ as follows:
\begin{align*}
  \Delta(x)
  & = \sfB_F(\pi^*(x) \parallel \h \pi_S(x)) - \sfB_F(\pi^*(x)
\parallel \h \pi(x)) \\
  & = \sfB_F(\h \pi(x) \parallel \h \pi_S(x)) + \tri*{\nabla F(\h \pi(x)) - \nabla F(\h \pi_S(x)), \pi^*(x) - \h \pi(x)}\\
  & = \sfB_F(\h \pi(x) \parallel \h \pi_S(x)) + \sfB_F(\h \pi_S(x) \parallel \h \pi(x)) + \tri*{\nabla F(\h \pi(x)) - \nabla F(\h \pi_S(x)), \pi^*(x) - \h \pi(x)}\\
  & \quad - \sfB_F(\h \pi_S(x) \parallel \h \pi(x))\\
  & = \tri{\nabla F(\h \pi(x)) - \nabla F(\h \pi_S(x)), \h \pi(x) - \h \pi_S(x)} + \tri*{\nabla F(\h \pi(x)) - \nabla F(\h \pi_S(x)), \pi^*(x) - \h \pi(x)}\\
  & \quad - \sfB_F(\h \pi_S(x) \parallel \h \pi(x))
    \tag{identity $\sfB_F(a \parallel b) + \sfB_F(b \parallel a) = \tri*{\nabla F(a) - \nabla F(b), a - b}$}\\
  & = \tri*{\nabla F(\h \pi(x)) - \nabla F(\h \pi_S(x)), \pi^*(x) - \h \pi_S(x)}
    - \sfB_F(\h \pi_S(x) \parallel \h \pi(x)).
\end{align*}
Define $\Psi(x)$ to be the first term on the right-hand side:
\begin{align*}
\Psi(x) = \tri*{\nabla F(\h \pi(x)) - \nabla F(\h \pi_S(x)), \pi^*(x) - \h \pi_S(x)}.
\end{align*}
We will now bound $\E[\Psi]$ by decomposing it into the sum of two
terms $\E[\Psi] = (\E - \h \E)[\Psi] + \h \E[\Psi]$, where we use the
shorhand $\h \E$ to denote the empirical expectation.  By definition
of $\e_m$ and since $\Psi$ is a sum of three functions in $\cF$, we
have $(\E - \h \E)[\Psi] \leq 3 \e_m$ (recall from the derivation
relating $\Delta(x)$ and $\Psi(x)$ that we can express $\Psi(x)$ as
$\Psi(x) = \sfB_F(\pi^*(x) \parallel \h \pi_S(x)) - \sfB_F(\pi^*(x)
\parallel \h \pi(x)) + \sfB_F(\h \pi_S(x) \parallel \h \pi(x))$). By
the optimality of $\h \pi_S$ as an empirical minimizer, for any
$\pi \in \ov \Pi$, we have,
\[
\h \E\bracket*{\tri*{\nabla F(\h \pi_S(x)) - \nabla F(\pi_0(x)), \pi(x) - \h \pi_S(x)}} \geq 0.
\]
Thus, since $\pi^*$ is in $\ov \Pi$, we can write
\[
\begin{aligned}
\h \E[\Psi(x)]
  & = \h \E\bracket*{\tri*{\nabla F(\h \pi(x)) - \nabla F(\h \pi_S(x)), \pi^*(x) - \h \pi_S(x)}}\\
  & = \h \E\bracket*{\tri*{\nabla F(\h \pi(x)) - \nabla F(\pi_0(x)), \pi^*(x) - \h \pi_S(x)}}\\
  & \quad  -  \h \E\bracket*{\tri*{\nabla F(\h \pi_S(x)) - \nabla F(\pi_0(x)), \pi^*(x) - \h \pi_S(x)}}\\
  & \leq \h \E\bracket*{\tri*{\nabla F(\h \pi(x)) - \nabla F(\pi_0(x)), \pi^*(x) - \h \pi_S(x)}}.
\end{aligned}
\]
We now write the right-hand side
as the difference of two terms and apply the triangular-type Bregman identity to
each:
\begin{align*}
  & \h \E \bracket*{\tri{\nabla F(\h \pi(x)) - \nabla F(\pi_0(x)), \pi^*(x) - \h \pi_S(x)}}\\
  & = \h \E \bracket*{\tri*{\nabla F(\h \pi(x)) - \nabla F(\pi_0(x)), \pi^*(x) - \pi_0(x)}}
    - \h \E \bracket*{\tri*{\nabla F(\h \pi(x)) - \nabla F(\pi_0(x)), \h \pi_S(x) - \pi_0(x)}}\\
  & = \h \E \bracket*{\sfB_F(\pi^*(x) \parallel \pi_0(x)) + \sfB_F(\pi_0(x) \parallel \h \pi(x)) - \sfB_F(\pi^*(x) \parallel \h \pi(x))}\\
  & \quad - \h \E \bracket*{\sfB_F(\h \pi_S(x) \parallel \pi_0(x)) + \sfB_F(\pi_0(x) \parallel \h \pi(x)) - \sfB_F(\h \pi_S(x) \parallel \h \pi(x))}\\
  & = \h \E \bracket*{\sfB_F(\pi^*(x) \parallel \pi_0(x)) - \sfB_F(\h \pi_S(x) \parallel \pi_0(x))}
  - \h \E \bracket*{\sfB_F(\pi^*(x) \parallel \h \pi)} + \h \E\bracket*{\sfB_F(\h \pi_S(x) \parallel \h \pi(x))}.
\end{align*}
Dropping the nonpositive term $-\h \E\sfB_F(\pi^*(x) \parallel \h \pi(x)) \leq 0$ yields
\[
  \h \E[\Psi(x)]
  \leq \h \E \bracket*{\sfB_F(\pi^*(x) \parallel \pi_0(x)) - \sfB_F(\h \pi_S(x) \parallel \pi_0(x))}
  + \h \E\bracket*{\sfB_F(\h \pi_S(x) \parallel \h \pi(x))}.
\]
Combining this inequality with the one bounding the difference of population
and empirical terms, we obtain
\[
  \E [\Psi(x)] \leq 3\e_m
  + \h \E \bracket*{\sfB_F(\pi^*(x) \parallel \pi_0(x)) - \sfB_F(\h \pi_S(x) \parallel \pi_0(x))}
  + \h \E\sfB_F(\h \pi_S(x) \parallel \h \pi(x)).
\]
Replacing the empirical difference by population plus deviations gives
\[
\begin{aligned}
  & \mspace{-150mu} \h \E \bracket*{\sfB_F(\pi^*(x) \parallel \pi_0(x)) - \sfB_F(\h \pi_S(x) \parallel \pi_0(x))}\\
  & = \E \bracket*{\sfB_F(\pi^*(x) \parallel \pi_0(x)) - \sfB_F(\h \pi_S(x) \parallel \pi_0(x))}\\
  & \quad + (\h \E - \E )\bracket*{\sfB_F(\pi^*(x) \parallel \pi_0(x))}
    - (\h \E - \E )\bracket*{\sfB_F(\h \pi_S(x) \parallel \pi_0(x))}\\
  & \leq \E \bracket*{\sfB_F(\pi^*(x) \parallel \pi_0(x)) - \sfB_F(\h \pi_S(x) \parallel \pi_0(x))} + 2\e_m,
\end{aligned}
\]
since $\abs*{(\h \E - \E )[\sfB_F(\pi(x) \parallel \pi_0(x))]} \leq \e_m$ for all $\pi \in \ov \Pi$.
Thus, we have
\[
  \E [\Psi(x)] \leq 5\e_m
  + \E \bracket*{\sfB_F(\pi^*(x) \parallel \pi_0(x)) - \sfB_F(\h \pi_S(x) \parallel \pi_0(x))}
  + \h \E\sfB_F(\h \pi_S(x) \parallel \h \pi(x)).
\]
In view of that, we have
\begin{align*}
  \E[\Delta(x)]
  & = \E [\Psi] - \E \sfB_F(\h \pi_S(x) \parallel \h \pi(x))\\
  & \leq 5\e_m
    + \E \bracket*{\sfB_F(\pi^*(x) \parallel \pi_0(x)) - \sfB_F(\h \pi_S(x) \parallel \pi_0(x))}
    + \paren*{\h \E - \E}\bracket*{\sfB_F(\h \pi_S(x) \parallel \h \pi(x))}.
\end{align*}
The last term is bounded by $\e_m$ since $\sfB_F(\h \pi_S(x) \parallel
\h \pi)$ is in $\cF$. Thus,
\begin{align*}
  \E[\Delta(x)]
  & \leq 6\e_m + \E \bracket*{\sfB_F(\pi^*(x) \parallel \pi_0(x)) - \sfB_F(\h \pi(x) \parallel \pi_0(x))}\\
  & = 6\e_m + \E \bracket*{\sfB_F(\pi^*(x) \parallel \pi_0(x))}
    - \E \bracket*{\sfB_F(\h \pi(x) \parallel \pi_0(x))} \\ 
  & \quad + \E \bracket*{\sfB_F(\h \pi(x) \parallel \pi_0(x))}
    - \E \bracket*{\sfB_F(\h \pi(x) \parallel \pi_0(x))}\\
  & \leq 6\e_m + \E \bracket*{\sfB_F(\pi^*(x) \parallel \pi_0(x)) - \sfB_F(\h \pi(x) \parallel \pi_0(x))},
\end{align*}
since $\h \pi$ is a minimizer and thus
$\E \bracket*{\sfB_F(\h \pi(x) \parallel \pi_0(x)) - \sfB_F(\h \pi_S(x) \parallel \pi_0(x))} \leq 0$.
Combining these inequalities gives
\[
  \E \bracket*{\sfB_F(\pi^*(x) \parallel \h \pi_S(x))}
  \leq \E \bracket*{\sfB_F(\pi^*(x) \parallel \h \pi(x))} + 6\e_m
  + \E \bracket*{\sfB_F(\pi^*(x) \parallel \pi_0(x)) - \sfB_F(\h \pi(x) \parallel \pi_0(x))}.
\]
This completes the proof.
\end{proof}
The quantity $\e_m$ is precisely the empirical process term that is
controlled by generalization bounds. In particular, $\e_m$ can be
bounded (with high probability over the draw of the sample $S$) in
terms of the Rademacher complexity of the considered family, the
sample size $m$, and an upper bound on the function values within that
family (see, for example, \citep{KoltchinskiiPanchenko2000,
  KoltchinskiiPanchenko2002,BartlettMendelson2002,
  MohriRostamizadehTalwalkar2018}).

The theorem bounds the excess divergence
$\E \bracket*{\sfB_F(\pi^*(x) \parallel \h \pi_S(x))} - \E
\bracket*{\sfB_F(\pi^*(x) \parallel \h \pi)}$ by the sum of an
estimation error $6 \e_m$ and an approximation error
$\E \bracket*{\sfB_F(\pi^*(x) \parallel \pi_0(x)) - \sfB_F(\h \pi(x)
  \parallel \pi_0(x))}$.  The approximation term is inherent and
non-negative by the Bregman Pythagorean theorem.  It can be relatively
small when $\pi^*$ is relative close to $\pi_0$ or $\h \pi$.
Previously, we established that
$\E \bracket*{\sfB_F(\pi^*(x) \parallel \h \pi(x))} \leq \E
\bracket*{\sfB_F(\pi^*(x) \parallel \pi_0(x))}$ for $\pi^* \in \ov \Pi$.
Thus, for a sufficiently large sample $m$ so that $\e_m$ is small,
and a controlled approximation error, a similar guarantee
approximately holds for the empirical minimizer $\h \pi_S$.

The factor $6$ appearing in the bound is due to a conservative
grouping of deviation terms.  Defining $\e_m$ exactly as
$\e_m = \sup_{\pi\in \ov \Pi}\abs*{(\E - \h \E) [\sfB_F(\pi(x) \parallel
  \pi_0(x))]} \vee \sup_{\pi, \pi'\in \ov \Pi}\abs*{(\E -\h
  \E)\bracket*{\sfB_F(\pi' \parallel \pi)}}$ and using more detailed
calculations with tighter sign control, the same argument as in the
proof is likely to lead to a more favorable constant factor $2$.

\begin{corollary}[Guarantee for Empirical Improvement]
\label{cor:improvement-bound}
Under the assumptions of Theorem~\ref{th:empirical-direct-bregman-projection},
the following inequality holds for $\h \pi_S$:
\begin{multline*}
\E_{x \sim \sD_\sX} \bracket*{\sfB_F(\pi^*(x) \parallel \h \pi_S(x))}
- \E_{x \sim \sD_\sX} \bracket*{\sfB_F(\pi^*(x) \parallel \pi_0(x))} \\
\leq 6 \e_m
+ \E_{x \sim \sD_\sX} \bracket*{\sfB_F(\pi^*(x) \parallel \pi_0(x))} - 2 \E_{x \sim \sD_\sX} \bracket*{\sfB_F(\h \pi(x) \parallel \pi_0(x))}.
\end{multline*}
\end{corollary}
\begin{proof}
  The proof follows immediately the statement of
  Theorem~\ref{th:empirical-direct-bregman-projection}, and the
  Bregman Pythagorean theorem (see
  Theorem~\ref{th:direct-bregman-projection}).
\end{proof}
The corollary offers an explicit guarantee for the performance gain
achieved by the empirical solution $\h \pi_S$ over the initial
reference $\pi_0$.  It shows that the empirical solution $\h \pi_S$ is
a definite improvement over the starting point $\pi_0$ when the
inherent benefit of the projection, captured by
$2 \E[\sfB_F(\h \pi \parallel \pi_0)]$, is large enough to overcome the
statistical cost of finite-sample estimation.

\subsubsection{Discussion}

The guarantee provided by
Theorem~\ref{th:empirical-direct-bregman-projection} is somewhat
subtle, and its proof is correspondingly complex. This section aims
to unpack this result. We will first motivate why its specific form is
necessary, showing that a more direct guarantee on the improvement is
not possible without additional assumptions. We then show that by
introducing a standard strong-convexity assumption, it is possible to
achieve the kind of intuitive guarantee one might initially expect. To
formalize this, we first establish a general lemma before presenting
our main result under this new condition.

Let the improvement of a solution $\pi$ be defined as
\[
 \Improv(\pi) = \E\bracket*{\sfB_F(\pi^*(x) \parallel \pi_0(x))} -
 \E\bracket*{\sfB_F(\pi^*(x) \parallel \pi(x))}.
\]
If we could prove an inequality of the type
$\abs*{\E\bracket*{\sfB_F(\pi^*(x) \parallel \h \pi_S(x))} -
  \E\bracket*{\sfB_F(\pi^*(x) \parallel \h \pi(x))}} \leq C\e_m$ for
some positive constant $C$, this would lead directly to an improvement
inequality of the form
$\Improv(\h \pi_S) \geq \Improv(\h \pi) - C \e_m$. However, this
inequality is unlikely to hold in the general case.
To see why, consider a standard approach where we bound the population
difference by the empirical difference plus the generalization error:
\[
\E\bracket*{\sfB_F(\pi^* \parallel \h \pi_S)}
- \E\bracket*{\sfB_F(\pi^* \parallel \h \pi)}
\leq \paren*{\E_{x \sim \h \sD_\sX}\bracket*{\sfB_F(\pi^* \parallel \h \pi_S)}
- \E_{x \sim \h \sD_\sX}\bracket*{\sfB_F(\pi^* \parallel \h \pi)}} + 2\e_m.
\]
For the desired inequality to hold, the empirical term in the
parenthesis would need to be non-positive. Yet, there is no reason for
this to be true. The solution $\h \pi_S$ was chosen to minimize the
distance to $\pi_0$ on the sample, not the distance to $\pi^*$. We
only know that
$\E_{x \sim \h \sD_\sX}\bracket*{\sfB_F(\h \pi_S \parallel \pi_0)}
\leq \E_{x \sim \h \sD_\sX}\bracket*{\sfB_F(\h \pi \parallel
  \pi_0)}$. It is entirely possible for $\h \pi_S$ to be closer to
$\pi_0$ on the sample while simultaneously being farther from $\pi^*$.

This mismatch is precisely why the proof of our main theorem is more
involved and must include a \emph{misalignment term},
$\E\bracket*{\sfB_F(\pi^*(x) \parallel \pi_0(x)) - \sfB_F(\h \pi(x)
  \parallel \pi_0(x))}$, which is the mathematical price we pay for
the difference between our optimization objective (proximity to
$\pi_0$) and our evaluation goal (proximity to $\pi^*$).

\subsubsection{Guarantees under strong-convexity}

In the following lemma, we establish two key results that build upon these ideas.

\begin{lemma}[Two-sided guarantee for the empirical minimizer]
\label{lemma:two-sided-guarantee}
Under the assumptions of
Theorem~\ref{th:empirical-direct-bregman-projection} and using the
same notation, the following guarantees hold:
\begin{enumerate}
\item The population objective value for the empirical solution
  $\h \pi_S$ is tightly bounded around that of the population solution
  $\h \pi$:
    \[
     \E\bracket*{\sfB_F(\h \pi(x) \parallel \pi_0(x))}
     \leq \E\bracket*{\sfB_F(\h \pi_S(x) \parallel \pi_0(x))}
     \leq \E\bracket*{\sfB_F(\h \pi(x) \parallel \pi_0(x))} + 2\e_m.
    \]

  \item The improvement of the empirical solution is bounded below as
    follows:
    \[
      \Improv(\h \pi_S)
      \geq \E\bracket*{\sfB_F(\h \pi(x) \parallel \pi_0(x))}
      - \E\bracket*{\sfB_F(\pi^*(x) \parallel \h \pi(x))} - 6\e_m.
    \]
\end{enumerate}
\end{lemma}

\begin{proof}
  For the first part, the left-hand side inequality is a direct
  consequence of the optimality of $\h \pi$. For the right-hand side,
  we have the following chain of inequalities:
\begin{align*}
 \E\bracket*{\sfB_F(\h \pi_S(x) \parallel \pi_0(x))}
  & \leq \E_{x \sim \h \sD_\sX}\bracket*{\sfB_F(\h \pi_S(x) \parallel \pi_0(x))}
    + \e_m
  \tag{By definition of $\e_m$}\\
  & \leq \E_{x \sim \h \sD_\sX}\bracket*{\sfB_F(\h \pi(x) \parallel \pi_0(x))}
    + \e_m
  \tag{By optimality of $\h \pi_S$}\\
  & \leq \E_{x \sim \sD_\sX}\bracket*{\sfB_F(\h \pi(x) \parallel \pi_0(x))}
    + 2\e_m
  \tag{By definition of $\e_m$}.
\end{align*}
For the second part, the proof of
Theorem~\ref{th:empirical-direct-bregman-projection} provides the
following intermediate inequality:
\[
  \E\bracket*{\sfB_F(\pi^*(x) \parallel \h \pi_S(x))
    - \sfB_F(\pi^*(x) \parallel \h \pi(x))}
 \leq 6\e_m + \E\bracket*{\sfB_F(\pi^*(x) \parallel \pi_0(x))
   - \sfB_F(\h \pi_S(x) \parallel \pi_0(x))}.
\]
Rearranging this inequality to isolate the improvement term
$\Improv(\h \pi_S)$ gives:
\[
  \E\bracket*{\sfB_F(\pi^*(x) \parallel \pi_0(x))}
  - \E\bracket*{\sfB_F(\pi^*(x) \parallel \h \pi_S(x))}
  \geq \E\bracket*{\sfB_F(\h \pi_S(x) \parallel \pi_0(x))}
  - \E\bracket*{\sfB_F(\pi^*(x) \parallel \h \pi(x))} - 6\e_m.
\]
The left-hand side is $\Improv(\h \pi_S)$ by definition. For the
right-hand side, we use the result from the first part of this proof,
namely that
$\E\bracket*{\sfB_F(\h \pi_S(x) \parallel \pi_0(x))} \geq
\E\bracket*{\sfB_F(\h \pi(x) \parallel \pi_0(x))}$. Substituting this
yields the desired bound and completes the proof.
\end{proof}

We now use this lemma to derive of a guarantee on the improvement
of $\h \pi_S$ with respect to that of $\h \pi$ under a strong
convexity assumption.

\begin{theorem}[Improvement Guarantee under Strong Convexity]
\label{th:improvement-strong-convexity}
Adopt the assumptions and notation of
Theorem~\ref{th:empirical-direct-bregman-projection}. Assume the
objective functional
$J(\pi) = \E\bracket*{\sfB_F(\pi(x) \parallel \pi_0(x))}$ is
$\mu$-strongly convex with respect to the squared L2-norm, that is,
for any $\pi \in \ov \Pi$:
\[
  J(\pi)
  \geq J(\h \pi)
  + \frac{\mu}{2} \E\bracket*{\norm{\pi(x) - \h \pi(x)}^2}.
\]
Furthermore, assume that the error mapping
$\pi \mapsto \E\bracket*{\sfB_F(\pi^*(x) \parallel \pi(x))}$ is
$L_{\pi^*}$-Lipschitz with respect to the L2-norm.
Then, the improvement of the empirical solution is bounded by:
\[
  \Improv(\h \pi_S) \geq \Improv(\h \pi)
  - \frac{2L_{\pi^*}}{\sqrt{\mu}} \sqrt{\e_m}.
\]
\end{theorem}

\begin{proof}
  By Lemma~\ref{lemma:two-sided-guarantee}, the excess population risk
  of the empirical solution is bounded by $2\e_m$. The assumption of
  $\mu$-strong convexity gives us a lower bound on this same quantity:
\[
  \E\bracket*{\sfB_F(\h \pi_S(x) \parallel \pi_0(x))}
  - \E\bracket*{\sfB_F(\h \pi(x) \parallel \pi_0(x))}
  \geq \frac{\mu}{2} \E\bracket*{\norm{\h \pi_S(x) - \h \pi(x)}^2}.
\]
Thus, we obtain:
$\frac{\mu}{2} \E\bracket*{\norm{\h \pi_S(x) - \h \pi(x)}^2} \leq
2\e_m$. We are interested in the difference
$\E\bracket*{\sfB_F(\pi^* \parallel \h \pi_S(x))} -
\E\bracket*{\sfB_F(\pi^* \parallel \h \pi(x))}$. The assumption of
$L_{\pi^*}$-Lipschitz continuity of the error mapping allows us to
bound the absolute value of this difference using this inequality:
\begin{align*}
  \abs*{\E\bracket*{\sfB_F(\pi^* \parallel \h \pi_S(x))}
  - \E\bracket*{\sfB_F(\pi^* \parallel \h \pi(x))}}
  & \leq L_{\pi^*} \sqrt{\E\bracket*{\norm{\h \pi_S(x) - \h \pi(x)}^2}} \\
  & \leq L_{\pi^*} \sqrt{\frac{4\e_m}{\mu}}
  = \frac{2L_{\pi^*}}{\sqrt{\mu}} \sqrt{\e_m}.
\end{align*}
Thus, we have
\[
  \E\bracket*{\sfB_F(\pi^* \parallel \h \pi_S(x))}
  \leq \E\bracket*{\sfB_F(\pi^* \parallel \h \pi(x))}
  + \frac{2L_{\pi^*}}{\sqrt{\mu}} \sqrt{\e_m}.
\]
Subtracting $\E\bracket*{\sfB_F(\pi^* \parallel \pi_0(x))}$ from both
sides and reversing the inequality yields the desired result on the
improvement, which completes the proof.
\end{proof}

\subsection{Improvement guarantees for relaxed constraints}
\label{sec:relaxed-case}

In this section, we analyze the setting with \emph{relaxed}
constraints, where we no longer require strict equalities of the form
$\pi(x) = \pi(\Phi(x))$ for all $x \in \sX$.

Let $\sfD$ be a jointly convex divergence between distributions, such
as the relative entropy $\KL$, total variation (TV) or
$\ell_1$-distance, squared Euclidean distance, Jensen-Shannon
divergence $\JS$, or squared Hellinger distance $\Hell^2$, or the
symmetrized $\KL$.  For any
$\Lambda \geq 0$, define the set $\sC_\Lambda^\sfD$
\[
  \sC^\sfD_\Lambda = \curl*{\pi \in \Piall \colon \E_{x \sim \sD_\sX}
    \bracket*{\sfD(\pi(x), \pi(\Phi(x)))} \leq \Lambda},
\]
which is convex.
We then consider the following optimization problem:
\begin{align}
\label{opt:relaxed}
  \min_{\pi \in \Piall} \quad & \E_{x \sim \sD_\sX} \bracket*{\sfB_F(\pi \parallel \pi_0)}\\
  \text{subject to:}  \quad & \E_{x \sim \sD_\sX}
    \bracket*{\sfD(\pi(x), \pi(\Phi(x)))} \leq \Lambda.\nonumber
\end{align}
This is a convex optimization problem, since both the objective
$\pi \mapsto \E_{x \sim \sD_\sX} \bracket*{\sfB_F(\pi(x) \parallel
  \pi_0(x))}$ and the constraint
$\pi \mapsto \E_{x \sim \sD_\sX} \bracket*{\sfD(\pi(x),
  \pi(\Phi(x)))}$ are convex in $\pi$.
By standard Lagrangian duality (Slater’s condition holds since
$\Lambda > 0$ yields a strictly feasible interior), this is equivalent
to the unconstrained problem
\begin{align}
\label{opt:unconstrained}
  \min_{\pi \in \Piall} \quad & \E_{x \sim \sD_\sX} \bracket*{\sfB_F(\pi(x) \parallel \pi_0(x))}
                                + \lambda \E_{x \sim \sD_\sX}
    \bracket*{\sfD(\pi(x), \pi(\Phi(x)))},
\end{align}
for some $\lambda \geq 0$.

\begin{theorem}[Improvement guarantees for relaxed constraints]
\label{th:relaxed-constraints}
Assume that $\Phi$ is an involution.  Let $\|\cdot\|$ be a norm on
$\Rset^{|\mathcal{Y}|} \supseteq \Delta(\mathcal{Y})$. Assume that the
Bregman generator $F$ is $\mu_F$-strongly convex and the divergence
$\sfD$ is $\mu_\sfD$-strongly convex, both with respect to this norm
$\|\cdot\|$.  Define the $L_2$ norm on the function space using this
$\norm{\pi}_{L_2} = \bracket*{\E_{x \sim \sD_\sX}
  \bracket{\|\pi(x)\|^2}}^{\frac{1}{2}}$.
Define the incoherence gap of $\pi_0$ relative to this $L_2$ norm as
$\Delta_{\mathrm{coh}} = \inf_{\pi \in \Ccoh} \norm{\pi_0 -
  \pi}_{L_2}^2$.
Let $\h \pi$ be the $\sfB_F$ Bregman-projection of $\pi_0$ onto the
relaxed constraint set
$\sC^\sfD_\Lambda = \curl*{\pi \in \Piall \colon \E_{x \sim \sD_\sX}
  \bracket*{\sfD(\pi(x), \pi(\Phi(x)))} \leq \Lambda}$ (solution to
Problem~\ref{opt:relaxed}). Then, the improvement of $\h \pi$
satisfies
\[
 \Improv(\h \pi)
 \geq \frac{\mu_F}{2}\bracket*{\sqrt{\Delta_{\mathrm{coh}}}
   - \sqrt{\frac{\Lambda}{2 \mu_\sfD}}}_+^2.
\]
Furthermore, let $\gamma_0 = \norm{\pi_0 - \pi_0 \circ \Phi}_{L_2}^2$
and
$C_\Phi = \sup_{\pi \in \Piall} \frac{\norm{\pi \circ
    \Phi}_{L_2}}{\norm{\pi}_{L_2}}$ be the operator norm of the
composition operator induced by $\Phi$ with respect to the $L_2$
norm. The incoherence gap is bounded by the computable incoherence
$\gamma_0$:
\begin{align*}
  \frac{\gamma_0}{(1 + C_\Phi)^2}
  \leq \Delta_{\mathrm{coh}}
  \leq \frac{\gamma_0}{4}.
\end{align*}
\end{theorem}

\begin{proof}
  By the joint convexity of $\sfD$,
  $\pi \mapsto \E[\sfD(\pi(x) \parallel \pi(\Phi(x)))]$ is convex
  and continuous, thus $\sC^\sfD_\Lambda$ is closed and convex and the
  $\sfB_F$-projection of $\pi_0$ onto $\sC^\sfD_\Lambda$, $\h \pi$, is
  well defined.

By Theorem~\ref{th:direct-bregman-projection}, the improvement of the
Bregman projection $\h \pi$ onto $\sC^\sfD_\Lambda$ over $\pi_0$ is
lower bounded as follows:
\[
  \Improv(\h \pi)
  = \E_{x \sim \sD_\sX} \bracket*{\sfB_F\paren*{\pi^*(x) \parallel \pi_0(x)}}
  - \E_{x \sim \sD_\sX} \bracket*{\sfB_F \paren*{\pi^*(x) \parallel \h \pi(x)}}
\geq \E_{x \sim \sD_\sX}[\sfB_F(\h \pi(x) \parallel \pi_0(x))].
\]
By the $\mu_F$-strong convexity of $F$, we have
\[
\Improv(\h \pi) \geq \inf_{\pi \in \sC^\sfD_\Lambda}
\E[\sfB_F(\pi(x) \parallel \pi_0(x))] \geq \inf_{\pi \in
  \sC^\sfD_\Lambda} \frac{\mu_F}{2} \norm*{\pi - \pi_0}_{L_2}^2 =
\frac{\mu_F}{2} \bracket*{\mathrm{dist}_{L_2}(\pi_0, \sC^\sfD_\Lambda)}^2.
\]
For any $\pi \in \sC^\sfD_\Lambda$, by strong-convexity of $\sfD$, we
can write:
\[
  \Lambda
  \geq \E\bracket*{\sfD(\pi(x) \parallel \pi(\Phi(x)))}
  \geq \frac{\mu_\sfD}{2} \norm*{\pi - \pi \circ \Phi}^2_{L_2}.
\]
Note that for any $\pi$,
$\pi_{\textrm{sym}} = \frac{1}{2} \bracket*{\pi + \pi \circ \Phi}$ is
in $\Ccoh$. Thus, for any $\pi \in \sC^\sfD_\Lambda$, we have
\[
  \mathrm{dist}_{L_2}(\pi, \Ccoh)
  \leq \norm*{\pi - \pi_{\textrm{sym}}}_{L_2}
  = \frac{1}{2} \norm*{\pi - \pi \circ \Phi}_{L_2}
  \leq \frac{1}{2} \sqrt{\frac{2 \Lambda}{\mu_\sfD}}
  = \sqrt{\frac{\Lambda}{2 \mu_\sfD}}.
\]
Thus, $\sC^\sfD_\Lambda$ lies in the ${L_2}$-ball of radius
$r = \sqrt{\frac{\Lambda}{2 \mu_\sfD}}$ around $\Ccoh$.

Let $\Ccoh^r$ be the $L_2$-ball of radius
$r = \sqrt{\frac{\Lambda}{2 \mu_\sfD}}$ around $\Ccoh$.  We have shown
that $\sC^\sfD_\Lambda \subseteq \Ccoh^r$, which implies
$\mathrm{dist}_{L_2}(\pi_0, \sC^\sfD_\Lambda) \geq
\mathrm{dist}_{L_2}(\pi_0, \Ccoh^r)$.
\begin{figure}[t]
  \centering
  \includegraphics[scale=.38]{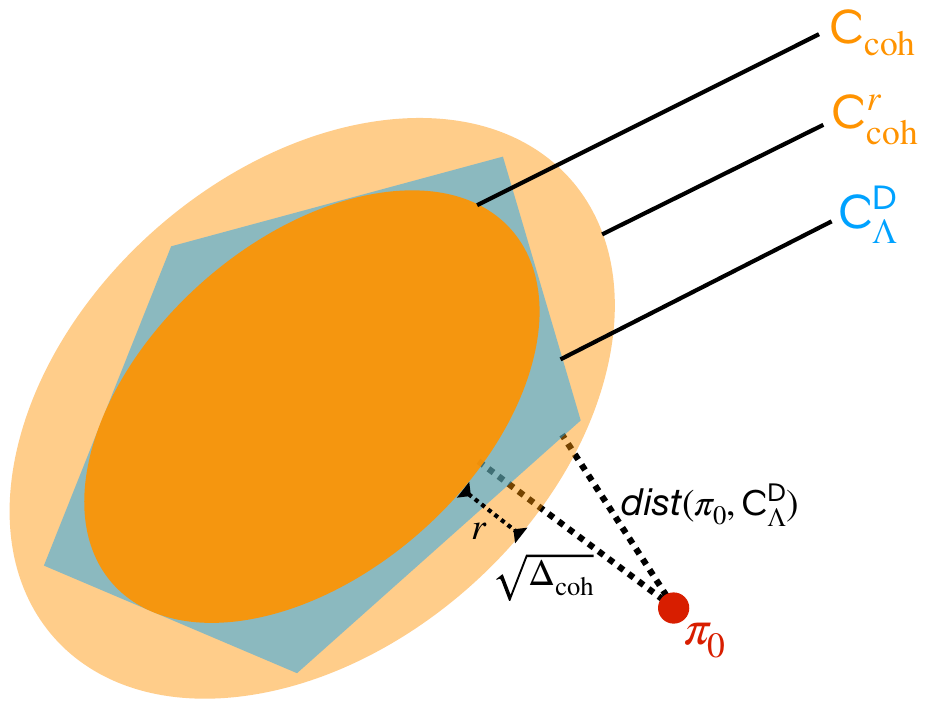}
  \caption{Illustration of the proof of Theorem~\ref{th:relaxed-constraints}.}
  \label{fig:topological}
\end{figure}
For any $c^r \in \Ccoh^r$, there exists $c \in \Ccoh$ with
$\norm{c - c^r}_{L_2} \leq r$. Thus, by the triangle inequality:
\[
\norm{\pi_0 - c^r}_{L_2}
\geq \norm{\pi_0 - c}_{L_2} - \norm{c^r - c}_{L_2}
\geq \norm{\pi_0 - c}_{L_2} - r.
\]
Taking the infimum over $c \in \Ccoh$ on the right-hand side gives
$\norm{\pi_0 - c^r}_{L_2} \geq \mathrm{dist}_{L_2}
(\pi_0, \Ccoh) - r$. Next, taking the infimum over $c^r \in \Ccoh^r$ on the
left-hand side yields:
\[
\mathrm{dist}_{L_2} (\pi_0, \Ccoh^r) \geq \mathrm{dist}_{L_2}
(\pi_0, \Ccoh) - r.
\]
Since the distance is non-negative, and $\sqrt{\Delta_{\mathrm{coh}}} = \mathrm{dist}_{L_2}(\pi_0, \Ccoh)$, we have:
\[
\mathrm{dist}_{L_2}(\pi_0, \sC^\sfD_\Lambda)
\geq \mathrm{dist}_{L_2}(\pi_0, \Ccoh^r)
\geq \bracket*{\sqrt{\Delta_{\mathrm{coh}}} - r}_+.
\]
This proves the main improvement bound.

We now prove the bounds on $\Delta_{\mathrm{coh}}$.
Note that 
$\pi_{\textrm{sym}} = \frac{1}{2}(\pi_0 + \pi_0 \circ \Phi)$ is always in
$\Ccoh$. Since the infimum is the distance to the closest point, we have
the upper bound:
\[
    \sqrt{\Delta_{\mathrm{coh}}} = \mathrm{dist}_{L_2}(\pi_0, \Ccoh)
    \leq \norm{\pi_0 - \pi_{\textrm{sym}}}_{L_2}
    = \norm*{\pi_0 - \frac{1}{2}(\pi_0 + \pi_0 \circ \Phi)}_{L_2}
    = \frac{1}{2} \norm{\pi_0 - \pi_0 \circ \Phi}_{L_2} = \frac{1}{2}\sqrt{\gamma_0}.
\]
For the lower bound, let $\pi$ be any policy in $\Ccoh$, so $\pi = \pi \circ \Phi$.
By the triangle inequality:
\[
    \sqrt{\gamma_0} = \norm{\pi_0 - \pi_0 \circ \Phi}_{L_2}
    = \norm{(\pi_0 - \pi) - (\pi_0 \circ \Phi - \pi \circ \Phi)}_{L_2}
    = \norm{(\pi_0 - \pi) - (\pi_0 - \pi) \circ \Phi}_{L_2}
    \leq \norm{\pi_0 - \pi}_{L_2} + \norm{(\pi_0 - \pi) \circ \Phi}_{L_2}.
\]
By definition of the operator norm $C_\Phi$, we have
$\norm{(\pi_0 - \pi) \circ \Phi}_{L_2} \leq C_\Phi \norm{\pi_0 - \pi}_{L_2}$.
Substituting this back gives
$\sqrt{\gamma_0} \leq \norm{\pi_0 - \pi}_{L_2} + C_\Phi \norm{\pi_0 - \pi}_{L_2}
= (1 + C_\Phi) \norm{\pi_0 - \pi}_{L_2}$.
Rearranging, we get $\norm{\pi_0 - \pi}_{L_2} \geq \frac{\sqrt{\gamma_0}}{1 + C_\Phi}$.
Since this holds for any $\pi \in \Ccoh$, it holds for the infimum:
\[
  \sqrt{\Delta_{\mathrm{coh}}} = \inf_{\pi \in \Ccoh} \norm{\pi_0 - \pi}_{L_2}
  \geq \frac{\sqrt{\gamma_0}}{1 + C_\Phi}.
\]
Squaring the upper and lower bounds gives the result.
\end{proof}

The theorem gives guarantees for both Problems~\ref{opt:relaxed} and
\ref{opt:unconstrained}. In the specific case of the $\JS$ divergence,
we assume the $\ell_1$-norm
$\norm{\, \cdot \,} = \norm{\, \cdot \,}_1$. By
Sch\"utzenberger-Pinsker's inequality, we have
$\JS(p,q) \geq \frac{1}{8} \norm{p-q}^2_1$.  Comparing this to the
strong convexity assumption
$\sfD(p, q) \geq \frac{\mu_\sfD}{2} \norm{p-q}^2$, we get
$\mu_{\JS} = 1/4$.
Substituting this into the theorem yields the improvement bound:
\[
  \Improv(\h \pi) \geq \frac{\mu_F}{2}
  \bracket*{\sqrt{\Delta_{\mathrm{coh}}} - \sqrt{2 \Lambda}}_+^2.
\]
We obtain an identical guarantee for the squared $\Hell^2$ divergence,
since $\Hell^2(p,q) \geq \frac{1}{8} \norm{p-q}^2_1$, and thus
$\mu_{\Hell^2} = 1/4$.
For the symmetrized $\KL$ divergence,
$\KL^{\textrm{sym}}(p,q) = \KL(p,q) + \KL(q,p)$,
Sch\"utzenberger-Pinsker's inequality
$\KL(p,q) \geq 2 D_{TV}(p,q)^2 = \frac{1}{2} \norm{p-q}_1^2$ implies
$\KL^{\textrm{sym}}(p,q) \geq \norm{p-q}_1^2$.  This gives
$\mu_{\KL^{\textrm{sym}}} = 2$, leading to the guarantee:
\[
  \Improv(\h \pi) \geq \frac{\mu_F}{2}
  \bracket*{\sqrt{\Delta_{\mathrm{coh}}} - \frac{\sqrt{\Lambda}}{2}}_+^2.
\]
These results show that the improvement remains strictly positive as
long as $\Lambda$ is not too large, or, equivalently, for sufficiently
small values of $\lambda$.

\textbf{Case of $\Phi$-Invariant Distributions.}
The bounds in Theorem~\ref{th:relaxed-constraints} simplify
substantially and become more practical if we assume the data
distribution $\sD_\sX$ is $\Phi$-invariant.  The $L_2$-norm is then
invariant under composition with $\Phi$:
$\norm{\pi \circ \Phi}_{L_2}^2 = \norm{\pi}_{L_2}^2$, which implies
that the operator norm $C_\Phi = 1$.  Substituting $C_\Phi = 1$ into
the bounds from Theorem~\ref{th:relaxed-constraints} gives an exact
identity:
$\Delta_{\mathrm{coh}} = \frac{1}{4} \gamma_0 = \frac{1}{4} \E_{x \sim
  \sD_\sX} \bracket*{\norm{\pi_0(x) - \pi_0(\Phi(x))}^2}$.

This makes the abstract incoherence gap $\Delta_{\mathrm{coh}}$
directly computable from the initial policy $\pi_0$.  The improvement
guarantee can then be expressed in this fully computable form:
\[
  \Improv(\h \pi)
  \geq \frac{\mu_F}{2}\bracket*{\frac{1}{2}\sqrt{\gamma_0}
    - \sqrt{\frac{\Lambda}{2 \mu_\sfD}}}_+^2.
\]
For example, using the $\ell_1$-norm and the $\JS$ divergence
($\mu_{\JS} = 1/4$), the bound becomes:
\[
  \Improv(\h \pi) \geq \frac{\mu_F}{2}
  \bracket*{\frac{1}{2}\sqrt{\gamma_{0,1}} - \sqrt{2 \Lambda}}_+^2,
\]
where
$\gamma_{0,1} = \E_{x \sim \sD_\sX} \bracket*{\norm{\pi_0(x) -
    \pi_0(\Phi(x))}_1^2}$.

\subsection{Failure of Pythagorean improvement for minimax projections}
\label{sec:failure-minimax}

In previous sections, we analyzed improvement for a single, fixed
Bregman divergence $\sfB_F$. We now consider a family of Bregman
divergences generated by a convex set of Legendre generators, $\cF$. A
natural robust solution in this setting is defined by the minimax
objective:
\[
\min_{\pi \in \Pi \cap \Ccoh} \max_{F \in \cF}\sfB_F(\pi \parallel \pi_0).
\]
This optimization problem is convex, since each Bregman divergence
$\sfB_F(\cdot \parallel \pi_0)$ is convex in its first argument and
the pointwise maximum of convex functions remains convex.  However, as
the following result demonstrates, this robust minimax formulation
does not, in general, guarantee a Pythagorean improvement over the
source model $\pi_0$.

\begin{theorem}
\label{th:minimax_failure_function}
There exist finite sets $\sX, \sY$, a closed convex model class
$\Pi \subseteq \Piall$, an involution
$\Phi$ on $\sX$ defining a nontrivial coherence set $\Ccoh$, a source
model $\pi_0 \colon \sX \to \Delta(\sY)$ with $\pi_0 \notin \Pi$, and
a convex family of Legendre generators $\cF$ such that the unique
minimax solution
\[
  \pi_{mm}^*
  = \argmin_{\pi \in \Pi \cap \Ccoh} \max_{F \in \cF}\sfB_F(\pi \parallel \pi_0)
\]
violates the Pythagorean improvement condition: there exists
$F \in \cF$ and a feasible reference model $\pi^* \in \Pi\cap\Ccoh$ with
\[
\sfB_F(\pi^* \parallel \pi_{mm}^*)>\sfB_F(\pi^* \parallel \pi_0).
\]
\end{theorem}

\begin{proof}
  We give an explicit construction modulo a parameter $M > 0$
  later selected to be large enough.

Let $\sX = \{x_1, x_2\}$ and $\sY = \{1, 2, 3\}$.  We identify any
model $\pi\colon \sX \to \Delta(\sY)$ with the concatenated vector
$\pi = (\pi(x_1)_1, \pi(x_1)_2, \pi(x_1)_3,\ \pi(x_2)_1, \pi(x_2)_2,
\pi(x_2)_3) \in \Rset^6$.  Define the product simplex 
$\wt \Pi = \Delta_3 \times \Delta_3 \subset \Rset^6$ and introduce the
involution $\Phi$ on $\sX$ that swaps the two inputs:
$\Phi(x_1) = x_2$, $\Phi(x_2) = x_1$. The corresponding coherence set
is $\Ccoh = \{\pi\colon \pi(x_1) = \pi(x_2)\}$.

Define $\Pi$ by
$\Pi = \curl*{(q, q)\colon q = (q_1, q_2, q_3) \in \Delta_3 \
  \text{and} \ q_1\leq \tfrac{1}{2}}$, which is closed and convex.
Thus, we have
$\Pi \cap \Ccoh = \curl*{(q, q) \colon q \in \Delta_3,\ q_1\leq
  \tfrac{1}{2}}$, which is also closed and convex.  
Define $\pi_0 \in \Piall$ by $\pi_0(x_1) = (1, 0, 0)$, and
$\pi_0(x_2) = (0, 1, 0)$.  Note that $\pi_0 \in \Piall$ lies outside
$\Pi$ and outside $\Ccoh$.
Since $\pi_0(x_1)_1 = 1 > \tfrac{1}{2}$, any coherent model
$\pi = (q, q)$ with $q_1 = 1$ would fail the constraint
$q_1 \leq \tfrac{1}{2}$.

Let $M > 0$ be a large scalar. Define two quadratic Legendre
generators with diagonal Hessians $A_1, A_2 \in \Rset^{6\times6}$:
$A_1 = \diag(M, 1, 1, \;1, 1, 1)$, and
$A_2 = \diag(1, 1, 1, \;1, M, 1)$.
For $k \in \{1, 2\}$ set
\[
F_k(\pi) = \tfrac{1}{2}\pi^\top A_k\pi \quad \text{and therefore} \quad
\sfB_{F_k}(p \parallel  q) = \tfrac{1}{2}(p-q)^\top A_k(p-q),
\]
and define $\cF = \conv(\{F_1, F_2\})$.

A feasible coherent model in $\Pi \cap \Ccoh$ must be of the form
$\pi = (q, q)$ with $q = (q_1, q_2, q_3) \in \Delta_3$ and
$q_1 \leq \tfrac{1}{2}$.  For fixed $q_1, q_2$ the coordinate $q_3$
contributes $\tfrac{1}{2}(q_3-0)^2$ to both
$\sfB_{F_1}(\pi \parallel \pi_0)$ and
$\sfB_{F_2}(\pi \parallel \pi_0)$ (since
$\pi_0(x_1)_3 = \pi_0(x_2)_3 = 0$), so any minimizer of the minimax
objective takes $q_3 = 0$.  With $q_3 = 0$ we have $q_1 + q_2 = 1$ and
$q_1, q_2 \in [0, 1]$, with the extra feasible restriction
$q_1 \leq \tfrac{1}{2}$.

Let $\pi = (q_1, 1 - q_1, 0, q_1, 1-q_1, 0)$ and $\pi_0 = (1, 0, 0, 0, 1, 0)$.
We compute the two costs:
\[
\begin{aligned}
\sfB_{F_1}(\pi \parallel \pi_0)
& =  \tfrac{1}{2}\Big(M(q_1-1)^2+((1 - q_1) - 0)^2\Big)
   +\tfrac{1}{2}\Big((q_1-0)^2+((1 - q_1) - 1)^2\Big)\\
& =  \tfrac{1}{2}\Big(M(q_1-1)^2 + (1 - q_1)^2 + q_1^2 + (1 - q_1 - 1)^2\Big)\\
& =  \tfrac{1}{2}\Big((M+1)(q_1 - 1)^2 + 2q_1^2\Big) \eqqcolon f_1(q_1),\\[4pt]
\sfB_{F_2}(\pi \parallel \pi_0)
& =  \tfrac{1}{2}\Big((q_1-1)^2+((1 - q_1)-0)^2\Big)
   +\tfrac{1}{2}\Big((q_1-0)^2+M((1 - q_1)-1)^2\Big)\\
& =  \tfrac{1}{2}\Big((q_1-1)^2 + (1 - q_1)^2 + q_1^2 + M(-q_1)^2\Big)\\
& =  \tfrac{1}{2}\Big(2(q_1-1)^2 + (M + 1)q_1^2\Big) \eqqcolon f_2(q_1).
\end{aligned}
\]

We seek $\argmin_{q_1 \in [0, 1/2]} \max(f_1(q_1), f_2(q_1))$.
Assuming $M>1$, the two functions are equal when
$(M-1)(q_1-1)^2 = (M-1)q_1^2$, which implies $(q_1-1)^2 = q_1^2$, so
$-2q_1+1 = 0$, i.e., $q_1 = \tfrac{1}{2}$.  This candidate is feasible
for $\Pi\cap\Ccoh$.  The derivatives at $q_1 = \tfrac{1}{2}$ are
$f_1'(\tfrac{1}{2}) = (M+1)(-\tfrac{1}{2}) + 2(\tfrac{1}{2}) =
\tfrac{1-M}{2} < 0$ and
$f_2'(\tfrac{1}{2}) = 2(-\tfrac{1}{2}) + (M+1)(\tfrac{1}{2}) =
\tfrac{M-1}{2} > 0$.  Since $f_1$ is decreasing and $f_2$ is
increasing at $q_1 = \tfrac{1}{2}$, this point is the unique minimizer
of $\max(f_1, f_2)$ (both $f_1$ and $f_2$ are strictly convex
quadratic functions) on $[0, \tfrac{1}{2}]$. Thus, we have
$\pi_{mm}^* = \paren*{\tfrac{1}{2}, \tfrac{1}{2}, 0, \ \tfrac{1}{2},
\tfrac{1}{2}, 0}$.

Choose the feasible reference
$\pi^* = (0, 1, 0, \ 0, 1, 0) \in \Pi\cap\Ccoh$.  We can compute the
difference
$d = \pi^* - \pi_{mm}^* = (-\tfrac{1}{2}, \tfrac{1}{2}, 0,
-\tfrac{1}{2}, \tfrac{1}{2}, 0)$.  Using
$A_2 = \diag(1, 1, 1, \;1, M, 1)$,
\begin{align*}
\sfB_{F_2}(\pi^* \parallel \pi_{mm}^*)
& =  \tfrac{1}{2}\, d^\top A_2 d 
=  \tfrac{1}{2}\Big( 1(-\tfrac{1}{2})^2+1(\tfrac{1}{2})^2 + 1(-\tfrac{1}{2})^2 + M(\tfrac{1}{2})^2 \Big)
=  \tfrac{1}{2}\Big(\tfrac{3+M}{4}\Big) = \tfrac{3+M}{8}.
\end{align*}
On the other hand, using $\pi_0 = (1, 0, 0, 0, 1, 0)$,
$\pi^*-\pi_0  =  (-1, 1, 0, 0, 0, 0)$ so
\[
\sfB_{F_2}(\pi^* \parallel \pi_0)
 =  \tfrac{1}{2}\Big( 1(-1)^2 + 1(1)^2 + 1(0)^2 + 1(0)^2 + M(0)^2 + 1(0)^2 \Big)
 =  \tfrac{1}{2}(1 + 1) = 1.
\]
Thus, we have
\[
\sfB_{F_2}(\pi^* \parallel \pi_{mm}^*) - \sfB_{F_2}(\pi^* \parallel \pi_0)
 =  \Big(\tfrac{3+M}{8}\Big) - 1
 =  \tfrac{M-5}{8}.
\]
Therefore, for any choice of $M > 5$ we
have $\tfrac{M - 5}{8} > 0$, that is,
\[
\sfB_{F_2}(\pi^* \parallel \pi_{mm}^*) > \sfB_{F_2}(\pi^* \parallel \pi_0).
\]
Thus the minimax solution $\pi_{mm}^*$ does not satisfy the
Pythagorean-improvement inequality for the divergence $F_2 \in \cF$,
which completes the proof.
\end{proof}

\subsection{Uniform improvement guarantees and impossibility results}
\label{sec:uniform-improvement-guarantee}

Given the negative result of the previous subsection, a natural
question is whether the single $F$ improvement guarantees can ever be
strengthened to hold \emph{uniformly} over an entire family of Bregman
divergences $\sfB_F$. That is, can we find a single, universal model
$\h \pi$ that is guaranteed to be an improvement over the baseline
$\pi_0$ for every divergence in a convex family $\cF$?

The following theorem shows that such a universal improvement
guarantee is indeed attainable under some strong assumptions: the
Bregman divergences generated by $\cF$ must all be jointly convex in
their arguments and a particular \emph{orbit model} must be in $\Pi$.

\begin{theorem}[Improvement via orbitwise averaging for jointly
  convex Bregman divergences]
\label{th:jointly_convex_orbit_average}
Let $\Pi$ be a closed and convex family of models
$\pi\colon \sX \to \Delta(\sY)$, and let $\Phi\colon \sX \to \sX$ be
an involution. Let $\pi^* \in \Pi \cap \Ccoh$ be a coherent reference
model. Let $\pi_0 \in \Pi$ be a baseline model.
Let $\cF$ be a convex family of Legendre generators such that, for
every $F \in \cF$, the Bregman divergence $\sfB_F$ is jointly convex.
Define the orbitwise averaged/coherent model
\[
  \h \pi_{\orbit}(x)
  = \frac{\Pr[x] \pi_0(x) + \Pr[\Phi(x)] \pi_0(\Phi(x))}{\Pr[x] + \Pr[\Phi(x)]}.
\]
Assume either that $\h \pi_{\orbit} \in \Pi$, or that $\Pi$ is
closed under the swap $\pi \mapsto \pi \circ \Phi$.  Then, for all
$F \in \cF$, the following improvement inequality holds:
\[
  \E_{x \sim \sD_\sX}\bracket*{\sfB_F(\pi^*(x) \parallel \h \pi_{\orbit}(x))}
    \leq \E_{x \sim \sD_\sX}\bracket*{\sfB_F(\pi^*(x) \parallel \pi_0(x))}.
\]
\end{theorem}

\begin{proof}
  Partition $\sX$ into orbits under $\Phi$, i.e., sets
  $\{x, \Phi(x)\}$. Fix a single orbit with weights $w_x = \Pr[x]$ and
  $w_{\Phi(x)} = \Pr[\Phi(x)]$. Since $\pi^*$ is coherent,
  $\pi^*(x) = \pi^*(\Phi(x))$ on the orbit.
By joint convexity of $\sfB_F$, we have
\begin{align*}
\sfB_F\paren*{\pi^*(x) \,\Big\|\, \h \pi_{\orbit}(x)}
& = \sfB_F\paren*{\frac{w_x + w_{\Phi(x)}}{w_x + w_{\Phi(x)}} \pi^*(x) \,\bigg\|\, \frac{w_x \pi_0(x) + w_{\Phi(x)} \pi_0(\Phi(x))}{w_x + w_{\Phi(x)}}}\\
& \leq \frac{w_x}{w_x + w_{\Phi(x)}} \sfB_F(\pi^*(x) \parallel \pi_0(x))
  + \frac{w_{\Phi(x)}}{w_x + w_{\Phi(x)}} \sfB_F(\pi^*(x) \parallel \pi_0(\Phi(x))).
\end{align*}
Multiply both sides by $(w_x + w_{\Phi(x)})$ and sum over all orbits:
\[
\sum_{\text{orbits}} \paren*{w_x + w_{\Phi(x)}} \, \sfB_F(\pi^*(x) \parallel \h \pi_{\orbit}(x))
\leq \sum_{\text{orbits}} w_x \sfB_F(\pi^*(x) \parallel \pi_0(x))
+ w_{\Phi(x)} \sfB_F(\pi^*(x) \parallel \pi_0(\Phi(x))).
\]
Finally, by assumption, $\h \pi_{\orbit} \in \Pi$ and is
feasible. Taking expectation over $x \sim \sD_\sX$ yields
\[
  \E_{x \sim \sD_\sX} \bracket*{\sfB_F(\pi^*(x)  \parallel  \h \pi_{\orbit}(x))}
  \leq \E_{x \sim \sD_\sX}\bracket*{\sfB_F(\pi^*(x)  \parallel  \pi_0(x))},
\]
and completes the proof.
\end{proof}

Since the Mahalanobis family is a family of jointly convex Bregman
divergences, the following is a direct consequence of the theorem.

\begin{corollary}[Improvement for the Mahalanobis family]
\label{cor:mahalanobis}
Let $\Pi$, $\Phi$, $\pi_0$, and $\pi^*$ be as in
Theorem~\ref{th:jointly_convex_orbit_average}.
Let $\cM$ be a convex set of symmetric positive semidefinite
matrices. For each $\sfM \in \cM$, define the quadratic (Mahalanobis)
Bregman divergence
\[
\sfB_\sfM(\sfp \parallel \sfq) = \tfrac{1}{2} (\sfp - \sfq)^\top \sfM (\sfp - \sfq).
\]
Then, the orbitwise averaged coherent model
\[
  \pi_\orbit(x)
  = \frac{\Pr[x] \pi_0(x) + \Pr[\Phi(x)] \pi_0(\Phi(x))}{\Pr[x] + \Pr[\Phi(x)]}
\]
satisfies
\[
\E_{x \sim \sD_\sX}\big[ \sfB_\sfM(\pi^*(x) \parallel \pi_\orbit(x)) \big] \;\le\; 
\E_{x \sim \sD_\sX}\big[ \sfB_\sfM(\pi^*(x) \parallel \pi_0(x)) \big],
\qquad \forall \sfM \in \cM,
\]
provided either $\pi_\orbit \in \Pi$ or $\Pi$ is convex
and closed under the swap $\pi \mapsto \pi \circ \Phi$.
\end{corollary}

\begin{proof}
  For each fixed $\sfM \succeq 0$, the Mahalanobis divergence
  $\sfB_\sfM(\sfp\|\sfq) = \frac{1}{2} (\sfp-\sfq)^\top \sfM (\sfp-\sfq)$ is
  jointly convex in $(\sfp,\sfq)$. The set $\cM$ is convex, so the
  convex family $\{\sfB_\sfM\}_{\sfM \in \cM}$ consists of jointly convex
  Bregman divergences.  The result then follows directly from
  Theorem~\ref{th:jointly_convex_orbit_average}.
\end{proof}

The previous result showed that the orbitwise projection
$\pi_{\orbit}$ guarantees universal improvement over a family of
jointly convex Bregman divergences, provided it is a feasible model.
A natural question arises: if $\pi_{\orbit}$ is not feasible, that is,
$\pi_{\orbit} \notin \Pi$, could some \emph{other} coherent model
$\h \pi \in \Pi \cap \Ccoh$ take its place and offer the same
universal guarantee?

The following theorem provides a strong negative answer. It
establishes that the universal improvement property is so restrictive
that it uniquely identifies the orbit projection as the only possible
solution. Consequently, if the orbit projection itself is not a
permissible model, no such universal improver can exist.

\begin{theorem}[Impossibility when the orbit projection is infeasible]
\label{th:impossibility-orbit-infeasible}
Let $\Phi \colon \sX \to \sX$ be an involution, and let $\Ccoh$ denote
the subset of coherent models $\pi$ satisfying
$\pi(x) = \pi(\Phi(x))$ for all $x \in \sX$.  
Let $\pi_0$ be a baseline model, and denote by $\pi_{\orbit}$ the
(weighted) Euclidean projection of $\pi_0$ onto $\Ccoh$, that is,
\[
  \pi_{\orbit}
  = \argmin_{\pi \in \Ccoh}
    \E_{x \sim \sD_\sX}\bracket*{
      \|\pi(x) - \pi_0(x)\|_2^2 }.
\]
Let $\cF$ be a convex family of Legendre generator functions that
includes all quadratic generators
$F_\sfM(\sfp) = \tfrac{1}{2} \sfp^\top \sfM \sfp$, 
for symmetric matrices $\sfM$ whose span is the full space of symmetric
matrices on $\Rset^{|\sY|}$.
Assume that $\h \pi \in \Pi \cap \Ccoh$ satisfies the following
\emph{universal improvement} property:
\begin{equation}
\label{eq:universal-improvement}
  \forall \sD_\sX,\; \forall F \in \cF,\;
  \forall \pi^* \in \Ccoh, \qquad
  \E_{x \sim \sD_\sX}\!\bracket*{
      \sfB_F(\pi^*(x) \parallel \h \pi(x))}
  \leq
  \E_{x \sim \sD_\sX}\!\bracket*{
      \sfB_F(\pi^*(x) \parallel \pi_0(x))}.
\end{equation}
Then, necessarily $\h \pi = \pi_{\orbit}$.  
In particular, if $\pi_{\orbit} \notin \Pi$, no
$\h \pi \in \Pi \cap \Ccoh$ can satisfy
\eqref{eq:universal-improvement}.
\end{theorem}

\begin{proof}
Since inequality~\eqref{eq:universal-improvement} must hold for all
distributions $\sD_\sX$, we may restrict our attention to
distributions supported on a single orbit
$\cO = \curl*{x, \Phi(x)}$.  
For any nonnegative weights $w_x, w_{\Phi(x)}$ (not both zero),
consider the distribution $\sD_\sX$ that assigns mass proportional to
these weights on $x$ and $\Phi(x)$.  
Then, for every $F \in \cF$ and every coherent $\pi^*$ 
(so that $\pi^*(x) = \pi^*(\Phi(x)) = \sfp$),
the condition~\eqref{eq:universal-improvement} reduces to the
single-orbit inequality:
\begin{align}
\label{eq:local-ineq}
  w_x \sfB_F(\sfp \parallel \h \pi(x))
  + w_{\Phi(x)} \sfB_F(\sfp \parallel \h \pi(\Phi(x)))  
  &\leq
    w_x \sfB_F(\sfp \parallel \pi_0(x))
    + w_{\Phi(x)} \sfB_F(\sfp \parallel \pi_0(\Phi(x))).
\end{align}
Since $\h \pi$ is coherent, 
$\h \pi(x) = \h \pi(\Phi(x)) = \sfq$.
Now take $F = F_\sfM$ to be a quadratic generator
$F_\sfM(\sfp) = \tfrac{1}{2} \sfp^\top \sfM \sfp$,
for which $\sfB_{F_\sfM}(\sfp \parallel \sfq)
= \tfrac{1}{2} (\sfp - \sfq)^\top \sfM (\sfp - \sfq)$.  
After normalizing by the total weight 
$w_x + w_{\Phi(x)}$, inequality~\eqref{eq:local-ineq} becomes
\[
  (\sfp - \sfq)^\top \sfM (\sfp - \sfq)
  \leq
  w'_x (\sfp - \pi_0(x))^\top \sfM (\sfp - \pi_0(x))
  + w'_{\Phi(x)} (\sfp - \pi_0(\Phi(x)))^\top \sfM (\sfp - \pi_0(\Phi(x))),
\]
where $w'_x$ and $w'_{\Phi(x)}$ are the normalized weights.  
This inequality must hold for every $\sfp \in \Delta(\sY)$ and for
every symmetric matrix $\sfM$ from a family whose span is the full
space of symmetric matrices.

The terms quadratic in $\sfp$ (of the form $\sfp^\top \sfM \sfp$)
cancel on both sides, leaving an inequality linear in $\sfp$.  
A linear inequality holding for all $\sfp \in \Delta(\sY)$ and for a
spanning set of matrices $\sfM$ implies equality, since any deviation
would be detected by some choice of $\sfM$.  
Evaluating this equality at two points $\sfp_1, \sfp_2$ and
subtracting gives
$\sfM(\pi_{\orbit}(x) - \sfq)$ orthogonal to all difference vectors
$\sfp_1 - \sfp_2$.  
These difference vectors span the subspace of vectors with zero sum of
components, so $\sfM(\pi_{\orbit}(x) - \sfq)$ must be a constant
multiple of $\mathbf{1}$.  
Because this must hold for a spanning set of $\sfM$, the only
possibility is $\pi_{\orbit}(x) - \sfq = \vec{0}$.  
Hence,
\[
  \sfq
  = \frac{w_x \pi_0(x) + w_{\Phi(x)} \pi_0(\Phi(x))}
          {w_x + w_{\Phi(x)}},
\]
which is precisely $\pi_{\orbit}(x)$, the orbitwise projection of
$\pi_0$ onto $\Ccoh$.

Since the argument applies independently to each orbit (because we may
choose $\sD_\sX$ supported on any orbit), we conclude that
$\h \pi(x) = \pi_{\orbit}(x)$ for all $x$, hence
$\h \pi = \pi_{\orbit}$.  
If $\pi_{\orbit} \notin \Pi$, no model in $\Pi \cap \Ccoh$ can satisfy
\eqref{eq:universal-improvement}.
\end{proof}

We now turn to the second key condition from our positive result:
joint convexity. The following theorem demonstrates that if the family
of divergences generated by $\cF$ is not uniformly jointly convex, the
guarantee of universal improvement necessarily breaks down, regardless
of the feasibility of any model, even when $\cF$ includes all
quadratic generators.

\begin{theorem}[Impossibility for Non-Jointly Convex Divergences]
  \label{th:impossibility-non-jointly-convex}
  Let $\cF$ be a convex family of Legendre generator functions that
  includes all quadratic generators. Suppose there exists at least one
  generator $F \in \cF$ for which the corresponding Bregman divergence
  $\sfB_F$ is \emph{not jointly convex} in its two arguments.  Then,
  no coherent model $\h \pi \in \Ccoh$ can satisfy the universal
  improvement property.  That is, for any candidate coherent model
  $\h \pi \in \Ccoh$, there exist a baseline model $\pi_0$, a
  distribution $\sD_\sX$, and a coherent target $\pi^* \in \Ccoh$ such
  that
\[
  \E_{x \sim \sD_\sX} \bracket*{
      \sfB_F(\pi^*(x) \parallel \h \pi(x))}
  > \E_{x \sim \sD_\sX} \bracket*{
      \sfB_F(\pi^*(x) \parallel \pi_0(x))}.
\]
\end{theorem}

\begin{proof}
  We construct an explicit counterexample leveraging the specific
  failure of joint convexity.  By assumption, there exist points
  $\sfq_1, \sfq_2$ in the domain of $F$, a target $\sfp^*$, and a
  weight $\lambda \in (0,1)$ such that Jensen's inequality is
  reversed:
\begin{equation}
\label{eq:reversed-jensen}
  \sfB_F(\sfp^* \parallel \lambda \sfq_1 + (1 - \lambda)\sfq_2)
  >
  \lambda \sfB_F(\sfp^* \parallel \sfq_1)
  + (1 - \lambda) \sfB_F(\sfp^* \parallel \sfq_2).
\end{equation}
We map this inequality to a learning setup with an involution
$\Phi \colon \sX \to \sX$ and an orbit $\cO = \curl*{x, \Phi(x)}$.
Let the distribution $\sD_\sX$ place masses proportional to the
weights $\lambda$ and $1 - \lambda$: $w_x = \lambda$ and
$w_{\Phi(x)} = 1 - \lambda$.  Define the baseline model on this orbit
as $\pi_0(x) = \sfq_1$ and $\pi_0(\Phi(x)) = \sfq_2$.  Let the
coherent target model be constant on the orbit:
$\pi^*(x) = \pi^*(\Phi(x)) = \sfp^*$.

The orbitwise projection of $\pi_0$ onto $\Ccoh$ under these weights is
\[
  \pi_{\orbit}(x)
  = \frac{w_x \pi_0(x) + w_{\Phi(x)} \pi_0(\Phi(x))}
         {w_x + w_{\Phi(x)}}
  = \lambda \sfq_1 + (1 - \lambda) \sfq_2.
\]
The expected Bregman loss of this orbit model under the constructed
distribution is
\[
  \E[\sfB_F(\pi^* \parallel \pi_{\orbit})]
  = \sfB_F(\sfp^* \parallel \lambda \sfq_1 + (1 - \lambda) \sfq_2),
\]
while the expected loss of the baseline model is
\[
  \E[\sfB_F(\pi^* \parallel \pi_0)]
  = \lambda \sfB_F(\sfp^* \parallel \sfq_1)
    + (1 - \lambda) \sfB_F(\sfp^* \parallel \sfq_2).
\]
By inequality~\eqref{eq:reversed-jensen}, we have
\[
  \E[\sfB_F(\pi^* \parallel \pi_{\orbit})]
  > \E[\sfB_F(\pi^* \parallel \pi_0)].
\]
Thus, even the orbitwise projection, the best possible coherent
improver in the jointly convex case, performs strictly worse than the
baseline when joint convexity fails.

For the specific baseline $\pi_0$ constructed above, the corresponding
orbit projection is $\pi_{\orbit}$. By
Theorem~\ref{th:impossibility-orbit-infeasible}, this implies that the
universal improvement property fails for any candidate $\h \pi$.
\end{proof}
Taken together, our impossibility results demonstrate that a universal
improvement model can exist only under restrictive conditions, for
example when the family of Bregman divergences is jointly convex
and the orbitwise projection of the baseline is a feasible model.

\section{Two-Step Coherence Projection}
\label{sec:two-step-projection}

In the previous section, we studied the \emph{direct projection}
method, where the baseline $\pi_0$ is projected directly onto the set
$(\Ccoh \cap \Pi)$ to obtain a coherent conditional distribution
$\h \pi \in \Pi$. Here, we consider an alternative \emph{two-step
 projection} (or \emph{double-projection}) procedure: First, $\pi_0$
is projected onto $\Ccohdag$, yielding a coherent function $\ov
\pi$. Then, $\ov \pi$ is projected onto $(\Ccoh \cap \Pi)$ to produce
the final solution $\hh \pi$.

For the definition of our \emph{two-step projection} in this section,
we require Bregman generator functions $F$ defined over a closed set
$\sK$ containing the non-negative orthant $\Rset_{\geq 0}^d$ with its
interior encompassing the entire positive orthant $\Rset_{++}^d$.
Most commonly used Bregman divergences satisfy this property. Even
when they do not, it is often possible to extend the generator
function to a convex function with such a domain.
For instance, generators such as the squared Euclidean distance and,
more generally, Mahalanobis functions are defined on all of
$\Rset^d$. Other examples, including the negative entropy (which
induces the unnormalized relative entropy and yields the $\KL$
divergence by restriction), the Itakura–Saito generator, and many
$\alpha$- and $\beta$-entropy functions, are defined over the full
positive orthant $\Rset_{++}^d$. While it is mathematically possible
to construct generators defined only on a subset of the orthant, for
example, on the region satisfying $a < \sum_{i = 1}^d p_i < b$ for
some $a, b > 0$, such choices are uncommon in the literature.
Throughout this section, we will assume generator functions $F$ with a
domain satisfying this property.

We begin by proving that, remarkably, the direct and two-step
projections \emph{coincide} for a very broad class of Bregman
divergences (Section~\ref{sec:single=two}). This class includes not
only the squared Euclidean distance and $\KL$ divergence, but all common
\emph{separable} (e.g., Itakura-Saito, $\alpha$-/$\beta$-divergences)
and \emph{quadratic} (e.g., Mahalanobis) divergences.

The equivalence leverages a general property of Bregman projections
onto affine sets, which applies directly to our unrestricted coherence
set $\Ccohddag$ that is a linear subspace. A key part of our proof is
showing that for this entire class of divergences, the first-step
projection $\ov\pi$, which projects the non-negative $\pi_0$ onto the
\emph{linear subspace} $\Ccohddag$, is guaranteed to be non-negative
(Subsection~\ref{sec:pyth-equality}). This ensures $\ov \pi$ also lies
in the non-negative cone $\Ccohdag$, which is the crucial condition
for the equivalence to hold.

As a consequence, the double-step projection enjoys the same
improvement guarantee as the single-step one when they
coincide. Nonetheless, we derive an alternative improvement guarantee
for the two-step projection solution $\hh \pi$
(Subsection~\ref{sec:two-step-improvement}), which will serve as a key
technical tool in establishing several subsequent results. A central
component of this analysis is a closed-form expression for $\ov\pi$
(the Bregman centroid).
The result of Subsection~\ref{sec:two-step-improvement} first enables
us to derive explicit performance guarantees in terms of the Hellinger
distance for the special case of the unnormalized relative entropy
(Subsection~\ref{sec:relative-entropy-guarantees}).
Moreover, we use it to establish a striking maximin property: the
two-step projection, and therefore also the single-step projection
when equivalent, coincides with selecting the best coherent model
against the worst-case reference distribution~$\pi^*$
(Subsection~\ref{sec:maximin-properties}).

As before, we assume that $\Phi$ is an involution, although all
results extend directly to the more general setting in which $\Phi$
admits arbitrary finite orbits.

\subsection{Pythagorean theorem equality}
\label{sec:pyth-equality}

Our comparison of the direct- and two-step projections makes use of
the following equality case of the Pythagorean theorem, which holds
for any affine subspace, for which we give a general proof using
Fenchel duality. For the $\KL$ divergence, this was proven for a
general family of \emph{linear models} (in the sense of
\citeauthor{Csiszar1975}) \citep{Csiszar1975,CsiszarMatus2003}.

\begin{theorem}[Pythagorean Theorem equality]
\label{th:pythagorean-equality}
Let $F \colon \sK \to \Rset$ be of Legendre type. Let
$\sC \subseteq \sK \subseteq (\Rset^d)^\sX$ be a non-empty affine
subspace, and let $\pi_0 \in \Omega$. Let $\ov \pi$ be the
$\sfB_F$-projection of $\pi_0$ on $\sC$. Then, for any $\pi \in \sC$,
\[
 \sfB_F(\pi \parallel \pi_0)
 = \sfB_F(\pi \parallel \ov \pi) + \sfB_F(\ov \pi \parallel \pi_0).
\]
\end{theorem}

\begin{proof}
 Since $F(\pi_0)$ and $\tri{\nabla F(\pi_0), \pi_0}$ are constant
 terms in the definition of the Bregman divergence, the Bregman
 projection $\ov \pi$ is a solution of the following convex
 optimization problem:
\[
 \min_{\pi \in \sC} \curl*{F(\pi) - \tri{\nabla F(\pi_0), \pi}}.
\]
Define the indicator function $I_\sC(\pi)$ of the affine set $\sC$ by
$I_\sC(\pi) = 0$ if $\pi$ is in $\sC$, and $+\infty$ otherwise. Then, the
problem can be rewritten equivalently as
\[
\min_{\pi \in \Omega} \curl*{G(\pi) + I_\sC(\pi)},
\]
where $G(\pi) = F(\pi) - \tri*{\nabla F(\pi_0), \pi}$. Since
$\dom(G) = \Omega$, $I_\sC$ is convex and lower semi-continuous, and
$\dom(G) \cap \dom(I_\sC) = \sC \neq \emptyset$, the Fenchel duality
theorem \citep{Rockafellar1996} applies with strong duality:
\begin{align*}
 \min_{\pi \in \Omega} G(\pi) + I_\sC(\pi)
 = \max_{\theta \in \Omega^*} \curl{- G^*(-\theta) - I_\sC^*(\theta)},
\end{align*}
where the conjugate function $G^*$ is given by
$G^*(\phi) = F^*(\phi + \nabla F(\pi_0))$ via standard conjugate
function calculus rules and the conjugate of the indicator function,
$I_\sC^*$, is the support function
$\sigma_\sC(\theta) = \sup_{\pi \in \sC} \tri{\theta, \pi}$.
Let $\theta^*$ denote the solution of the dual problem, then,
we can write
\begin{equation}
\label{eq:strong_duality}
G(\ov \pi) + I_\sC(\ov \pi) = -G^*(-\theta^*) - I_\sC^*(\theta^*)
\Leftrightarrow G(\ov \pi) + G^*(-\theta^*) + I_\sC^*(\theta^*) = 0,
\end{equation}
since $\ov \pi$ is in $\sC$ ($I_\sC(\ov \pi) = 0$).
Now, by the Fenchel-Young inequality, we have
\[
\Delta_G = G(\ov \pi) + G^*(-\theta^*) - \tri*{\ov \pi, -\theta^*} \geq 0,
\quad \text{and} \quad
\Delta_I = I_\sC(\ov \pi) + I_\sC^*(\theta^*) - \tri*{\ov \pi, \theta^*} \geq 0.
\]
Adding up these two equalities and using $I_\sC(\ov \pi) = 0$ yields:
\[
\Delta_G + \Delta_I
= \bracket*{G(\ov \pi) + G^*(-\theta^*) + I_\sC^*(\theta^*)}
- \bracket*{\tri*{\ov \pi, - \theta^*} + \tri*{\ov \pi, \theta^*}}
= G(\ov \pi) + G^*(-\theta^*) + I_\sC^*(\theta^*).
\]
By \eqref{eq:strong_duality}, the right-hand side equals $0$, thus
$\Delta_G + \Delta_I = 0$. Since $\Delta_G \geq 0$ and
$\Delta_I \geq 0$, this implies $\Delta_G = 0$ and $\Delta_I = 0$.
Thus, equality holds in Fenchel-Young for both
$(G, \ov \pi, -\theta^*)$ and $(I_\sC, \ov \pi, \theta^*)$, which is
equivalent to the subgradient conditions
\[
 -\theta^* \in \partial G(\ov \pi) \quad \text{and} \quad
 \theta^* \in \partial I_\sC(\ov \pi).
\]
Since $F$ is Legendre, $G$ is differentiable on $\Omega$, so
$\partial G(\ov \pi) = \curl*{\nabla G(\ov \pi)} = \{\nabla F(\ov \pi)
- \nabla F(\pi_0)\}$. Therefore
\[
-\theta^* = \nabla F(\ov \pi) - \nabla F(\pi_0)
\Rightarrow
\theta^* = \nabla F(\pi_0) - \nabla F(\ov \pi).
\]
Writing $\sC = v + U$ where $U$ is the linear subspace corresponding
to $\sC$, we have $\partial I_\sC(\ov \pi) = U^\perp$. Thus, we have
$\theta^* = \nabla F(\pi_0) - \nabla F(\ov \pi) \in U^\perp$. 
For any $\pi \in \sC$ the difference $\pi - \ov \pi$ is in $U$,
therefore
$\langle \nabla F(\pi_0) - \nabla F(\ov \pi), \pi - \ov \pi \rangle =
0$. This is the Bregman orthogonality condition.
Expanding the two Bregman divergences and using this equality
gives:
\begin{align*}
\sfB_F(\pi \parallel \ov \pi) + \sfB_F(\ov \pi \parallel \pi_0)
& = \bracket*{F(\pi) - F(\ov \pi) - \tri*{\nabla F(\ov \pi), \pi - \ov \pi}} \\
&\quad + \bracket*{F(\ov \pi) - F(\pi_0)
 - \tri*{\nabla F(\pi_0), \ov \pi - \pi_0}} \\
& = F(\pi) - F(\pi_0) - \tri*{\nabla F(\ov \pi),\pi - \ov \pi}
 - \tri*{\nabla F(\pi_0), \ov \pi - \pi_0} \\
& = F(\pi) - F(\pi_0) - \tri*{\nabla F(\pi_0), \pi - \ov \pi}
 - \tri*{\nabla F(\pi_0), \ov \pi - \pi_0} \\
& \quad \text{(by orthogonality, since $\tri*{\nabla F(\ov \pi) - \nabla F(\pi_0), \pi - \ov \pi} = 0$)} \\
& = F(\pi) - F(\pi_0) - \tri*{\nabla F(\pi_0), \pi - \pi_0} \\
& = \sfB_F(\pi\parallel\pi_0),
\end{align*}
This completes the proof.
\end{proof}

\subsection{Comparison of direct and two-step projections}
\label{sec:single=two}

In this section, we show that the direct and two-step projections
coincide for a broad family of Bregman divergences, relying on the
fact that the coherent set $\Ccohddag$ is linear.
We cannot apply the result of
Theorem~\ref{th:pythagorean-equality} directly in our two-step
projection setting as $\Ccohdag$, the projection set for the first
stage, is not affine: $\Ccohdag$ a closed convex cone but is not
affine. However, the theorem can be applied to $\Ccohddag$, which is
linear. To leverage this, we first show that for a broad family of
Bregman divergences including most divergences typically used in
practice, the projection of a point in the simplex over $\Ccohddag$ is
always in $\Ccohdag$. Thus, projection over $\Ccohddag$ becomes
equivalent to projecting on $\Ccohdag$.

\begin{theorem}[Non-Negativity of Unconstrained Projection]
\label{th:projection-ddag-general}
Let $\pi_0 \in \Piall$, and let $\sfB_F$ be a Bregman divergence
generated by a Legendre function $F$ satisfying the assumption that
its domain is a closed convex set $\sK \supseteq \Rset_{\geq 0}^d$.

If $F$ is either:
\begin{enumerate}
\item Quadratic: $F(\sfp) = \frac{1}{2}\sfp^T \sfA \sfp + b^T \sfp$ for
 some $\sfA \succ 0$, or
\item Separable: $F(\sfp) = \sum_{i=1}^d f(\sfp_i)$ for some 1D convex
 function $f$.
\end{enumerate}
then the expected $\sfB_F$-projection of $\pi_0$ onto the linear
subspace $\Ccohddag$, denoted $\ov \pi$, is necessarily in the
non-negative cone $\Ccohdag$.
\end{theorem}

\begin{proof}
The projection $\ov \pi$ is the function $\pi \in \Ccohddag$ that
minimizes the expected loss
$L(\pi) = \E_{x \sim \sD_\sX}\bracket*{\sfB_F(\pi(x) | \pi_0(x))}$.
We must show that $\ov \pi(x) \ge 0$ for all $x$.

The problem decouples by orbit. For any orbit $\{x, \Phi(x)\}$,
the minimizer $\ov \pi(x) = \ov \pi(\Phi(x)) = c^*$ is the
vector $c \in \Rset^d$ that minimizes the orbit's loss:
\[
L_x(c)
= \P[x] \sfB_F(c | \pi_0(x)) + \P[\Phi(x)] \sfB_F(c | \pi_0(\Phi(x))).
\]
Let $Z = \P[x] + \P[\Phi(x)]$. Minimizing $L_x(c)$ is
equivalent to minimizing $\frac{1}{Z} L_x(c)$, which is
in the form required by \cref{lemma:right-bregman-centroid}
with $p = 2$, $\lambda_1 = \P[x]/Z$, $\lambda_2 = \P[\Phi(x)]/Z$,
$\sfq_1 = \pi_0(x)$, and $\sfq_2 = \pi_0(\Phi(x))$.
Let $\sfp = \pi_0(x)$ and $\sfq = \pi_0(\Phi(x))$.
By \cref{lemma:right-bregman-centroid}, the unique minimizer $c^*$
(the right Bregman centroid) satisfies:
\[
 \nabla F(c^*) = \frac{\P[x] \nabla F(\sfp)
 + \P[\Phi(x)] \nabla F(\sfq)}{\P[x] + \P[\Phi(x)]}
\]
Since $\pi_0 \in \Piall$, we know $\sfp, \sfq \ge 0$ (element-wise).
We now show $c^* \geq 0$ for both classes of generators.

\paragraph{1. Quadratic Generators.}
Let $F(\sfp) = \frac{1}{2}\sfp^T \sfA \sfp + b^T \sfp$. Then
$\nabla F(\sfp) = \sfA \sfp + b$. The solution $c^*$ satisfies:
\begin{align*}
\sfA c^* + b & = \frac{\P[x] (\sfA \sfp + b) + \P[\Phi(x)] (\sfA \sfq + b)}{\P[x] + \P[\Phi(x)]} \\
\sfA c^* + b & = \frac{\P[x] \sfA \sfp + \P[\Phi(x)] \sfA \sfq}{\P[x] + \P[\Phi(x)]} + b \\
\sfA c^* & = \sfA \paren*{ \frac{\P[x] \sfp + \P[\Phi(x)] \sfq}{\P[x] + \P[\Phi(x)]} }.
\end{align*}
Since $\sfA$ is invertible ($\sfA \succ 0$), this gives
$c^* = \frac{\P[x] \sfp + \P[\Phi(x)] \sfq}{\P[x] + \P[\Phi(x)]}$. This is
the weighted arithmetic mean of $p$ and $q$. Since $p, q \ge 0$,
their weighted average $c^*$ must also be non-negative.

\paragraph{2. Separable Generators.}
Let $F(\sfp) = \sum_{i=1}^d f(\sfp_i)$. The problem decouples by
component $i$. The solution $c_i^*$ must satisfy
$f'(c_i^*) = \frac{\P[x] f'(\sfp_i) + \P[\Phi(x)] f'(\sfq_i)}{\P[x] + \P[\Phi(x)]}$,
where $\sfp_i, \sfq_i \ge 0$.
We analyze two sub-cases based on the domain $\sK \supseteq \Rset_{\geq 0}^d$.

\begin{itemize}
 \item Case 2a: $f$ is defined on $[0, \infty)$.
 Since $f$ is convex, its derivative $f'$ is non-decreasing.
 As $\sfp_i \ge 0$ and $q_i \ge 0$, we have $f'(\sfp_i) \ge f'(0)$
 and $f'(\sfq_i) \ge f'(0)$.
 Therefore, the weighted average must also be $\ge f'(0)$:
 \[
 f'(c_i^*) = \text{avg}_{\lambda}(f'(\sfp_i), f'(\sfq_i)) \ge f'(0)
 \]
 By the monotonicity of $f'$, this implies $c_i^* \ge 0$.

 \item Case 2b: $f$ is defined on $(0, \infty)$ and steep at 0. This
 corresponds to $\sK = \Rset_{\geq 0}^d$ and
 $\Int(\sK) = \Rset_{++}^d$. Here,
 $\lim_{t \to 0^+} f'(t) = -\infty$. If $\sfp_i, \sfq_i > 0$, the
 argument from Case 2a holds, and $f'(c_i^*)$ is an average of
 finite numbers, $f'(\sfp_i), f'(\sfq_i) > -\infty$. Thus
 $f'(c_i^*) > -\infty$, which implies $c_i^* > 0$.
 
 If $\sfp_i = 0$ (or $\sfq_i = 0$), $\sfp_i$ is on the boundary. The
 divergence $\sfB_f(c_i | 0)$ is defined by extension. As the
 generator $f$ is steep, $\sfB_f(c_i | 0) = +\infty$ for $c_i > 0$
 and $\sfB_f(0 | 0) = 0$. To achieve a finite minimum for
 $L_{x,i}(c_i)$, we must avoid the $+\infty$ penalty. The unique
 minimizer is therefore $c_i^* = 0$.
\end{itemize}
In all cases, for both classes of generators, $c_i^* \ge 0$ for all
$i$. Thus, $\ov \pi(x) = c^* \geq 0$.
Since $\ov \pi \in \Ccohddag$ (by construction) and $\ov \pi \ge 0$
(by proof), $\ov \pi$ is in the intersection
$\Ccohddag \cap \Pialldag = \Ccohdag$.
\end{proof}

The theorem can be generalized to a more abstract rule presented in
the following.

\begin{theorem}[Domain-Closure of Unconstrained Projection]
\label{th:projection-ddag-general-cone}
Let $\sK$ be a convex cone in $\Rset^d$ (e.g., $\Rset_{\geq 0}^d$ or
$\mathbb{S}_{++}^d$), and let $F$ be a Legendre function whose domain
is $\sK$. Let $\Pialldag = (\sK)^{\sX}$ be the set of all functions
mapping to this cone. Let $\pi_0 \in \Pialldag$ (i.e.,
$\pi_0(x) \in \sK$ for all $x$).

If the Bregman divergence $\sfB_F$ has the property that for any
$p, q \in \sK$ and any weights $\lambda_1, \lambda_2 > 0$, the right
Bregman centroid $c^*$ defined by
$\nabla F(c^*) = \frac{\lambda_1 \nabla F(\sfp) + \lambda_2 \nabla
 F(\sfq)}{\lambda_1 + \lambda_2}$ is also in $\sK$, then $\ov \pi$, the
expected $\sfB_F$-projection of $\pi_0$ onto the linear subspace
$\Ccohddag$ is necessarily in $\Ccohdag = \Ccohddag \cap \Pialldag$.
\end{theorem}

\begin{proof}
 By definition, the projection $\ov \pi$ is in the linear subspace
 $\Ccohddag$. We only need to show it is also in $\Pialldag$, which
 means proving $\ov \pi(x) \in \sK$ for all $x$.
As established in the proof of \cref{th:projection-ddag-general}, the projection
$\ov \pi(x)$ for any orbit $\{x, \Phi(x)\}$ is the right Bregman
centroid $c^*$ of $\sfp = \pi_0(x)$ and $\sfq = \pi_0(\Phi(x))$ with weights
$\P[x]$ and $\P[\Phi(x)]$.
By assumption, $\sfp, \sfq \in \sK$. By the theorem's condition on
$\sfB_F$, the resulting centroid $c^* = \ov \pi(x)$ must also be in
$\sK$.
Since this holds for all $x$, $\ov \pi \in (\sK)^{\sX} = \Pialldag$.
Therefore, $\ov \pi \in \Ccohddag \cap \Pialldag = \Ccohdag$.
\end{proof}

This theorem is general and covers the vast majority of divergences
used in machine learning, which satisfy the centroid-closure property

for the non-negative cone $\sK = \Rset_{\geq 0}^d$.
The quadratic class, including squared Euclidean and
Mahalanobis distances, is covered because their centroid is the
weighted arithmetic mean
($c^* = \text{avg}_{\lambda}(\sfp, \sfq)$),
which preserves non-negativity.
The \emph{separable} class is also covered, including the $\KL$
divergence, Itakura-Saito divergence, and most $\alpha$- and
$\beta$-divergences.  Their centroids are component-wise generalized
means (like the geometric mean for KL) or are forced to 0 by the steep
boundary of the generator, thus preserving non-negativity.
The principle also extends to other domains, such as the cone of
positive-definite matrices ($\sK = \mathbb{S}_{++}^d$). Divergences
like the Log-Determinant (Stein's loss) and von Neumann entropy also
satisfy this property, as their centroids, which are forms of harmonic
or geometric matrix means, remain within the positive-definite cone.

The primary requirement for this theorem is the ability to use
\cref{lemma:right-bregman-centroid} to define the centroid $c^*$.
This lemma, in turn, requires the generator $F$ to be of
Legendre type, meaning it must be both strictly
convex and steep.
Therefore, Bregman divergences generated by functions that are
not strictly convex (e.g., $L_1$-based generators) are
not covered. For such functions, the gradient $\nabla F$ is not
invertible, and the minimizer $c^*$ is not guaranteed to be
unique. Similarly, non-steep generators defined on bounded sets
would also fail these conditions.

\begin{theorem}[Equivalence of Direct and Two-Step Projections]
 \label{th:equivalence}
 Assume that the Bregman divergence $\sfB_F$ admits the property
 assumed in Theorem~\ref{th:projection-ddag-general-cone}. 
 Then, for any $\pi_0 \in \Piall$, the direct- and two-step
 projections of $\pi_0$ on $\Pi \cap \Ccoh$ coincide.
\end{theorem}

\begin{proof}
 Let $\h \pi = \Proj_{\Pi \cap \Ccoh}(\pi_0)$ be the direct projection.
 The two-step projection is $\hh \pi = \Proj_{\Pi \cap \Ccoh}(\ov \pi)$,
 where $\ov \pi = \Proj_{\Ccohdag}(\pi_0)$.

 By Theorem~\ref{th:projection-ddag-general-cone}, the property
 assumed on $\sfB_F$ ensures that the projection onto the non-negative
 cone $\Ccohdag$ is identical to the projection onto the linear
 subspace $\Ccohddag$. Thus,
 \[
 \ov \pi = \Proj_{\Ccohdag}(\pi_0) = \Proj_{\Ccohddag}(\pi_0).
 \]
 Since $\ov \pi$ is the projection onto the \emph{affine} set
 $\Ccohddag$, we can apply Theorem~\ref{th:pythagorean-equality}
 for any $\pi \in \Ccohddag$. As $\Pi \cap \Ccoh \subseteq \Ccohddag$,
 this holds for all $\pi \in \Pi \cap \Ccoh$:
\[
 \sfB_F(\pi \parallel \pi_0) = \sfB_F(\pi \parallel \ov \pi)
 + \sfB_F(\ov \pi \parallel \pi_0).
\]
To find the direct projection $\h \pi$, we minimize the left-hand side
over $\pi \in \Pi \cap \Ccoh$. Since $\sfB_F(\ov \pi \parallel \pi_0)$
is a constant, this is equivalent to minimizing the first term:
\[
 \h \pi = \argmin_{\pi \in \Pi \cap \Ccoh} \sfB_F(\pi \parallel \pi_0)
 = \argmin_{\pi \in \Pi \cap \Ccoh} \sfB_F(\pi \parallel \ov \pi).
\]
The right-hand side is precisely the definition of the second step
of the two-step projection, $\hh \pi$. Therefore,
$\h \pi = \hh \pi$.
\end{proof}

\subsection{Improvement guarantees via two step-projection}
\label{sec:two-step-improvement}

In this section, we provide an alternative improvement guarantee for
the double-step projection, which will serve as a key tool for the two
subsequent results.

\begin{theorem}[Coherent two-step Bregman–projection improves the baseline]
\label{th:two-step-bregman-projection}
Let $\Phi \colon \sX \to \sX$ be an involution and let $\Pi$ be a
family of conditional distribution functions
$\pi \colon \sX \to \Delta(\sY)$.  Assume that $F$ is a Legendre
function.  Let $\ov \pi \in \Pialldag$ be the Bregman-projection of
$\pi_0$ onto $\Ccohdag$:
$\ov \pi \in \argmin_{\pi \in \Ccohdag} \sfB_F(\pi \parallel \pi_0)$ and
$\hh \pi$ the Bregman-projection of $\ov \pi$ onto $\Ccoh \cap \Pi$:
$\hh \pi \in \argmin_{\pi \in \Ccoh \cap \Pi} \sfB_F(\pi \parallel
\pi_0)$. Then, $\hh \pi$ satisfies:
\begin{align*}
  \E_{x \sim \sD_\sX}\bracket*{\sfB_F\paren*{\pi^*(x) \parallel \hh \pi(x)}}
  & \leq \E_{x \sim \sD_\sX}\bracket*{\sfB_F\paren*{\pi^*(x) \parallel \pi_0(x)}} \\
  & \quad  - \bracket*{\E_{x \sim \sD_\sX}\bracket*{\sfB_F\paren*{\hh \pi(x) \parallel \ov \pi(x)}}
    + \E_{x \sim \sD_\sX}\bracket*{\sfB_F\paren*{\ov \pi(x) \parallel \pi_0(x)}}}.
\end{align*}
Furthermore, the following upper bound holds:
\begin{align*}
\E_{x \sim \sD_\sX} \bracket*{\sfB_F \paren*{\pi^*(x) \parallel \hh \pi(x)}}
& \leq \E_{x \sim \sD_\sX} \bracket*{\sfB_F\paren*{\pi^*(x) \parallel \pi_0(x)}}
- \bracket*{\E_{x \sim \sD_\sX}\bracket*{\sfB_F\paren*{\hh \pi(x) \parallel \ov \pi(x)}} + \delta},
\end{align*}
where $\delta = \E_{x \sim \sD_\sX} \bracket*{\bracket*{\lambda(x) F^*(u_0)
    + (1 - \lambda(x)) F^*(u_1)} - F^*(\lambda(x) u_0 + (1 - \lambda(x)) u_1)}
\geq 0$, with $\lambda(x) = \frac{\P(x)}{\P(x) + \P(\Phi(x))}$,
$u_0 = \nabla F(\pi_0 (x))$ and $u_1 = \nabla F(\pi_0 (\Phi(x)))$.
\end{theorem}
\begin{proof}
  Since $\Phi$ is an involution, we have $\sX = \Phi(\sX)$ and the
  expectation minimized to obtain $\ov \pi$ can be written as follows,
  as $\ov \pi$ is coherent:
  \[
  \E_{x \sim \sD_\sX} \bracket*{\sfB_F\paren*{\pi(x)  \parallel  \pi_0(x)}}
  = \frac{1}{2} \sum_{x \in \sX} \bracket[\Big]{\P[x] \, \sfB_F\paren*{\pi(x)  \parallel  \pi_0(x)} + \P[\Phi(x)] \, \sfB_F\paren*{\h \pi(x)  \parallel  \pi_0(\Phi(x))} }.
  \]
  Since $\ov \pi$ is only restricted to be in $\Piall$, the
  minimization of this objective can be decoupled into the following
  problem for each $x \in \sX$:
\begin{align*}
  \min_{\pi(x)} \quad & \lambda(x) \, \sfB_F\paren*{\h \pi(x)  \parallel  \pi_0(x)}
                        + (1 - \lambda(x)) \, \sfB_F\paren*{\h \pi(x)  \parallel  \pi_0(\Phi(x))},
\end{align*}
where $\lambda(x) = \frac{\P(x)}{\P(x) + \P(\Phi(x))}$.
Thus, since $F$ is a Legendre function,
by Lemma~\ref{lemma:right-bregman-centroid}, for any $x \in \sX$,
$\ov \pi(x)$ is given by
\[
  \ov \pi(x) = (\nabla F)^{-1}\paren*{\lambda(x) \nabla F(\pi_0(x))
  + (1 - \lambda(x)) \nabla F(\pi_0(\Phi(x)))}.
\]
By definition, $\hh \pi$ is a Bregman-projection of $\ov \pi$ onto
$(\Ccoh \cap \Pi)$. By Lemma~\ref{lemma:Bregman-expectation}, the
expectation of the Bregman divergence $\sfB_F$ is also a Bregman
divergence. Thus, since $\pi^*$ is in $(\Ccoh \cap \Pi)$, by
the Pythagorean theorem, we can write
\begin{align}
\label{eq:pyth-1}
  \E_{x \sim \sD_\sX}\bracket*{\sfB_F\paren*{\pi^*(x) \parallel \hh \pi(x)}}
  & \leq \E_{x \sim \sD_\sX}\bracket*{\sfB_F\paren*{\pi^*(x) \parallel \ov \pi(x)}} - 
    \E_{x \sim \sD_\sX}\bracket*{\sfB_F\paren*{\hh \pi(x) \parallel \ov \pi(x)}}.
\end{align}
By the Pythagorean theorem, considering the projection of $\pi_0$
onto $\Ccohdag$, we can also write
\begin{align*}
  \E_{x \sim \sD_\sX}\bracket*{\sfB_F\paren*{\pi^*(x) \parallel \ov \pi(x)}}
  & \leq \E_{x \sim \sD_\sX}\bracket*{\sfB_F\paren*{\pi^*(x) \parallel \pi_0(x)}} - 
    \E_{x \sim \sD_\sX}\bracket*{\sfB_F\paren*{\ov \pi(x) \parallel \pi_0(x)}}.
\end{align*}
Combining these inequalities yields:
\begin{align*}
  \E_{x \sim \sD_\sX}\bracket*{\sfB_F\paren*{\pi^*(x) \parallel \hh \pi(x)}}
  & \leq \E_{x \sim \sD_\sX}\bracket*{\sfB_F\paren*{\pi^*(x) \parallel \pi_0(x)}} \\
  & \quad  - \bracket*{\E_{x \sim \sD_\sX}\bracket*{\sfB_F\paren*{\hh \pi(x) \parallel \ov \pi(x)}}
    + \E_{x \sim \sD_\sX}\bracket*{\sfB_F\paren*{\ov \pi(x) \parallel \pi_0(x)}}}.
\end{align*}
This proves the first statement.
We now analyze more precisely the first term on the right-hand side of
\eqref{eq:pyth-1}. Using duality and
Lemma~\ref{lemma:bregman-identity}, we can write:
\begin{align*}
  & \E_{x \sim \sD_\sX}\bracket*{\sfB_F\paren*{\pi^*(x) \parallel \ov \pi(x)}}\\
  & = \E_{x \sim \sD_\sX}\bracket*{\sfB_{F^*} \paren*{\nabla F(\ov \pi(x))
    \parallel \nabla F (\pi^*(x))}} 
    \tag{$F$ Legendre and duality property}\\
  & = -\delta + \E_{x \sim \sD_\sX}\bracket*{\lambda(x) \, \sfB_{F^*} \paren*{\nabla F(\pi_0(x))
    \parallel \nabla F (\pi^*(x))}
    + (1 - \lambda(x)) \, \sfB_{F^*} \paren*{\nabla F(\pi_0(\Phi(x)))
    \parallel \nabla F (\pi^*(x))}} 
    \tag{Lemma~\ref{lemma:bregman-identity}}\\
  & = -\delta + \E_{x \sim \sD_\sX}\bracket*{\lambda(x) \, \sfB_F \paren*{\pi^*(x)
    \parallel \pi_0(x)}
    + (1 - \lambda(x)) \, \sfB_F \paren*{\pi^*(x)
    \parallel \pi_0(\Phi(x))}}.
    \tag{duality}
\end{align*}
We will show that the second expression coincides with
$\E_{x \sim \sD_\sX} \bracket*{\sfB_F\paren*{\pi^*(x) \parallel
    \pi_0(x)}}$.
Define $b(x) = \sfB_F \paren*{\pi^*(x) \parallel \pi_0(x)}$ and note that
$1 - \lambda(x) = \lambda(\Phi(x))$. Then, we can write:
\begin{align*}
 & \E_{x \sim \sD_\sX}\bracket*{\lambda(x) \, \sfB_F \paren*{\pi^*(x)
    \parallel \pi_0(x)}
    + (1 - \lambda(x)) \, \sfB_F \paren*{\pi^*(x)
    \parallel \pi_0(\Phi(x))}}\\
 & = \E_{x \sim \sD_\sX}\bracket*{\lambda(x) \, b(x) + \lambda(\Phi(x)) \, b(\Phi(x))}\\
 & = \frac{1}{2} \sum_{x \in \sX} \bracket*{\P(x) \lambda(x) b(x) + \P(x) \lambda(\Phi(x)) b(\Phi(x)) +  \P(\Phi(x)) \lambda(x) b(x) + \P(\Phi(x)) \lambda(\Phi(x)) b(\Phi(x))}\\
 & = \frac{1}{2} \sum_{x \in \sX} \bracket*{\P(x) \, b(x) + \P(\Phi(x)) \, b(\Phi(x))}\\
 & = \E_{x \sim \sD_\sX}[b(x)] = \E_{x \sim \sD_\sX} \bracket*{\sfB_F\paren*{\pi^*(x) \parallel
    \pi_0(x)}}.
\end{align*}
Thus, we have
\[
\E_{x \sim \sD_\sX}\bracket*{\sfB_F\paren*{\pi^*(x) \parallel \ov \pi(x)}}
  = \E_{x \sim \sD_\sX} \bracket*{\sfB_F\paren*{\pi^*(x) \parallel
      \pi_0(x)}}
  - \delta.
\]
Substituting this expression in \eqref{eq:pyth-1} yields the second statement and
completes the proof.
\end{proof}

When $\hh \pi = \h \pi$, 
the theorem establishes an alternative guarantee for $\h \pi$.
We briefly compare the improvement guarantees provided by our theorems
for $\h \pi$ and $\hh \pi$:
\begin{align*}
  \sfB_F(\pi^* \parallel \pi_0) - \sfB_F(\pi^* \parallel \h \pi)
  & \geq \sfB_F(\h \pi \parallel \pi_0)\\
  \sfB_F(\pi^* \parallel \pi_0) - \sfB_F(\pi^* \parallel \hh \pi)
  & \geq \sfB_F(\ov \pi \parallel \pi_0) + \sfB_F(\hh \pi \parallel \ov \pi).
\end{align*}
Since $\h \pi$ is in $\Pi \cap \Ccoh$ and thus also in $\Ccohdag$, by
the Pythagorean theorem applied to the projection of $\pi_0$ onto
$\Ccohdag$ and $\ov \pi$ onto $\Pi \cap \Ccoh$, we have
\begin{align*}
\sfB_F(\h \pi \parallel \pi_0)
&\geq \sfB_F(\h \pi \parallel \ov \pi) + \sfB_F(\ov \pi \parallel \pi_0)\\
\sfB_F(\h \pi \parallel \ov \pi)
& \geq \sfB_F(\h \pi \parallel \hh \pi) + \sfB_F(\hh \pi \parallel \ov \pi).
\end{align*}
Substituting the second inequality into the first yields:
\begin{align*}
  \sfB_F(\h \pi \parallel \pi_0)
  & \geq \sfB_F(\h \pi \parallel \hh \pi) + \sfB_F(\hh \pi \parallel \ov \pi)
    + \sfB_F(\ov \pi \parallel \pi_0)\\
  & \geq \sfB_F(\hh \pi \parallel \ov \pi)
    + \sfB_F(\ov \pi \parallel \pi_0).  
\end{align*}
Thus, the improvement guarantee for $\h \pi$ is always more favorable
than the one for $\hh \pi$.

\subsection{Improvement guarantees for relative entropy two-step
projection}
\label{sec:relative-entropy-guarantees}

When $F^*$ is $\mu$-strongly convex, which implies the
$1/\mu$-smoothness of $F$, the bound in the second statement of the
theorem admits an explicit upper bound
$- \frac{\mu}{2} \lambda (1 - \lambda) \norm*{\nabla F(\pi_0 (x)) -
  \nabla F(\pi_0 (\Phi(x))}^2$. In the following, we present
an explicit upper bound in the case of relative entropy,
$\sfB_F = \KL$, in terms of the Hellinger
distance.
Recall that the \emph{Hellinger distance} $\Hell(\sfp
\parallel \sfq)$ between two distributions $\sfp$ and $\sfq$ in
$\Delta_p$ is defined as
\[
\Hell(\sfp \parallel \sfq)
= \frac{1}{\sqrt{2}} \sqrt{\sum_{k = 1}^p \bracket*{\sqrt{\sfp_k} - \sqrt{\sfq_k}}^2},
\]
and that
the squared Hellinger distance is thus defined by
$\Hell^2(\sfp \parallel \sfq) = 1 - \sum_{k = 1}^p \sqrt{\sfp_k \sfq_k}$.
Both distances take values in $[0, 1]$.

\begin{corollary}
\label{lemma:improvement bound hellinger}
Let $F$ be the negative entropy, so that $\sfB_F = \KL$
is the (unnormalized) relative entropy. Then, under the assumptions of
Theorem~\ref{th:two-step-bregman-projection}, the $\log$-loss
improvement of $\hh \pi$ over the baseline $\pi_0$ is at least
\[
\E_{x \sim \sD_\sX} \bracket*{2 \min\curl*{\lambda(x), 1 - \lambda(x)} \,
  \Hell^2(\pi_0(x) \parallel \pi_0(\Phi(x)))},
\]
where $\lambda(x) = \frac{\P[x]}{\P[x] + \P[\Phi(x)]}$.
\end{corollary}

\begin{proof}
Let $\sfX = \pi_0(x)$ and $\sfY = \pi_0(\Phi(x))$, and define $u_0 =
\nabla F(\sfX)$ and $u_1 = \nabla F(\sfY)$ where $F \colon
\Rset_+^d \to \Rset$ is the negative entropy function defined by
$F(\sfp) = \sum_k \sfp_k \log \sfp_k$, for all $\sfp \in \Rset_+^d$.
$F$ is differentiable over $\Rset_{++}^d$.
We first derive a more explicit expression for
\[
\Delta = F^*(\lambda(x) u_0 + (1 - \lambda(x)) u_1)
- \bracket*{\lambda(x) F^*(u_0) + (1 - \lambda(x)) F^*(u_1)}.
\]
For any $\sfp \in \Rset_{++}^d$, $\nabla F(\sfp)$ is the vector of
coordinates $\log \sfp_k + 1$.
The conjugate function $F^*$ of $F$ is known to be defined
  by $F^*(y) = \sum_{k = 1}^d e^{y_k - 1}$ for all $y \in \Rset^d$.
  Thus, for any $\sfp \in \Rset_+^d$, we have
  $F^*(\nabla F(\sfp)) = \sum_k \sfp_k$.
  In light of that, we have $\bracket*{\lambda(x) F^*(u_0) + (1 -
    \lambda(x)) F^*(u_1)} = 1$ and
\begin{align*}
  \Delta
  = F^*(\lambda(x) u_0 + (1 - \lambda(x)) u_1) - 1
  = \sum_{k = 1}^d e^{\lambda(x) (\log \sfX_k + 1) + (1 - \lambda(x)) (\log \sfY_k + 1) - 1} - 1
  = \sum_{k = 1}^d \sfX_k^\lambda(x) \, \sfY_k^{1 - \lambda(x)} - 1.
\end{align*}
We can express $\Delta$ in terms of the $\lambda(x)$-R\'enyi divergence,
leverage the fact that the R\'enyi divergence is non-decreasing in $\lambda(x)$ and the identity $\sfD_{1/2}(\sfX \parallel \sfY) = - 2
\log \bracket*{1 - \Hell^2(\sfX \parallel \sfY)}$
\citep{VanErvenHarremoes2014} to derive the following upper bound for
$\lambda(x) \leq \frac{1}{2}$:
\begin{align*}
  \Delta
  & = \exp((\lambda(x) - 1) \sfD_{\lambda(x)}(\sfX \parallel \sfY)) - 1
  \tag{definition of $\lambda(x)$-R\'enyi divergence}\\
  & \leq \exp((\lambda(x) - 1) \sfD_{1/2}(\sfX \parallel \sfY)) - 1
  \tag{$\lambda(x)$-R\'enyi divergence non-decreasing function of $\lambda(x)$}\\
  & = \bracket*{1 - \Hell^2(\sfX \parallel \sfY)}^{2 (1 - \lambda(x))} - 1
  \tag{expression of $\sfD_{1/2}$ in terms of Hellinger distance}\\
  & \leq \bracket*{1 - 2 (1 - \lambda(x)) \Hell^2(\sfX \parallel \sfY)} - 1
  \tag{inequality $(1 - x)^{\alpha} \leq 1 - \alpha x$, valid for
    $\alpha \in [0, 1]$ and $x \in [0, 1]$}\\
  & = - 2 (1 - \lambda(x)) \, \Hell^2(\sfX \parallel \sfY).
\end{align*}
A symmetric argument for $\lambda(x) \geq \frac{1}{2}$ yields $\Delta
\leq - 2 \lambda(x) \, \Hell^2(\sfX \parallel \sfY)$ by first
expressing $\Delta$ as
\[
\exp(\lambda(x) \sfD_{1 - \lambda(x)}(\sfY \parallel
\sfX)) - 1.
\]
Hence, in all cases, we have
\[
\Delta \leq - 2 \min\curl*{\lambda(x), 1 - \lambda(x)} \,
\Hell^2(\sfX \parallel \sfY).
\]
The claim then follows by integrating this bound in
Theorem~\ref{th:two-step-bregman-projection}.
\end{proof}
Tighter inequalities with more favorable constants can be derived in
terms of the Hellinger distance at the expense of more complex proofs,
or by resorting to a refined Young's inequality such as that of
\cite{KittanehManasrah2010}. Alternative bounds can also be derived in
terms of the total variation, using the inequality
$\sfD_{\lambda(x)}(\sfX \parallel \sfY) \geq \frac{\alpha}{2} V^2(\sfX,
\sfY)$, valid for all $\lambda(x) \in (0, 1]$
  \citep{VanErvenHarremoes2014}, where $V(\sfX, \sfY)$ is the total
  variation of $\sfX$ and $\sfY$.

\subsection{Maximin properties}
\label{sec:maximin-properties}

In the previous sections, we considered a fixed reference (or ideal)
coherent conditional distribution function $\pi^*$ and showed that
both the direct and two-step Bregman projection solutions yield
improvements over the baseline $\pi_0$.  For a given conditional
distribution function $\pi$, the improvement for a reference $\pi^*$
is measured by
  \[
    \Improv_{\pi^*}(\pi)
    = \E_{x \sim \sD_\sX} \bracket*{\sfB_F\paren*{\pi^*(x) \parallel \pi_0(x)}}
    - \E_{x \sim \sD_\sX} \bracket*{\sfB_F \paren*{\pi^*(x) \parallel \pi(x)}}.
  \]
  We now turn to the problem of finding a solution that maximizes the
  worst-case improvement over the choice of the reference distribution
  $\pi^*$. This leads to the following max–min optimization problem:
  \begin{align}
  \label{problem:maxmin}
  \sup_{\pi \in \Ccoh \cap \Pi} \, \inf_{\pi^* \in \Ccoh \cap \Pi}
  \Improv_{\pi^*}(\pi).
\end{align}
When $\sfB_F$ is jointly convex, as in the squared distance and the
relative entropy, $\Improv_{\pi^*}(\pi)$ is a concave function
of $\pi$, since the infimum of a set of concave functions is also
concave. Since $\Pi$ is a (closed) convex set, \eqref{problem:maxmin}
is then a convex optimization in $\pi$. Our next result shows that the
two-step-Bregman projection $\hh \pi$ is always a solution to this
problem.

\begin{theorem}
\label{th:maxmin}
Let $\Phi \colon \sX \to \sX$ be an involution and let $\Pi$ be a
convex and closed family of conditional distribution functions
$\pi \colon \sX \to \Delta(\sY)$.  Assume that $F$ is a Legendre
function. Then, the two-step-Bregman projection is a solution of
Problem~\ref{problem:maxmin}.
\end{theorem}

\begin{proof}
  Since both $\pi^*$ and $\pi$ are in $\Ccoh$, we can write:
  \begin{align*}
    \Improv_{\pi^*}(\pi)
    & = \E_{x \sim \sD_\sX} \bracket*{\sfB_F(\pi^*(x) \parallel \pi_0(x))
      - \sfB_F(\pi^*(x) \parallel \pi(x)) }\\
    & = \frac{1}{2} \sum_{x \in \sX} \bracket[\Big]{
      \P(x) \, \sfB_F(\pi^*(x) \parallel \pi_0(x))
      + \P(\Phi(x)) \, \sfB_F(\pi^*(x) \parallel \pi_0(\Phi(x)))\\
    & \quad  - (\P(x) + \P(\Phi(x))) \, \sfB_F(\pi^*(x) \parallel \pi(x)) }\\
    & = \frac{1}{2} \sum_{x \in \sX} (\P(x) + \P(\Phi(x))) \, \bracket[\Big]{
      \lambda(x) \, \sfB_F(\pi^*(x) \parallel \pi_0(x))
      + (1 - \lambda(x)) \, \sfB_F(\pi^*(x) \parallel \pi_0(\Phi(x)))\\
    & \quad  - \sfB_F(\pi^*(x) \parallel \pi(x)) }\\
    & = \E_{x \sim \sD_\sX} \bracket[\Big]{
      \lambda(x) \, \sfB_F(\pi^*(x) \parallel \pi_0(x))
      + (1 - \lambda(x)) \, \sfB_F(\pi^*(x) \parallel \pi_0(\Phi(x)))
      - \sfB_F(\pi^*(x) \parallel \pi(x)) },
  \end{align*}
  where we adopt the notation used in Theorem~\ref{th:two-step-bregman-projection}:
  $\lambda(x) = \frac{\P(x)}{\P(x) + \P(\Phi(x))}$. Using
  Lemma~\ref{lemma:bregman-identity}, we can rewrite the sum of the
  first two terms within the expectation as follows:
  \begin{align*}
    & \lambda(x) \, \sfB_F(\pi^*(x) \parallel \pi_0(x))
    + (1 - \lambda(x)) \, \sfB_F(\pi^*(x) \parallel \pi_0(\Phi(x)))\\
    & = \lambda(x) \, \sfB_{F^*}(\nabla F(\pi_0(x)) \parallel \nabla F(\pi^*(x)))
    + (1 - \lambda(x)) \,
      \sfB_{F^*}(\nabla F(\pi_0(\Phi(x))) \parallel \nabla F(\pi^*(x)))\\
    & = \sfB_F(\pi^*(x) \parallel \ov \pi(x)) + \delta(x),
  \end{align*}
  where
  $\delta(x) = \bracket*{\lambda(x)
      F^*(u_0) + (1 - \lambda(x)) F^*(u_1)} - F^*(\lambda(x) u_0 + (1
    - \lambda(x)) u_1) \geq 0$, with
  $u_0 = \nabla F(\pi_0 (x))$ and $u_1 = \nabla F(\pi_0 (\Phi(x)))$,
  and
  \[
    \ov \pi(x) = (\nabla F)^{-1}\paren*{\lambda(x) \nabla F(\pi_0(x))
      + (1 - \lambda(x)) \nabla F(\pi_0(\Phi(x)))}.
  \]
  It is straightforward to check that $\ov \pi$ is coherent.
  Thus, we have
  \begin{align*}
    \sup_{\pi \in \Ccoh \cap \Pi} \inf_{\pi^* \in \Ccoh \cap \Pi} \Improv_{\pi^*}(\pi)
    & = \sup_{\pi \in \Ccoh \cap \Pi}  \inf_{\pi^* \in \Ccoh \cap \Pi}
      \E_{x \sim \sD_\sX} \bracket[\big]{\sfB_F(\pi^*(x) \parallel \ov \pi(x))
      + \delta(x) - \sfB_F(\pi^*(x) \parallel \pi(x))}\\
    & \leq \sup_{\pi \in \Ccoh \cap \Pi} \E_{x \sim \sD_\sX} \bracket[\big]{
      \sfB_F(\hh \pi(x) \parallel \ov \pi(x)))
      + \delta(x) - \sfB_F(\hh \pi(x) \parallel \pi(x))}
    \tag{$\pi^*$ can be chosen to be $\hh \pi$ since $\hh \pi \in \Ccoh \cap \Pi$}\\
    & = \E_{x \sim \sD_\sX} \bracket*{
      \sfB_F(\hh \pi(x) \parallel \ov \pi(x)))
      + \delta(x)}\\
    & \leq \Improv_{\pi^*}(\hh \pi),
      \tag{Theorem~\ref{th:two-step-bregman-projection}}
  \end{align*}
  for any $\pi^* \in \Ccoh \cap \Pi$. This shows that
  $\Improv_{\pi^*}(\hh \pi)$ attains the value of
  Problem~\ref{problem:maxmin} and that $\hh \pi$ is a solution.
\end{proof}
This establishes an additional favorable property for the two-step
projection solution $\hh \pi$ and thus also for the single-step
projection $\h \pi$ when the two coincide, for Bregman divergences
whose generator is defined over the full positive orthant.

\section{Characterization of Improvement Mechanisms}
\label{sec:characterization}

In this section, we study mechanisms that improve a baseline function
$\pi_0 \in \Pi$ under Bregman divergences. We begin with the case of a
fixed divergence, showing that any mechanism guaranteeing improvement
must coincide with a Bregman projection onto a block-constant set
(Subsection~\ref{sec:characterization-single-F}). We then turn to the
stronger requirement of improvement across \emph{all} Legendre Bregman
divergences. This leads to a rigidity result: the outputs of any such
mechanism are forced to lie in a \emph{universal block-constant
  structure}, independent of the particular divergence
(Subsection~\ref{sec:characterization-rigidity}).

\subsection{Characterization for a single divergence}
\label{sec:characterization-single-F}

We begin by analyzing mechanisms that guarantee improvement with
respect to a fixed Legendre Bregman divergence. This case serves as a
building block for understanding the universal rigidity phenomenon
we analyze later.

\begin{theorem}[General improvement implies Bregman projection]
\label{th:single_F}
Let $F$ be a Legendre function and let $\pi_0 \in \Piall$. Suppose a
mechanism produces $\h \pi_F \in \Pi$ such that for every
$\pi^* \in \Pi$,
\[
  \E\bracket*{\sfB_F(\pi^*(x) \parallel \h \pi_F(x))}
  \leq
  \E\bracket*{\sfB_F(\pi^*(x) \parallel \pi_0(x))}
  - \E\bracket*{\sfB_F(\h \pi_F(x) \parallel \pi_0(x))}.
\]
Assume $\Pi$ is closed and convex and that a.e.\ the relevant values
lie in the interior of the domain of $F$.  Define the partition
$\curl*{\sX_i}_{i \in I}$ by the equivalence relation
$x \sim x' \Leftrightarrow \h \pi_F(x) = \h \pi_F(x')$ and set
\[
  \sC_F = \curl*{ \pi \in \Pi \colon \pi(x) = \pi(x'),
  \forall x, x' \in \sX_i,\ \forall i \in I}.
\]
Then
\[
  \h \pi_F
  = \argmin_{\pi \in \Pi \cap \sC_F}
  \E\bracket*{\sfB_F(\pi(x) \parallel \pi_0(x))}.
\]
That is, $\h \pi_F$ is the (unique) Bregman projection of $\pi_0$ onto
$\Pi \cap \sC_F$.
\end{theorem}

\begin{proof}
  Fix any comparator $\pi \in \Pi \cap \sC_F$. Consider the global
  convex combination (valid because $\Pi$ is convex)
  $\pi_t = (1 - t) \h \pi_F + t \pi$, $t \in [0, 1]$.  Define the
  affine test function $g$ as follows:
\[
  g(t) = \E\bracket*{\sfB_F(\pi_t(x) \parallel \h \pi_F(x))}
         - \E\bracket*{\sfB_F(\pi_t(x) \parallel \pi_0(x))}
         + \E\bracket*{\sfB_F(\h \pi_F(x) \parallel \pi_0(x))}.
\]
In the definition of $g$ due to the difference of the first
Bregman divergences, the terms in $F(\pi_t(x))$ vanish. As
a result, $g$ is an affine function of $t$.
By the assumed strong improvement property we have $g(t) \leq 0$ for
every $t \in [0, 1]$. Since $\pi_0$ and $\h \pi_F$ are
fixed, we can write
\[
  g(0) = \E\bracket*{\sfB_F(\h \pi_F \parallel \h \pi_F)}
         - \E\bracket*{\sfB_F(\h \pi_F \parallel \pi_0)}
         + \E\bracket*{\sfB_F(\h \pi_F \parallel \pi_0)} = 0.
\]
Thus, $g$ is affine in $t$ and satisfies $g(t) \leq 0$ for all
$t \in [0, 1]$ with $g(0) = 0$. Its right-hand derivative at $0$
therefore satisfies $g'(0^+) \leq 0$. For any $t \in (0, 1)$,
we have
\begin{align*}
g'(t)
  & = \frac{d}{dt} \bracket[\Big]{\E\bracket*{\sfB_F(\pi_t(x) \parallel \h \pi_F(x))}
    - \E\bracket*{\sfB_F(\pi_t(x) \parallel \pi_0(x))}} \\
  & = \frac{d}{dt} \bracket[\Big]{\E\big[F(\pi_0(x)) - F(\h \pi_F(x))\big] 
    + \E\bracket*{\tri*{\nabla F(\pi_0(x)), \pi_t(x) - \pi_0(x)}} \\
  & \quad - \E\bracket*{\tri*{\nabla F(\h \pi_F(x)), \pi_t(x) - \h \pi_F(x)}}}\\
  & = \frac{d}{dt} \bracket[\Big]{
    \E\bracket*{\tri*{\nabla F(\pi_0(x)), \pi_t(x) - \pi_0(x)}} 
    - \E\bracket*{\tri*{\nabla F(\h \pi_F(x)), \pi_t(x) - \h \pi_F(x)}}}\\
  & = -\E\bracket[\Big]{\tri*{\nabla F(\h \pi_F(x)) - \nabla F(\pi_0(x)),\ \pi(x)-\h \pi_F(x)}}.
\end{align*}
Thus $g'(0^+) \leq 0$ implies
\[
  \E\bracket*{\tri*{\nabla F(\h \pi_F(x)) - \nabla F(\pi_0(x)),\ \pi(x)-\h \pi_F(x)}}
  \ge 0.
\]
Since this holds for every $\pi\in\Pi\cap\sC_F$, it is exactly the
first-order optimality (variational) condition for minimizing the
strictly convex functional
$\pi \mapsto \E\bracket*{\sfB_F(\pi \parallel \pi_0)}$ over the closed
convex set $\Pi\cap\sC_F$. Uniqueness of the minimizer follows the
fact that $F$ is Legendre. Thus, $\h \pi_F$ equals that minimizer, that
is, it is the Bregman projection of $\pi_0$ onto $\Pi \cap \sC_F$.
\end{proof}
If the equivalence classes $\sX_i$ are singletons or arise as orbits
of a function $\Phi$ (e.g., an involution) with
$\pi(x) = \pi(\Phi(x))$, then $\sC_F$ coincides with the generalized
coherence set $\Ccoh$. In this case, the mechanism output $\h \pi_F$
is exactly the Bregman projection onto $\Ccoh$.

The theorem's requirement may appear too stringent: not only does it
require the mechanism $\cM$ to improve upon the baseline $\pi_0$ for
any reference distribution $\pi^*$ but it also requires that
improvement to be at least
$\E\bracket*{\sfB_F(\h \pi_F(x) \parallel \pi_0(x))}$. However, as
shown by Theorem~\ref{th:direct-bregman-projection}, that improvement
is already guaranteed by a Bregman divergence projection over a convex
set. In light of that result, the requirement is very natural.

Although the proof of Theorem~\ref{th:single_F} relies on standard
variational arguments, the theorem itself is far from trivial. The key
insight is that the strong universal improvement condition, requiring
the mechanism to reduce the Bregman divergence relative to \emph{all}
comparators $\pi^* \in \Pi$, forces the output $\h \pi_F$ to lie in a
highly structured subset of $\Pi$. Specifically, the equivalence-class
construction defining $\sC_F$ is not imposed a priori but emerges
naturally from the mechanism’s output: any two inputs that the
mechanism maps to the same value define a class, and the mechanism’s
output must be constant across each class. This structure is
remarkable because it reveals a rigidity phenomenon: a mechanism
cannot achieve universal improvement without collapsing its output
onto a convex set $\Pi \cap \sC_F$ that respects these equivalence
classes. In other words, the theorem shows that universal improvement
does not merely constrain the value of $\h \pi_F$, but also forces a
specific piecewise-constant geometry on its output, which is a
nontrivial structural characterization.

\subsection{Rigidity under all divergences -- Topological proof}
\label{sec:characterization-rigidity}

In mathematics, rigidity theorems show that certain structures,
particularly in geometry and topology, are resisting to perturbation
\citep{Mostow1973,GromovPansu1991,Spatzier2004}. We find an analogous
principle in self-improvement, where a mechanism must improve an
initial policy $\pi_0$ simultaneously across all Legendre Bregman
divergences. While the optimal solution $\h \pi_F$ can vary with each
divergence $F$, this demand for universal improvement imposes a
powerful structural constraint. This is the rigidity we identify: all
possible outputs are confined to the same block-constant set, as with
the coherent set.

This result is stronger than the single-$F$ case.  For a fixed
Legendre function $F$, the mechanism output $\h \pi_F$ is a Bregman
projection onto a set $\sC_F$ that can depend on the output itself. In
contrast, the demand for universal improvement forces all outputs to
project onto a single, common set $\sC$. This shared constraint is
what forges the rigid, block-constant structure that any such
mechanism must invariably produce, regardless of the specific
divergence $F$.

The proof of rigidity for general families of divergences is far from
trivial. At first sight, one might expect the argument to proceed by
establishing a general variational inequality for the mechanism and
then intersecting the coherence sets $\sC_F$ across all generators
$F$. However, this approach quickly runs into trouble: even for two
functions, the intersection $\sC_{F_1} \cap \sC_{F_2}$ need not be
nontrivial. The real challenge is not to characterize the intersection
of partitions but to prove that the partitions themselves are
identical. Continuity of the map $F \mapsto \h \pi_F$ is an important
ingredient, yet it does not by itself resolve this issue. A critical
insight is that the rigidity proof may require fundamentally different
strategies depending on the geometry of the constraint set $\Pi$,
making a general theorem substantially more delicate.


In this section, we prove a rigidity theorem using a topological
argument.  We will assume stability, that is the continuity of the
output map $F \mapsto \h\pi_F$. We will further assume a strictness
assumption, which provides a safety margin in the optimality
condition. Together, surprisingly, these assumptions allow us to prove
that the solution's coherence set $\sC_F$ is \emph{sticky}, forcing
every generator in a connected family to share the same universal
coherence set.

\begin{theorem}[Topological rigidity]
\label{th:topological_rigidity}
Let $\cF$ be a family of Legendre functions whose domains contain a
fixed open convex neighborhood of $\Delta(\sY)$. Let
$\Pi \subset \Piall$ be closed and convex and let $\pi_0 \in \Piall$.
Let $\cM$ be a mechanism that for every $F \in \cF$ produces an output
$\h \pi_F \in \Pi$ satisfying the strong improvement inequality: for
all $\pi^* \in \Pi$,
\[
  \E\bracket*{\sfB_F(\pi^*(x) \parallel \h \pi_F(x))}
  \leq
  \E\bracket*{\sfB_F(\pi^*(x) \parallel \pi_0(x))}
  - \E\bracket*{\sfB_F(\h \pi_F(x) \parallel \pi_0(x))}.
\]
Assume that $F \mapsto \h \pi_F$ is continuous from the generator
topology (uniform convergence of $F$ and $\nabla F$ on $\Delta(\sY)$)
into $L^1(\sD_\sX)$ (stability).

Define the partition $\curl*{\sX_i}_{i \in I}$ by the equivalence
relation $x \sim x' \Leftrightarrow \h \pi_{F_0}(x) = \h \pi_{F_0}(x')$
(using the level sets of a chosen seed $F_0 \in \cF$) and set
  $\sC_{F_0} = \curl*{ \pi \in \Pi \colon \pi(x) = \pi(x'),
  \forall x, x' \in \sX_i,\ \forall i \in I }$.
Assume further \emph{strictness} on the path-connected family
$\wt \cF \subseteq \cF$ containing $F_0$: there exists
$\gamma > 0$ such that for every $F \in \wt \cF$,
\[
  \inf_{\pi \in \Pi \setminus \sC_{F_0}}
  \E\bracket*{\tri*{\nabla F(\h \pi_F(x)) - \nabla F(\pi_0(x)),\ 
    \pi(x) - \h \pi_F(x)}}
  \geq \gamma.
\]
Then, for every $F \in \wt \cF$ the mechanism output satisfies
$\h \pi_F \in \sC_{F_0}$ a.e.  In particular, the partition of $\sX$
induced by $\h \pi_F$ is identical for all $F \in \wt \cF$.
\end{theorem}

Intuitively, the theorem's topological argument relies on the
\emph{strictness} assumption to create a "potential well" around the
universal coherence set $\sC_{F_0}$. The variational inequality, which
is derived from the improvement guarantee, can be thought of as
measuring the slope of a loss function ($\Psi_F$ in the proof) at the
solution $\h \pi_F$.
The standard optimality condition (with a $\geq 0$ bound) would just
mean that $\h \pi_F$ is at a minimal point within the set $\sC_{F_0}$
(the floor of the well). Any step from $\h \pi_F$ to another
$\pi \in \sC_{F_0}$ is flat or uphill.
The strictness condition,
$\inf_{\pi \in \Pi \setminus \sC_{F_0}} \Psi_F(\pi) \geq \gamma > 0$,
is much stronger. It means that for \emph{any} point $\pi$
\emph{outside} the set $\sC_{F_0}$ ("cliffs" of the well), the slope
in that direction is not just non-negative, but is strictly positive
and uniformly bounded by $\gamma$. This $\gamma$-gap ensures that the
set $\sC_{F_0}$ is robustly optimal, with no other competitors just
outside the set that are nearly as good.
The proof leverages this $\gamma$-gap. The stability assumption
guarantees that as we perturb $F$, the solution $\h \pi_F$ can only
move slightly. The $\gamma$-gap ensures this slight move is not
enough to jump out of the potential well. This makes the coherence
set $\sC_{F_0}$ "sticky", forcing all solutions along the continuous
path $\wt \cF$ to remain trapped inside it.

\begin{proof}
  Define
  $S = \curl*{F \in \wt \cF \, \colon \, \h \pi_F \in \sC_{F_0}}$.  We
  show that $S$ is nonempty, open and closed in $\wt \cF$, which by
  path-connectedness implies $S = \wt \cF$.
  $\sC_{F_0}$ is considered modulo null sets, thus, throughout the
  proof, the qualifier ‘a.e.’ will be understood implicitly in the
  relevant statements, and thus will not be repeated.
  
  By assumption, $F_0$ is in $\wt \cF$ and
  $\h \pi_{F_0} \in \sC_{F_0}$, so $S$ contains $F_0$ and is
  non-empty.  Let $(F_n)_{n\ge 1} \subset S$ with $F_n\to
  F_\infty$. By stability, we also have
  $\h \pi_{F_n} \to \h \pi_{F_\infty}$ in $L^1(\sD_\sX)$. Each
  $\h \pi_{F_n}$ is in $\sC_{F_0}$ and thus blockwise constant on the
  fixed partition defining $\sC_{F_0}$. The set $\sC_{F_0}$ is closed
  in $L^1(\sD_\sX)$, so the limit $\h \pi_{F_\infty}\in
  \sC_{F_0}$. Hence $F_\infty\in S$, and $S$ is closed.

  Fix any $G \in S$. We will define a neighborhood $\sU$ of $G$
  contained in $S$. To do so, we show that the strictness property of
  $G$ extends to a neighborhood $\sU$. This disallows any $F \in \sU$
  from having an output $\h \pi_F$ outside $\sC_{F_0}$, as its
  functional $\Psi_F(\h\pi_F)$ must be zero by definition, yet the
  strictness property requires it to be strictly positive.
  By stability, the map $F \mapsto \h \pi_F$ is continuous from the
  generator topology (uniform convergence of $F$ and $\nabla F$ on
  $\Delta(\sY)$) to $L^1(\sD_\sX)$.  Define, for any $F \in \wt \cF$
  and any $\pi \in \Pi$,
\[
  \Psi_F(\pi) = \E\bracket*{\tri*{\nabla F(\h \pi_F(x)) - \nabla
      F(\pi_0(x)),\ \pi(x) - \h \pi_F(x)}}.
\]
The map $(F, \pi) \mapsto \Psi_F(\pi)$ is continuous since
$F \mapsto \nabla F$ is continuous in the topology and
$F \mapsto \h \pi_F$ is continuous by stability.  This implies that
the map $F \mapsto \sup_{\pi\in\Pi} |\Psi_F(\pi) - \Psi_G(\pi)|$ is
also continuous.
\begin{figure}[t]
  \centering
  \includegraphics[scale=.18]{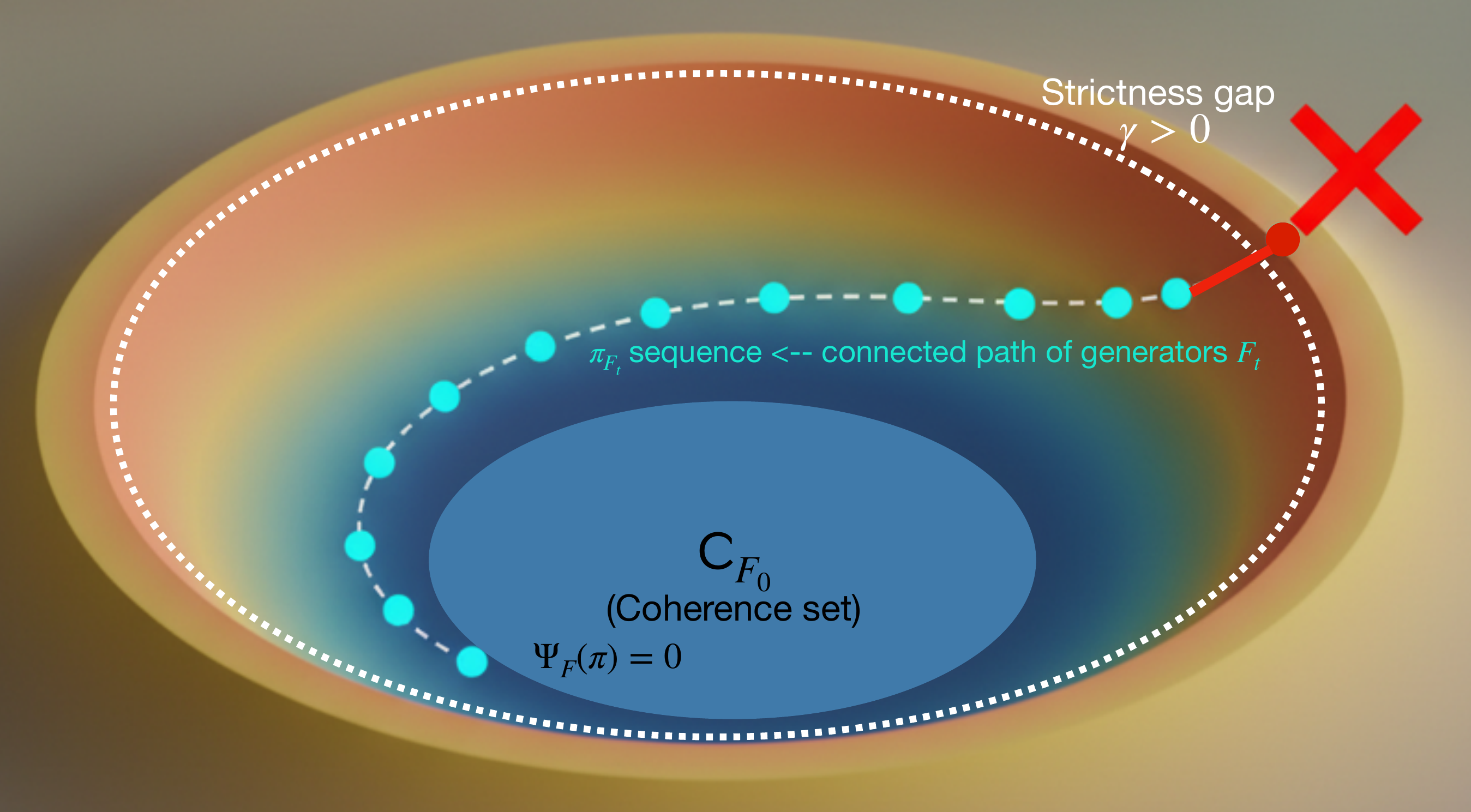}
  \caption{Illustration of the proof of Theorem~\ref{th:topological_rigidity}.}
  \label{fig:relaxed}
\end{figure}
Therefore, we can choose a neighborhood $\sU$ of $G$ small enough so
that for every $F \in \sU$, since $\pi$ and $\h\pi_F$ take values in
the simplex and therefore $\abs*{\pi - \h\pi_F}$ is uniformly bounded:
\[
  \sup_{\pi\in\Pi} \abs*{\Psi_F(\pi) - \Psi_G(\pi)} < \gamma/3.
\]
Since $G$ is in $S$, by the strictness hypothesis for all
$\pi \in \Pi \setminus \sC_{F_0}$, we have $\Psi_G(\pi) \ge \gamma$.
Combining these inequalities using the triangle inequality, for all
$F \in \sU$ and all $\pi \in \Pi \setminus \sC_{F_0}$, we have
\[
  \Psi_F(\pi) = \Psi_G(\pi) + (\Psi_F(\pi) - \Psi_G(\pi))
  \ge \gamma - |\Psi_F(\pi) - \Psi_G(\pi)|
  \ge \gamma - \gamma/3 = 2\gamma/3.
\]
This establishes that the strictness property holds (with a smaller
gap $2\gamma/3$) for the entire neighborhood $\sU$.
Now, suppose for contradiction that there exists some $F \in \sU$ such
that its output is not in the set: $\h \pi_F \not \in \sC_{F_0}$.
Since $\h \pi_F \in \Pi \setminus \sC_{F_0}$, we can use this
$\h \pi_F$ as the competitor $\pi$ in the strictness bound we just
derived: $\Psi_F(\h \pi_F) \ge 2\gamma/3 > 0$.  However, by definition
of $\Psi_F(\cdot)$, evaluating it at its own solution $\h \pi_F$
yields: $\Psi_F(\h \pi_F) = 0$.  This gives the final contradiction
$0 \geq 2\gamma/3 > 0$. Thus, no such $F$ exists, and for every
$F \in \sU$, we have $\h \pi_F \in \sC_{F_0}$.  This proves that
$\sU \subseteq S$, and that $S$ is open.

Since $S = \wt \cF$, for every $F \in \wt \cF$ we have
$\h \pi_F \in \sC_{F_0}$. This completes the proof.
\end{proof}
We now comment on the uniform strictness assumption adopted in
the theorem.
One might hope that a weaker, pointwise strictness assumption (i.e.,
$\Psi(\pi) > 0$ for all $\pi \in \Pi \setminus \sC_{F_0}$) would be
sufficient, especially if $\Pi$ is compact. This is not the case. The
set of competitors $\Pi \setminus \sC_{F_0}$ is not compact: it is a
compact set minus a closed subset, which is relatively open. A
continuous function on a non-compact set is not guaranteed to attain
its minimum, and its infimum can be zero even if the function is
strictly positive everywhere. For example, if $\Pi = [-1, 1]$ and
$\sC_{F_0} = \curl*{0}$, the functional $\Psi(\pi) = \pi^2$ is
strictly positive on $\Pi \setminus \sC_{F_0}$, yet
$\inf \Psi(\pi) = 0$. Thus, the uniform $\gamma$-gap is a necessary
and stronger condition.

This uniform $\gamma$-gap holds for some important classes of convex
sets but fails for others.  The assumption is likely to hold for
polyhedral sets, such as the probability simplex. In that case, $\Pi$
is composed of flat "faces." A competitor $\pi \notin \sC_{F_0}$ must
lie on a different face (or be non-aligned on the same face). This
discrete, piecewise-linear geometry can bound the functional
$\Psi(\pi)$ away from zero, creating the required $\gamma$-gap.
Conversely, the assumption often fails for smoothly curved sets.
In examples with the $L^2$-ball, the functional $\Psi(\pi)$ can
be strictly positive for all competitors, yet its infimum can still be
zero, as the value is approached on the boundary of $\sC_{F_0}$.

A promising alternative proof is a topological two-path
argument. First, we establish a baseline coherence set $\sC_0$
using a seed generator $F_0$. For any other generator $G$, we define a
continuous path of generators $F_t$ connecting $F_0$ to $G$. The proof
then compares two continuous solution paths: (1) the "true" path
$\h\pi_t = P_{\Pi}^{F_t}(\pi_0)$, which is continuous by the stability
assumption, and (2) the candidate path
$\pi_t^* = P_{\sC_0}^{F_t}(\pi_0)$, the projection onto the fixed
set $\sC_0$. Since both paths start at the same point
($\h\pi_0 = \pi_0^*$), the proof's objective is to show they are
identical ($\h\pi_t = \pi_t^*$ for all $t$), which would imply
$\h\pi_G \in \sC_0$. This requires showing that the candidate path
$\pi_t^*$ also satisfies the global variational inequality, which is
typically guaranteed by a strictness assumption that creates a
potential well, forcing the true path to \emph{stick} to the
set $\sC_0$. However, perhaps alternative techniques can
be used to enforce that with weaker assumptions.

We assume stability, that is, the continuity of the map $F \mapsto \h
\pi_F$ (from the $C^1$ generator topology to $L^1(\sD_\sX)$). This is
a standard regularity assumption in parametric optimization;
continuity of the unique minimizer $\h \pi_F$ is expected since it
minimizes a strictly convex functional, $\E\bracket*{\sfB_F(\cdot
\parallel \pi_0)}$, which varies continuously with $F$ in the $C^1$
topology over the compact set $\Pi$ (see for example Berge's Maximum
Theorem \citep{Berge1963,AliprantisBorder2006}).  Rigorously verifying
the conditions for continuity in this specific functional setting is
left for future work.

\subsubsection{Toy example}

This simple example is convenient to illustrate the rigidity
theorem. The example is first constructed so that $\Pi$ enforces a
block-equality constraint. The coherence set $\sC$ (derived from any
generator) then coincides with $\Pi$, and the strictness condition
holds \emph{vacuously}, as $\Pi \setminus \sC$ is empty. We then show
how strictness \emph{fails} in a non-trivial case.

Let $\sX = \curl*{1, 2, 3}$ and $\sY = \curl*{0, 1}$. We identify a
conditional model $\pi$ with the vector $p = (p(1), p(2), p(3))$ where
$p(x) = \Pr[y=1 \mid x]$. Let $\sD_\sX$ be uniform on $\sX$ (weight
$1/3$ each). Fix the baseline $\pi_0$ with
\[
p_0(1) = 0.10, \qquad p_0(2) = 0.80, \qquad p_0(3) = 0.40.
\]
Define the feasible set $\Pi$ by the single block constraint that
$p(1) = p(2)$, that is
\[
\Pi = \curl*{ \pi = (p(1), p(2), p(3)) \colon p(1) = p(2),\ p(i) \in[0, 1] }.
\]
The corresponding block partition is $\sX_1=\curl*{1, 2}$ and
$\sX_2 = \curl*{3}$. The block-constant coherence set defined relative
to $\Pi$ is
\[
\sC = \curl*{ \pi\in\Pi : \pi(1) = \pi(2)} = \Pi.
\]
In this case, $\Pi \setminus \sC$ is the empty set, so the infimum in
the strictness condition is taken over an empty set (and is
$+\infty$), meaning the condition is vacuously satisfied. As predicted
by Theorem~\ref{th:topological_rigidity}, the projection $\h \pi_F$
must lie in $\sC$ for any $F$. We verify this for several generators.

The projection $\h \pi_F$ of $\pi_0$ onto $\Pi$ minimizes the expected
Bregman divergence. Since the constraint only links $p(1)$ and $p(2)$,
the minimizer $p(3)$ is always $p_0(3) = 0.40$. The value
$q^\star = p(1) = p(2)$ minimizes
$\sfB_F(q \parallel p_0(1)) + \sfB_F(q \parallel p_0(2))$. By
Lemma~\ref{lemma:right-bregman-centroid}, $q^\star$ is the Bregman
centroid, satisfying
$\nabla F(q^\star) = \frac{1}{2} (\nabla F(p_0(1)) + \nabla
F(p_0(2)))$.

\textbf{Squared-Euclidean generator.}
Take $F_{\mathrm{sq}}(p) = \tfrac{1}{2} \norm*{p}^2$. Then
$\nabla F(p) = p$. The centroid is the arithmetic mean:
\[
q^\star = \frac{1}{2} (p_0(1) + p_0(2)) = \frac{0.10 + 0.80}{2} = 0.45.
\]
Hence
\[
\h \pi_{\mathrm{sq}} = (0.45, 0.45, 0.40) \in \sC.
\]

\textbf{Negative-entropy generator ($F(p) = p \log p$).}
Here $\nabla F(p) = \log p + 1$. The centroid $q^\star$ satisfies:
\begin{align*}
\log q^\star + 1 & = \frac{1}{2} [(\log p_0(1) + 1) + (\log p_0(2) + 1)] \\
\log q^\star & = \frac{1}{2} [\log(0.1) + \log(0.8)] = \log \sqrt{0.08} \\
q^\star & = \sqrt{0.08} \approx 0.2828.
\end{align*}
Thus
\[
\h \pi_{\mathrm{KL}} = (0.2828, 0.2828, 0.40) \in \sC.
\]

\textbf{Negative logarithm generator ($F(p) = -\log p$).}
Here $\nabla F(p) = -1/p$. The centroid $q^\star$ (the harmonic mean)
satisfies:
\begin{align*}
-1/q^\star & = \frac{1}{2} [-1/p_0(1) - 1/p_0(2)] \\
1/q^\star & = \frac{1}{2} [1/0.1 + 1/0.8] = \frac{1}{2}[10 + 1.25] = 5.625 \\
q^\star & = 1/5.625 = 8/45 \approx 0.1778.
\end{align*}
Thus
\[
\h \pi_{\mathrm{Itakura-Saito}} = (0.1778, 0.1778, 0.40) \in \sC.
\]
As expected, all solutions lie in $\sC = \Pi$, but the specific
projection depends on the generator $F$.

\textbf{Non-separable Bregman generator.}
Here, we consider the generator
\[
F(p) = \sum_{x = 1}^3 \paren*{\frac{1}{2} p(x)^2} + p(1) p(3).
\]
This generator is not separable, so we must minimize
$G(p) = \sfB_F(p \parallel p_0)$ subject to $p(1) = p(2) = q$ and
$p \in [0, 1]^3$. This is equivalent to minimizing
$G(q, p_3) = q^2 + \frac{1}{2}p_3^2 + q p_3 - 1.3 q - 0.5 p_3$ (see
text for derivation).  The Hessian is
\[
\nabla^2 F(p)
=
\begin{pmatrix}
1 & 0 & 1 \\
0 & 1 & 0 \\
1 & 0 & 1
\end{pmatrix},
\]
which is positive semi-definite (eigenvalues $0, 1, 2$). The
unconstrained minimizer is $(0.8, 0.8, -0.3)$, which is outside the
domain. A search of the boundaries of the feasible set $[0, 1]^2$ for
$(q, p_3)$ finds the minimum at $(q, p_3) = (0.65, 0)$.  The
projection is thus
\[
\h \pi_{\mathrm{Non-Separable-Gen}} = (0.65, 0.65, 0) \in \sC,
\]
which is distinct from previous projections.

\textbf{Testing strictness non-trivially, a failure case.}
The strictness assumption
($\inf_{\pi \in \Pi \setminus \sC_{F_0}} \Psi_F(\pi) \ge \gamma > 0$)
is a strong condition. A meaningful test requires a scenario where the
competitor set $\Pi \setminus \sC_{F_0}$ is non-empty. We can
construct a simple and sound case where this assumption fails.

Let the feasible set be the unconstrained cube, $\Pi = \Piall = [0, 1]^3$. 
Let the generator be the squared-Euclidean generator, $F_0 = F_{\mathrm{sq}}$, 
so $\nabla F_0(p) = p$.
Now, to induce a non-trivial coherence set, we choose a baseline
$\pi_0$ that already has ties: $\pi_0 = (0.5, 0.5, 0.2)$.  Since
$\pi_0$ is already inside the unconstrained set $\Pi$, the Bregman
projection is just $\pi_0$ itself:
$\h \pi_{F_0} = \pi_0 = (0.5, 0.5, 0.2)$.

The set $\sC_{F_0}$ is defined by the level-set partition of 
$\h \pi_{F_0}$. Since $\h \pi_{F_0}(1) = \h \pi_{F_0}(2)$, the 
induced partition is $\sX_1=\{1, 2\}, \sX_2=\{3\}$, and the 
coherence set is:
\[
 \sC_{F_0} = \curl*{\pi \in \Pi \colon \pi(1) = \pi(2)}.
\]
In this case, $\Pi$ is the full cube, while $\sC_{F_0}$ is just the 2D
slice of the cube where the first two coordinates are
equal. Therefore, the competitor set $\Pi \setminus \sC_{F_0}$ is
non-empty, and our test is non-trivial.
We now check the infimum of the functional $\Psi_{F_0}(\pi)$ over the
competitor set:
\[
 \Psi_{F_0}(\pi)
 = \E\bracket*{\tri*{\nabla F_0(\h \pi_{F_0})
     - \nabla F_0(\pi_0),\ \pi - \h \pi_{F_0}}}.
\]
Since $\h \pi_{F_0} = \pi_0$, the gradient term is zero:
$\nabla F_0(\h \pi_{F_0}) - \nabla F_0(\pi_0) = \h \pi_{F_0} - \pi_0 =
0$.  Therefore, the functional is identically zero for all $\pi$:
$\Psi_{F_0}(\pi) = \E\bracket*{\tri*{0,\ \pi - \h \pi_{F_0}}} = 0$.
The strictness condition requires
$\inf_{\pi \in \Pi \setminus \sC_{F_0}} \Psi_{F_0}(\pi) \geq \gamma >
0$.  However, our calculation shows:
\[
  \inf_{\pi \in \Pi \setminus \sC_{F_0}} \Psi_{F_0}(\pi)
  = \inf_{\pi(1) \neq \pi(2)} 0 = 0.
\]
The strictness condition would require this infimum to be strictly
positive. Thus it fails, even though the competitor set is non-empty.
This is a “zero-slope” failure mode illustrating that strictness is a
genuinely non-trivial requirement, not implied by coherence or
feasibility alone.

\subsubsection{Interpretation and significance of rigidity}

The rigidity theorems establish a surprising result: while the
mechanism's outputs $\h\pi_F$ vary with the choice of generator $F$,
their underlying level-set structure remains constant. In this
section, we interpret this universal coherence property and discuss why
it is a powerful and non-trivial consequence of the stability
assumption.

Interpretation of the Rigidity Result.
Our rigidity theorems reveal a powerful structural property of Bregman
projections under a stability assumption. For any given generator $F$,
its output $\h \pi_F$ induces a partition of the domain $\sX$ based on
its level sets. This partition, in turn, defines a \emph{coherence
  set} $\sC_F$, an affine subspace of functions in $\Pi$ that are
constant on those same level sets.
The crucial conclusion of the theorems is that this partition is, in
fact, \emph{universal}: it is independent of the specific generator
chosen. All generators in a connected family are forced to produce
outputs that are piecewise-constant on the \emph{exact same}
underlying partition, defining a single, universal coherence set $\sC$.

Nontriviality of the Result.
This result is far from trivial. A priori, there is no reason to
expect the level-set partitions to agree across different
generators. Each generator $F$ endows the space with a different
geometry, leading to a different Bregman divergence and, in general, a
different projection $\h \pi_F$. As we saw in our numerical examples,
the \emph{values} of the projections $\h \pi_F$ and $\h \pi_G$ can
differ significantly for generators $F \neq G$.
The theorem's surprising conclusion is that despite these differences
in values, the underlying \emph{topological structure} (the level-set
partition) of all solutions must collapse to a single, fundamental
one. The stability assumption can be viewed as acting as a powerful
regularizer, forcing this structure to be rigid and independent of the
specific geometric "metric" ($F$) used for the projection.

\subsection{Rigidity under all divergences -- Algebraic proof}

An alternative proof that does not require the strictness assumption
imposes instead a structural assumption on $\Pi$. Here, we present a
proof in the case where $\Pi$ is a closed affine subspace.
The argument first establishes the result for the subclass of twice
continuously differentiable Legendre functions and then extends it to
all Legendre functions by invoking their density within the entire
family (Lemma~\ref{lemma:local_c1_approx_corrected}).

\begin{theorem}[Rigidity of Improvement for Affine Constraints]
\label{th:rigidity_affine}
Let $\cF$ be the family of all Legendre functions.
Let
$\Pi \subset \Piall$ be a closed convex set that is also an affine
subspace, and let $\pi_0 \in \Piall$. Suppose a mechanism $\cM$
produces for any $F \in \cF$ an output $\h \pi_F \in \Pi$
satisfying the strong improvement inequality: for all $\pi^* \in \Pi$,
\[
  \E\bracket*{\sfB_F(\pi^*(x) \parallel \h \pi_F(x))}
  \leq
  \E\bracket*{\sfB_F(\pi^*(x) \parallel \pi_0(x))}
  - \E\bracket*{\sfB_F(\h \pi_F(x) \parallel \pi_0(x))}.
\]
Assume that the map $F \mapsto \h \pi_F$ is continuous from the $C^1$
generator topology to $L^1(\sD_\sX)$ (stability).
Then, there exists a measurable partition $\curl*{\sX_i}_{i \in I}$ of
$\sX$ such that, defining the universal coherence set
\[
  \sC
  = \curl*{ \pi \in \Pi \colon \pi(x) = \pi(x')\ \text{for a.e.\ }
  x, x' \in \sX_i,\ \forall i \in I},
\]
we have $\h \pi_F \in \sC$ for all $F \in \cF$.
\end{theorem}

\begin{proof}
 The hypothesis holds for all $F \in \cF$ and thus the subset
 $\mathcal{F}_{C^2}$. By Theorem \ref{th:rigidity_affine_C2}, this
 proves the existence of a universal partition
 $\curl*{\sX_i}_{i \in I}$ and a universal coherence set $\sC$ such
 that $\h \pi_G \in \sC$ for all $G \in \cF_{C^2}$.
 
 Now, select any function $F' \in \cF \setminus \cF_{C^2}$. The set
 $\cF_{C^2}$ is dense in $\cF$ in the required generator topology
 (see for example \citep{Rockafellar1996}). Therefore, there exists
 a sequence of smooth generators $\{G_n\} \subset \cF_{C^2}$ such
 that $G_n \to F'$.
 By the stability assumption, this generator convergence implies the
 convergence of their outputs in $L^1$:
 $\h \pi_{G_n} \to \h \pi_{F'}$ in $L^1(\sD_\sX)$. Since $G_n$ is in
 $\cF_{C^2}$, $\{\h \pi_{G_n}\}$ belongs to the set $\sC$. The set
 $\sC$ is defined by linear equality constraints, making it a closed
 set in the $L^1$ topology. Thus, the limit point $\h \pi_{F'}$ must
 also be in $\sC$. Since this holds for any arbitrary $F'$, we have
 proven that $\h \pi_F \in \sC$ for all $F \in \cF$.
\end{proof}

We now prove the result for the special case of $C^2$ Legendre
functions.

\begin{theorem}[Rigidity of Improvement for Affine Constraints -- $C^2$ functions]
\label{th:rigidity_affine_C2}
Let $\cF_{C^2}$ be the family of all $C^2$ Legendre generators. Let
$\Pi \subset \Piall$ be a closed convex set that is also an affine
subspace, and let $\pi_0 \in \Piall$. Suppose a mechanism $\cM$
produces for any $F \in \cF_{C^2}$ an output $\h \pi_F \in \Pi$
satisfying the strong improvement inequality: for all $\pi^* \in \Pi$,
\[
  \E\bracket*{\sfB_F(\pi^*(x) \parallel \h \pi_F(x))}
  \leq
  \E\bracket*{\sfB_F(\pi^*(x) \parallel \pi_0(x))}
  - \E\bracket*{\sfB_F(\h \pi_F(x) \parallel \pi_0(x))}.
\]
Assume that the map $F \mapsto \h \pi_F$ is continuous from the $C^1$
generator topology to $L^1(\sD_\sX)$ (stability).
Then, there exists a measurable partition $\{\sX_i\}_i$ of $\sX$ such
that, defining the universal coherence set
\[
  \sC := \curl*{ \pi\in\Pi \colon \pi(x)=\pi(x')\
  \text{for a.e.\ } x,x'\in\sX_i,\ \forall i \in I },
\]
we have $\h \pi_F \in \sC$ for every $F \in \cF_{C^2}$.
\end{theorem}

Before proceeding with the proof, we note that the strong improvement
inequality is exactly equivalent to the variational inequality (VI)
used in the proof. This equivalence is a direct consequence of the
Bregman three-point identity:
For any $\pi^*, \h \pi_F, \pi_0$, 
\[
  \E\bracket*{\sfB_F(\pi^* \parallel \h \pi_F)} = 
  \left( \E\bracket*{\sfB_F(\pi^* \parallel \pi_0)} - 
  \E\bracket*{\sfB_F(\h \pi_F \parallel \pi_0)} \right) -
  \E\bracket*{\tri*{\nabla F(\h \pi_F) - \nabla F(\pi_0), \pi^* - \h \pi_F}}.
\]
The theorem's hypothesis (the strong improvement inequality) is the
statement that the first term on the right-hand side is an upper bound
on the left-hand side.
Thus, the theorem's hypothesis is precisely equivalent to the
variational inequality:
\[
  \E_{x \sim \sD_\sX}\bracket*{\tri*{\nabla F(\h \pi_F(x)) - \nabla F(\pi_0(x)),
  \pi^*(x) - \h \pi_F(x)}} \geq 0,
  \qquad\forall \pi^* \in \Pi.
\]
The proof that follows will rely directly on this variational form.

\begin{proof}
 For any $F, G \in \cF_{C^2}$, we let
 $V_F = \nabla F(\h \pi_F) - \nabla F(\pi_0)$ and write
 $A(F, G) = \E\bracket*{\tri*{V_F - V_G, \h \pi_G - \h \pi_F}}$. A
 general four-point inequality, $A(F, G) \geq 0$, can be shown
 straightforwardly to hold for all $F, G \in \cF_{C^2}$, using the
 variational inequalities.
 
 The proof is in two main parts. First, we leverage the affine
 subspace hypothesis to prove that the four-point inequality is
 always an equality ($A(F,G) = 0$) in our context. Second, we use
 this fact to run a proof by contradiction.

\textbf{1. Four-Point Equilibrium, $A(F, G) = 0$.}
Let $\Pi = v_0 + \Pi_0$ be our affine subspace, where $\Pi_0$ is the
corresponding parallel linear subspace.

The variation inequality (VI),
$\E[\tri*{V_F, \pi^* - \h \pi_F}] \geq 0$ holds for all $\pi^* \in \Pi$.
Since $\h \pi_F \in \Pi$, for any vector $v \in \Pi_0$, the
competitors $\pi^* = \h \pi_F + v$ and $\pi^{**} = \h \pi_F - v$ both
also lie in $\Pi$. Testing the inequality with both competitors yields:
\begin{align*}
 \E\bracket*{\tri*{V_F,\ (\h \pi_F + v) - \h \pi_F}}
 & = \E\bracket*{\tri*{V_F,\ v}} \geq 0 \\
 \E\bracket*{\tri*{V_F,\ (\h \pi_F - v) - \h \pi_F}}
 & = \E\bracket*{\tri*{V_F,\ -v}} \geq 0.
\end{align*}
These two facts force the variation inequality to be an equality. The
error vector $V_F$ must be orthogonal to the subspace $\Pi_0$:
$\E\bracket*{\tri*{V_F,\ v}} = 0$, for all $v \in \Pi_0$. This holds
for any generator. Thus, for any pair $(F, G)$, we have
$\E[\tri{V_F, v}] = 0$ and $\E[\tri{V_G, v}] = 0$.

By hypothesis, $\h \pi_F \in \Pi$ and $\h \pi_G \in \Pi$. Their difference
$\Delta_{F,G} = \h \pi_G - \h \pi_F$ must therefore lie in the parallel
subspace $\Pi_0$. Since $\Delta_{F,G} \in \Pi_0$, we can use it as the
test vector $v$ in both orthogonality conditions, which gives:
\[
 \E[\tri{V_F, \Delta_{F,G}}] = 0 \quad \text{and} \quad \E[\tri{V_G, \Delta_{F,G}}] = 0.
\]
The four-point term is the difference of these two zeros:
\[
  A(F,G) = \E[\tri{V_F - V_G, \Delta_{F,G}}] = \E[\tri{V_F, \Delta_{F,G}}]
  - \E[\tri{V_G, \Delta_{F,G}}] = 0 - 0 = 0.
\]
This proves Part 1.

\paragraph{2. Contradiction via Perturbation.}
Assume by contradiction that there exists a non-rigid pair
$(F_1, F_2)$ with $\Delta = \h \pi_2 - \h \pi_1 \not \equiv 0$.
By the Lusin-based approximation (Lemma~\ref{lem:lusin_refined}), for
any $\eta > 0$, we can find a compact set $K \subset \sX$ with
$\mu(\sX \setminus K) < \eta$ and a finite measurable partition
$\{K_j\}_{j = 1}^m$ of $K$ such that on each $K_j$, the relevant
functions are $\eta$-close to a set of constant representatives:
\begin{itemize}

\item $\sup_{x \in K_j} \norm*{\h \pi_k(x) - q_{k,j}} \leq \eta$
 (for $k = 0, 1, 2$, where $\h \pi_0 = \pi_0$);
 
\item $\sup_{x \in K_j} \norm*{\nabla^2 F_k(x) - M_{k,j}} \leq \eta$ (for
 $k = 1, 2$).

\end{itemize}
Let $w_{k,j} = q_{k,j} - q_{0,j}$ and
$\Delta_j = q_{2,j} - q_{1,j}$. Since $\Delta \not\equiv 0$, we can
choose the partition fine enough to ensure that for at least one $j$,
$\mu(K_j)>0$ and $\Delta_j \neq 0$, which guarantees that
$-\e_0 = - \sum_{j=1}^m \mu(K_j) \norm*{\Delta_j}^2 < 0$.
This data defines a continuous, linear functional $\Psi$ which maps
lists of Hessian matrices to a scalar. For two lists
$\sS_1 = \{S_{1,j}\}_{j=1}^m$ and $\sS_2 = \{S_{2,j}\}_{j=1}^m$, we define:
\[
 \Psi(\sS_1, \sS_2)
 = \sum_{j=1}^m \mu(K_j) \tri*{S_{1,j} w_{1,j} - S_{2,j} w_{2,j}, \Delta_j}.
\]
Let $\cI = \{I\}_{j=1}^m$ be the list of identity
 matrices. The functional $\Psi$ evaluated at $(\cI, \cI)$ is:
\[
 \Psi(\cI, \cI)
 = \sum_{j=1}^m \mu(K_j) \tri*{I w_{1,j} - I w_{2,j}, \Delta_j}
 = \sum_{j=1}^m \mu(K_j) \tri*{-\Delta_j, \Delta_j}
 = - \sum_{j=1}^m \mu(K_j) \|\Delta_j\|^2
 = -\e_0 < 0.
\]
Let $\cM_1 = \{M_{1,j}\}_{j=1}^m$ and $\cM_2 = \{M_{2,j}\}_{j=1}^m$ be
the lists of representative Hessians for $F_1$ and $F_2$,
respectively. By the local $C^1$-approximation
(Lemma~\ref{lemma:local_c1_approx_corrected}), the true integral
$A(F_1, F_2)$ is approximated by the functional $\Psi$ evaluated at
these constant Hessians. By choosing the Lusin parameter $\eta$
sufficiently small, we can make the approximation error, $\delta_A$,
arbitrarily small.
\[
 \delta_A(\eta)
 = |\Psi(\cM_1, \cM_2)|
 = |A(F_1, F_2) - \Psi(\cM_1, \cM_2)|,
 \quad \text{where } \delta_A(\eta) \to 0 \text{ as } \eta \to 0.
\]
Fix a small perturbation scale $\alpha > 0$. We choose $\alpha$ small
enough to satisfy the Legendre constraints of the bump lemma.
Construct the generators $G_k = F_k + \sum_j h_{kj}$, where the bumps
$h_{kj}$ are chosen (via
Lemma~\ref{lemma:localized_hessian_corrected}) to satisfy:
\begin{enumerate}
 \item $\nabla^2 h_{kj}(q_{kj}) = \alpha (I - M_{kj})$.
 \item The global $C^1$-norm of the perturbation is small:
 $\norm{G_k - F_k}_{C^1} < \delta_{bump}$.
\end{enumerate}
This yields new generators $G_k$ that are Legendre (since $\alpha$ is
small) and $C^1$-close to $F_k$. The local Hessians of $G_k$ are the
target list $\sS_{G,k} = \{(1-\alpha)M_{kj} + \alpha I\}$.

By the stability assumption, choosing $\delta_{bump}$ small ensures
the outputs are close, which in turn means all approximations are
controlled. We can choose our parameters (Lusin $\eta$ and bump-norm
$\delta_{bump}$) small enough to make the total approximation error
$\e_{approx}$ between the true integral and the $\Psi$-functional
arbitrarily small.

We know that the inequality 
$A(G_1, G_2) \geq 0$ holds for all Legendre functions $G_1, G_2$.
By the stability assumption and the Lusin lemmas
(Lemmas \ref{lemma:local_c1_approx_corrected} and \ref{lem:lusin_refined}),
the true integral $A(G_1, G_2)$ is a continuous function of its
underlying components. We can therefore choose our parameters (Lusin $\eta$
and bump-norm $\delta_{bump}$) small enough to make the total
approximation error $\e_{approx}$ between the true integral and its
piecewise-constant model, $\Psi(\sS_{G,1}, \sS_{G,2})$, arbitrarily small:
\[
 \abs*{A(G_1, G_2) - \Psi(\sS_{G,1}, \sS_{G,2})} \leq \e_{approx}.
\]
Combining this error bound with the fact that $A(G_1, G_2) \geq 0$
yields the key inequality for our contradiction:
\[
 \Psi(\sS_{G,1}, \sS_{G,2}) \geq A(G_1, G_2) - \e_{approx} \geq -\e_{approx}.
\]
We can calculate the value of $\Psi(\sS_{G,1}, \sS_{G,2})$ explicitly:
\begin{align*}
 \Psi(\sS_{G,1}, \sS_{G,2})
 & = (1 - \alpha)\Psi(\cM_1, \cM_2) + \alpha \Psi(\cI, \cI) \\
 & = (1 - \alpha)\Psi(\cM_1, \cM_2) - \alpha\e_0.
\end{align*}
This gives us the inequality:
$(1 - \alpha)\Psi(\cM_1, \cM_2) - \alpha\e_0 \geq -\e_{approx}$. We
know $|\Psi(\cM_1, \cM_2)| \leq \delta_A$. Since $(1 - \alpha)>0$,
this means
$(1 - \alpha) \Psi(\cM_1, \cM_2) \geq -(1 - \alpha) \delta_A$. We use
the upper bound: $\Psi(\cM_1, \cM_2) \leq \delta_A$. This gives:
$\alpha\e_0 \leq (1 - \alpha) \Psi(\cM_1, \cM_2) + \e_{approx} \leq (1
- \alpha)\delta_A + \e_{approx}$, and the final inequality
$\alpha\e_0 \leq (1 - \alpha) \delta_A + \e_{approx}$.

This is a contradiction: The term on the left, $\alpha \e_0$, is a
fixed positive number (since we fixed a small $\alpha > 0$ and
$\e_0 > 0$); the term on the right,
$(1 - \alpha)\delta_A + \e_{approx}$, is the sum of our approximation
errors, which we can make arbitrarily small by choosing $\eta$ and
$\delta_{bump}$ to be small enough.
We can choose our parameters such that the total error is strictly
less that $\alpha\e_0 / 2$, which leads to the false statement
$\alpha\e_0 \leq \alpha\e_0 / 2$. This contradiction proves that the
initial assumption (a non-rigid pair $\Delta \not \equiv 0$) must be
false. This completes the proof.
\end{proof}

These results are already very compelling, establishing the rigidity
property across two major, distinct classes of problems. The algebraic
proof (Theorem~\ref{th:rigidity_affine_C2}) demonstrates that rigidity
holds for the entire class of affine subspaces by leveraging their
unique orthogonality properties (i.e., $A(F,G)=0$). In parallel, the
topological proof (Theorem~\ref{th:topological_rigidity}) provides a
robust framework for a different, broad class of sets, such as
polyhedral sets, by replacing the strong geometric assumption with the
analytical strictness condition.

While this does not yet cover all possible convex sets, such as
smoothly curved $L^2$-balls, where both of these current proof strategies
can fail, these theorems lay a clear foundation. It is likely that a
more general proof could be developed by building on these ideas. A
promising direction for future work, for example, is to find a weaker
condition than strictness that is still sufficient to guarantee the
stickiness of the solution path in the topological proof, which would
extend the result to a more general class of curved sets.

\subsubsection{Toy examples}

In this section, we present and discuss two simple examples illustrating
Theorem~\ref{th:rigidity_affine}.

\textbf{Example with affine constraints}
Here, we give a concrete nontrivial example where $\Pi$ is affine, a
hyperplane, and where Theorem~\ref{th:rigidity_affine} applies,
without any separate strictness hypothesis.  The induced coherence set
$\sC$ (functions constant on the block $\curl*{1, 2}$) is nontrivial
and arises from the interaction of the symmetry of $\Pi$ with the
symmetry of the baseline $\pi_0$, not from an explicit equality
constraint in $\Pi$.

Let $\sX = \curl*{1, 2, 3}$ and $\sY = \curl*{0, 1}$.  Identify a
conditional model $\pi$ with the vector $p = (p(1), p(2), p(3))$ where
$p(x) = \Pr(y = 1 \mid x)$.  Take $\sD_\sX$ to be the 
uniform distribution on $\sX$.
Define the affine feasible set $\Pi$ by the single linear constraint
\[
  \Pi = \curl*{\pi \in \Piall \colon p(1) + p(2) + p(3) = S },
\]
for some fixed constant $S\in(0, 3)$ (so $\Pi$ is an affine hyperplane in
$\Rset^3$ intersected with the cube $[0, 1]^3$).  Note that $\Pi$ is an
affine set and, importantly for our example, it is invariant under the
swap (permutation) of coordinates $1 \leftrightarrow 2$.

Choose a baseline $\pi_0$ that is symmetric in coordinates $1$ and $2$,
but does not lie in $\Pi$. For instance take
\[
  \pi_0 = (a,a,b), \qquad a,b\in(0,1), \qquad a+a+b \neq S.
\]
Concretely, one may pick $a=0.30,\ b=0.60$ and $S=1.0$, so that
$\pi_0=(0.30,0.30,0.60)$ with sum $1.20\neq 1.0$.

Fix any Legendre generator $F$ (e.g. KL or squared Euclidean).  Consider
the Bregman projection of $\pi_0$ onto the affine set $\Pi$:
\[
  \h\pi_F = \argmin_{\pi\in\Pi} \E\bracket*{\sfB_F(\pi(x)\|\pi_0(x))}.
\]

Since $\Pi$ is invariant under the transposition $\tau$ that swaps
indices $1$ and $2$, and since the baseline $\pi_0$ satisfies
$\tau \pi_0 = \pi_0$ (that is, $\pi_0$ is symmetric on the pair
$\curl*{1, 2}$), the following standard symmetry argument shows the
projection must be symmetric as well.

Let $J(\pi) = \E\bracket*{\sfB_F(\pi\|\pi_0)}$ denote the strictly
convex objective. For any feasible $\pi\in\Pi$ the swapped vector
$\tau\pi$ is also feasible (since $\Pi$ is swap-invariant) and
satisfies
$J(\tau\pi) = \E\bracket*{\sfB_F(\tau\pi\|\tau\pi_0)} =
\E\bracket*{\sfB_F(\tau\pi\|\pi_0)}$, because $\tau\pi_0 =
\pi_0$. Therefore $J(\pi)$ and $J(\tau\pi)$ are equal for every
feasible \(\pi\). By uniqueness of the minimizer (strict convexity),
the minimizer must be fixed by the swap: $\h \pi_F = \tau \h \pi_F$.

Thus $\h\pi_F(1) = \h\pi_F(2)$ and the projection $\h \pi_F$ is
block-constant on $\curl*{1, 2}$ for \emph{every} Legendre generator
$F$. The induced partition is therefore nontrivial (it merges inputs 1
and 2 into the same block: $\sX_1 = \curl*{1, 2}$, while
$\sX_2 = \curl*{3}$) even though the affine constraint defining $\Pi$
did not itself force $p(1) = p(2)$.

The conclusion that $\h\pi_F$ is block-constant on $\curl*{1, 2}$ for
all Legendre $F$ is immediate from the symmetry and uniqueness
argument and illustrates the rigidity phenomenon in a nontrivial
setting.


\textbf{Example with an asymmetric baseline}
This example illustrates Theorem~\ref{th:rigidity_affine} using a
different construction. Here, the affine set $\Pi$ itself enforces the
block-constant structure, and this structure is imposed on an
asymmetric baseline $\pi_0$.

Let $\sX = \curl*{1, 2, 3, 4}$ with $\sD_\sX$ a uniform distribution.
Define $\Pi$ as an affine subspace that forces two separate blocks to
be constant:
$\Pi = \curl*{\pi \in \Piall \colon p(1) = p(2) \text{ and } p(3) =
  p(4) }$.  We choose a baseline $\pi_0$ that is not symmetric and
does not lie in $\Pi$: $\pi_0 = (0.1, 0.3, 0.9, 0.5)$.  The mechanism
$\cM$ is the Bregman projection $\h \pi_F =
\proj_{\Pi}^{F}(\pi_0)$. This satisfies the theorem's assumptions, as
$\Pi$ is affine.

Since the constraints on blocks $\curl*{1, 2}$ and $\curl*{3, 4}$ are
independent, the projection problem decouples.

\textbf{A. Squared-Euclidean Generator ($F_{\mathrm{sq}}$)}
The projection $\h \pi_{\mathrm{sq}}$ minimizes the $L_2$ distance. The
solution for each block is the  mean of the $\pi_0$
values in that block.
\begin{itemize}
\item Block 1: $q_1 = p(1) = p(2)
  = \frac{p_0(1) + p_0(2)}{2} = \frac{0.1 + 0.3}{2} = 0.20$.

\item Block 2: $q_2 = p(3) = p(4)
  = \frac{p_0(3) + p_0(4)}{2} = \frac{0.9 + 0.5}{2} = 0.70$.
\end{itemize}
The projection is: $\h \pi_{\mathrm{sq}} = (0.20, 0.20, 0.70, 0.70)$.

\textbf{B. Negative-Entropy / KL Generator ($F_{\mathrm{KL}}$)}
The projection $\h \pi_{\mathrm{KL}}$ is the Bregman centroid, which
for this generator is the geometric mean of the $\pi_0$
values in each block.
\begin{itemize}
\item Block 1: $q_1 = p(1) = p(2)
  = \sqrt{p_0(1) \times p_0(2)} = \sqrt{0.1 \times 0.3} = \sqrt{0.03} \approx 0.173$.

\item Block 2: $q_2 = p(3) = p(4)
  = \sqrt{p_0(3) \times p_0(4)} = \sqrt{0.9 \times 0.5} = \sqrt{0.45} \approx 0.671$.

\end{itemize}
The projection is: $\h \pi_{\mathrm{KL}} \approx (0.173, 0.173, 0.671, 0.671)$.

As the theorem predicted, the rigidity phenomenon is observed. The
values of the projections are different, but $\h \pi_{\mathrm{sq}}$
and $\h \pi_{\mathrm{KL}}$ induce the same partition:
$\sX_1 = \curl*{1, 2}$, $\sX_2 = \curl*{3, 4}$.
The structure of the output is rigid and identical for all generators
$F$ because it is dictated entirely by the affine subspace $\Pi$.

\subsubsection{Extension via kernel methods}

The most direct way to extend the algebraic proof beyond affine sets,
using a simple and powerful idea, is indeed \emph{kernel methods}.

The algebraic proof (Theorem~\ref{th:rigidity_affine})
based on the affine nature of $\Pi$ provides a parallel linear subspace
$\Pi_0$. This is the key that allows us to test the variational
inequality with both $+v$ and $-v$ for any $v \in \Pi_0$, which in
turn proves the crucial "orthogonality" condition:
\[
 \E\bracket*{\tri*{V_F,\ v}} = 0, \quad \forall v \in \Pi_0.
\]
This orthogonality is what forces the four-point term $A(F, G)$ to be
zero, completing the proof. The extension idea is to apply this same
method in a high-dimensional feature space where $\Pi$ becomes affine:

\begin{itemize}
\item Mapping to a feature space: We can use a positive definite
  kernel-based feature mapping $\phi \colon \Piall \to \Hset$ to map
  our models $\pi$ from their original space to a high-dimensional
  Reproducing Kernel Hilbert Space (RKHS), $\Hset$.

\item Linearizing the set: If we can find a kernel such that the image
  of our (non-affine) set $\Pi$, denoted $\phi(\Pi)$, is an affine
  subspace in $\Hset$, then the entire algebraic proof holds.

\item Applying the proof in $\Hset$: The variational inequality would
  be expressed using the inner product in $\Hset$:
  \[
    \tri*{V_F, \phi(\pi^*) - \phi(\h \pi_F) }_{\Hset} \geq 0,
    \qquad \forall \pi^* \in \Pi.
  \]
  Here, $V_F$ is the error vector in the dual of $\Hset$.
  Since $\phi(\Pi)$ is an affine subspace, the exact same
  orthogonality argument applies: $V_F$ must be orthogonal to the
  parallel linear subspace $\phi(\Pi)_0$. This again forces
  $A(F, G) = 0$, where $A$ is now defined using the
  $\Hset$-inner product.
\end{itemize}

This \emph{kernelization} would prove that the mapped outputs,
$\phi(\h \pi_F)$, must have a universal block-constant structure in
the feature space $\Hset$. This implies that the original outputs
$\h \pi_F$ must also be rigid, conforming to a potentially non-linear
structure defined by the pre-image of that universal partition in
$\Hset$.

\subsubsection{Example: linearizing a spherical constraint via kernel methods}

Here, we illustrate the kernel approach with a very simple non-affine
feasible set that becomes affine after a nonlinear feature
mapping. Let $\sX = \curl*{1, 2}$ and $\pi = (p_1,p_2) \in
\Rset^2$. Consider the (non-affine) quadratic constraint:
\[
\Pi = \left\{ \pi \in \Rset_{\ge 0}^2 : p_1^2 + p_2^2 = 1 \right\},
\]
the positive quadrant of the unit circle. Since $\Pi$ is curved, the
algebraic proof of Theorem~\ref{th:rigidity_affine} does not apply in
the original space.
We now use a polynomial kernel to \emph{linearize} this constraint.  
Define the nonlinear feature map
\[
\phi(\pi) = (p_1^2,\, p_2^2) = (z_1, z_2).
\]
For any $\pi\in\Pi$, we have $z_1 + z_2 = 1$, so
\[
\phi(\Pi)
= \left\{ z \in \Rset_{\ge 0}^2 : z_1 + z_2 = 1 \right\},
\]
which is a line segment, an affine subset of the feature space
$\Hset$.

Thus, in the feature space, the feasibility region is affine, and the
projections $\phi(\h \pi_F)$ satisfy the orthogonality condition of
Theorem~\ref{th:rigidity_affine}. The rigidity argument therefore
applies directly in $\Hset$.

We can further illustrate this with a numerical example.  Take an
asymmetric baseline $\pi_0 = (1.0, 0.5)$ and consider two generating
functions.

\textbf{A. Squared Euclidean generator ($F_{\mathrm{sq}}$):}
\[
\h \pi_{\mathrm{sq}}
= \argmin_{p_1^2+p_2^2=1}
\frac{1}{2}\bigl( (p_1 - 1.0)^2 + (p_2 - 0.5)^2 \bigr)
\approx (0.894, 0.447).
\]

\textbf{B. Weighted squared generator ($F_W$):}.
\[
F_W(p) = \tfrac{1}{2}(p_1^2 + 10p_2^2),
\qquad
\h \pi_{W}
\approx (0.985, 0.174).
\]
Their images under $\phi$ both lie on the affine set $\phi(\Pi)$:
\[
\phi(\h \pi_{\mathrm{sq}}) \approx (0.80,0.20),\qquad
\phi(\h \pi_W) \approx (0.97,0.03).
\]

\begin{table}[t]
\centering
\begin{tabular}{@{}lll@{}}
\toprule
Generator ($F$)
& $\h \pi_F$ in original space
& $\phi(\h \pi_F)$ in feature space \\
\midrule
$F_{\mathrm{sq}}$ (L2) & $(0.894, 0.447)$ & $(0.80, 0.20)$ \\
$F_W$ (Weighted L2) & $(0.985, 0.174)$ & $(0.97, 0.03)$ \\
\bottomrule
\end{tabular}
\caption{Projections in the original and feature spaces.}
\label{table:circle-kernel}
\end{table}

Although the two Bregman projections differ visibly on the curved set
$\Pi$, their mapped projections lie on the same affine segment
$\phi(\Pi)$. Theorem~\ref{th:rigidity_affine} therefore forces the
mapped outputs to share a universal level-set structure. This yields a
corresponding nonlinear rigidity for the original projections
$\h \pi_F$ on $\Pi$.


\section{Conclusion}

We addressed the fundamental challenge of designing reliably
self-improving systems. While existing approaches are often effective
in practice, they largely rely on heuristics and lack formal
guarantees. To overcome this, we introduced a principled framework for
self-improvement based on the concept of \emph{coherence}, which
enforces consistency of model outputs under task-preserving input
transformations. Our theoretical analysis shows that iteratively
increasing coherence leads to monotonic improvement, formalized
through guaranteed reductions in expected Bregman divergence.

A key contribution of this work is our \emph{characterization
  theorem}, which establishes that any mechanism providing robust
improvement guarantees must align with a coherence-based structure,
under some broad assumptions. This result elevates our framework from
a specific method to a foundational principle, offering a rigorous
lens for designing and analyzing self-improving systems.

\ignore{
While the primary focus of this paper is theoretical, we also provide
empirical evidence that coherence-based projections yield tangible
performance gains.
}

We leave it to future work to further validate these methods on
diverse benchmarks, explore relaxed assumptions, extend to multi-agent
settings, and further investigate multiple connections to information
geometry and causal inference. By providing a solid theoretical
foundation, we hope to enable the development of self-improving models
that are not only more capable but also predictable, reliable, and
safe.

\newpage
\bibliography{ref}
\bibliographystyle{abbrvnat}

\newpage
\appendix

\ignore{
\renewcommand{\contentsname}{Contents of Appendix}
\tableofcontents
\addtocontents{toc}{\protect\setcounter{tocdepth}{3}} 
\clearpage
}

\section{Proofs of Bregman Divergence Results}
\label{app:appendix-bregman-proofs}

This section contains the detailed proofs of the lemmas and technical
results on Bregman divergences used throughout the paper.

\subsection{Centroid properties and projection lemmas}

\BregmanIdentity*
\begin{proof}
By definition of the Bregman divergence:
\[
\sfB_F(\sfp_k \parallel \sfq) = F(\sfp_k) - F(\sfq) - \tri{\nabla F(\sfq), \sfp_k - \sfq}.
\]
Summing over $k$ with weights $\lambda_k$ gives
\[
\sum_{k = 1}^p \lambda_k \sfB_F(\sfp_k \parallel \sfq) 
= \sum_{k = 1}^p \lambda_k F(\sfp_k) - F(\sfq) - \tri{\nabla F(\sfq), \sum_{k = 1}^p \lambda_k (\sfp_k - \sfq)}.
\]
Similarly,
\[
\sfB_F\Big(\sum_{k = 1}^p \lambda_k \sfp_k \parallel \sfq\Big) 
= F\Big(\sum_{k = 1}^p \lambda_k \sfp_k\Big) - F(\sfq) - \tri{\nabla F(\sfq), \sum_{k = 1}^p \lambda_k \sfp_k - \sfq}.
\]
Subtracting the second from the first yields
\[
\sum_{k = 1}^p \lambda_k F(\sfp_k) - F\Big(\sum_{k = 1}^p \lambda_k \sfp_k\Big),
\]
which is independent of $\sfq$, as claimed.
\end{proof}

\RightBregmanCentroid*
\begin{proof}
By the duality identity for Legendre functions:
\[
\sum_{k = 1}^p \lambda_k \sfB_F(\sfp \parallel \sfq_k) 
= \sum_{k = 1}^p \lambda_k \sfB_{F^*}(\nabla F(\sfq_k) \parallel \nabla F(\sfp)).
\]
Using Lemma~\ref{lemma:bregman-identity}, this is equivalent, up to an additive constant independent of $\sfp$, to
\[
\sfB_{F^*}\Big(\sum_{k = 1}^p \lambda_k \nabla F(\sfq_k) \parallel \nabla F(\sfp)\Big)
= \sfB_F\Big(\sfp \parallel (\nabla F)^{-1}\Big(\sum_{k = 1}^p \lambda_k \nabla F(\sfq_k)\Big)\Big).
\]
Thus, the unconstrained minimizer is
\[
\sfp^* = (\nabla F)^{-1}\Big(\sum_{k = 1}^p \lambda_k \nabla F(\sfq_k)\Big).
\]
For a closed convex set $\sC \subseteq \sK$, the minimizer over
$K$ is given by the Bregman projection of $\sfp^*$ onto $K$, as
required.
\end{proof}

\MinBregmanDiv*
\begin{proof}
  Immediate from Lemma~\ref{lemma:right-bregman-centroid}: the minimum
  of the sum over $K$ equals the sum evaluated at the unconstrained
  centroid plus the divergence between the projection $\sfp^*_\sC$
  and the unconstrained centroid $\sfp^*$.
\end{proof}

\subsection{Fenchel-Bregman inequality}

\BregmanFenchel*
\begin{proof}
By definition of the Bregman divergences:
\begin{align*}
  & \sfB_F(u \parallel v) + \sfB_{F^*}(\alpha \parallel \beta)
    - \tri{u - v, \alpha - \beta} \\
  & = F(u) - F(v) - \tri{\nabla F(v),u - v} + F^*(\alpha) - F^*(\beta)
    - \tri{\nabla F^*(\beta),\alpha - \beta} - \tri{u - v, \alpha - \beta} \\
  & = \big(F(u)+F^*(\alpha) - \tri{u,\alpha}\big)
    - \big(F(v)+F^*(\beta) - \tri{v,\beta}\big) + \tri{u - v, \beta - \nabla F(v)}
    + \tri{\nabla F^*(\beta) - v, \alpha - \beta}.
\end{align*}
Now set $\beta = \nabla F(v)$. Since $F$ is Legendre,
$\nabla F^* = (\nabla F)^{-1}$, so $\nabla F^*(\beta) = v$. The last
two inner-product terms vanish. By Fenchel duality,
$F(u)+F^*(\alpha) - \tri{u,\alpha} \ge 0$ and
$F(v)+F^*(\beta) - \tri{v,\beta} = 0$, completing the proof.
\end{proof}

\ignore{
\renewcommand{\contentsname}{Contents of Appendix}
\tableofcontents
\addtocontents{toc}{\protect\setcounter{tocdepth}{3}} 
\clearpage
}

\section{Proofs for Characterization Theorems}

\subsection{Approximation by quadratic generators}

\begin{lemma}[Local $C^1$-approximation by quadratic generators]
\label{lemma:local_c1_approx_corrected}
Let $F$ be a $C^2$ Legendre generator on an open convex domain
$\Omega \subset \Rset^d$. Let $K \subset \Omega$ be a compact convex
set. Fix $q_0 \in K$ and let
\[
  Q_0(q)
  = F(q_0) + \langle \nabla F(q_0), q - q_0 \rangle
  + \tfrac{1}{2}(q-q_0)^\top \nabla^2 F(q_0) (q-q_0).
\]
Then $Q_0$ is a quadratic Legendre generator. Moreover, for every
$\e > 0$ there exists $\delta>0$ (depending on $F$, $K$,
$q_0$, and $\e$) such that for all
$q \in B(q_0,\delta)\cap K$,
\[
  |F(q) - Q_0(q)| < \e, \qquad
  \norm*{\nabla F(q) - \nabla Q_0(q)} < \e.
\]
\end{lemma}

\begin{proof}
  Since $F\in C^2(\Omega)$ and is Legendre,
  $S_0 = \nabla^2 F(q_0) \succ 0$, so $Q_0$ is a quadratic Legendre
  generator.
Since $\nabla^2 F$ is continuous on the compact $K$, it is uniformly
continuous there. Let $\omega(\cdot)$ be a modulus of continuity for
$\nabla^2 F$ on $K$: for all $c,c'\in K$,
\[
  \norm*{\nabla^2 F(c) - \nabla^2 F(c')} \leq \omega(\norm*{c - c'}),
\]
with $\omega(t)\to 0$ as $t\downarrow 0$.
Fix $\e>0$. Since $\omega(t)\to 0$ as $t\downarrow 0$, it follows
that $\omega(t)t \to 0$ and $\omega(t)t^2 \to 0$.
Thus, we can choose $\delta>0$ so small that
\[
  \omega(\delta)\delta \leq \e \qquad \text{and} \qquad
  \tfrac{1}{2}\omega(\delta)\delta^2 \leq \e.
\]
Such a $\delta>0$ exists because both $\omega(t)t$ and $\omega(t)t^2$
tend to $0$ as $t \to 0$.
For any $q \in B(q_0, \delta) \cap K$ define $c_t = q_0 + t(q - q_0)\in K$
for $t \in [0, 1]$. By the integral form of the remainder,
\[
  \nabla F(q) - \nabla Q_0(q)
  = \int_0^1 \paren*{\nabla^2 F(c_t) - \nabla^2 F(q_0)} (q - q_0) \, dt.
\]
Hence, as $\norm*{c_t - q_0} \leq \norm*{q - q_0} < \delta$,
\[
  \norm{\nabla F(q) - \nabla Q_0(q)}
  \leq \int_0^1 \norm*{\nabla^2 F(c_t) - \nabla^2 F(q_0)} \norm{q - q_0} \, dt
  \leq \omega(\delta)\, \norm{q - q_0} < \omega(\delta)\, \delta \leq \e,
\]
using the first condition on $\delta$.
For the function value, Taylor expansion with integral remainder gives
\[
  F(q) - Q_0(q)
  = \int_0^1 (1 - t)\, (q - q_0)^\top
  \big(\nabla^2 F(c_t) - \nabla^2 F(q_0)\big) (q-q_0)\, dt.
\]
Thus, we have
\[
  |F(q) - Q_0(q)|
  \leq \int_0^1 (1-t) \norm*{\nabla^2 F(c_t) - \nabla^2 F(q_0)} \norm*{q - q_0}^2 \, dt
  \leq \tfrac{1}{2} \omega(\delta)\, \norm{q - q_0}^2
  < \tfrac{1}{2} \omega(\delta)\, \delta^2 \leq \e,
\]
using the second condition on $\delta$.
\end{proof}

\subsection{Localized Hessian perturbation}

\begin{lemma}[Localized Hessian Perturbation]
\label{lemma:localized_hessian_corrected}
Let $K \subset \Rset^d$ be compact. Let $F$ be $C^2$ on an open
neighborhood of $K$. Fix $q \in K$ and let $H$ be any symmetric
$d \times d$ matrix.  Then for every $\delta>0$, every $\alpha > 0$,
and every neighborhood $U$ of $q$, there exists a smooth function
$h \colon \Rset^d \to \Rset$ and a constant $C$ (depending only on the
choice of bump function) such that:
\begin{enumerate}
\item $\supp(h) \subset U$.
  
\item $\sup_{\pi\in K}\norm*{\nabla h(\pi)} \leq \delta$.
  
\item $\nabla^2 h(q) = \alpha H$.
  
\item $\sup_{\pi \in \Rset^d} \norm*{\nabla^2 h(\pi)} \leq C \alpha \norm*{H}$.

\end{enumerate}
Moreover, if $\nabla^2 F(\pi) \succeq \lambda I$ on $K$ for some
$\lambda>0$, then by choosing $\alpha$ sufficiently small (e.g.,
$\alpha < \lambda / (C \norm*{H})$), the perturbed generator $F+h$
remains Legendre (its Hessian stays positive definite on $K$).
\end{lemma}

\begin{proof}
  Let $\chi\colon\Rset^d\to[0,1]$ be a smooth cutoff (bump) function
  with $\chi(0)=1$, $\nabla\chi(0)=0$, and $\supp(\chi)\subset B(0,1)$
  (the unit ball).  Let
  $Q_H(\pi) = \frac{1}{2} (\pi-q)^\top H (\pi-q)$. Note that
  $\nabla Q_H(q) = 0$ and $\nabla^2 Q_H(\pi) = H$ for all $\pi$.
  Given $\alpha, \delta, U, H$, we must choose a support-scaling
  parameter $\e > 0$.  Define the scaled bump function
  $\chi_\e(\pi) = \chi\paren*{\frac{\pi-q}{\e}}$.  We construct
  our perturbation $h$ by multiplying the desired $C^2$-scaling
  $\alpha$ with the product of the bump and the quadratic form:
\[
  h(\pi) = \alpha \, \chi_\e(\pi) \, Q_H(\pi).
\]
We now verify the properties by choosing $\e$ appropriately.
The support of $h$ is contained in the support of $\chi_\e$, which is
$\{\pi \colon \norm{\pi - q} < \e\} = B(q, \e)$, this proves (1). We can
choose $\e > 0$ small enough such that $B(q, \e) \subset U$.
We compute the Hessian using the product rule.
\[
    \nabla^2 h(\pi) = \alpha \bracket*{(\nabla^2 \chi_\e) Q_H + (\nabla \chi_\e) (\nabla Q_H)^\top + (\nabla Q_H) (\nabla \chi_\e)^\top + \chi_\e (\nabla^2 Q_H)}.
\]
At $\pi = q$, we have $\chi_\e(q) = \chi(0) = 1$,
$\nabla\chi_\e(q) = \frac{1}{\e}\nabla\chi(0) = 0$, and $Q_H(q) = 0$.
Substituting these values:
\[
  \nabla^2 h(q)
  = \alpha \bracket*{(\nabla^2 \chi_\e(q)) \cdot 0 + 0 + 0 + 1 \cdot H} = \alpha H.
\]
This proves (3).

The terms in the Hessian are (with $z = (\pi-q)/\e$):
\begin{itemize}
\item $\nabla^2 \chi_\e = \frac{1}{\e^2} \nabla^2\chi(z)$. This is non-zero only for $\norm{\pi-q} < \e$.
  
\item $\nabla \chi_\e = \frac{1}{\e} \nabla\chi(z)$.
  
\item $Q_H(\pi) = \frac{1}{2}(\pi-q)^\top H (\pi-q)$. On the support, $\norm{Q_H} \leq \frac{1}{2}\norm{H} \e^2$.
  
\item $\nabla Q_H(\pi) = H(\pi-q)$. On the support, $\norm{\nabla Q_H}
  \leq \|H\| \e$.
\end{itemize}
The $C^2$-norm of $h$ is bounded by:
\[
  \norm{\nabla^2 h}_\infty
  \leq \alpha \norm*{\frac{1}{\e^2}(\nabla^2 \chi) Q_H
    + \frac{1}{\e}(\nabla \chi)(\nabla Q_H)^\top + \dots + \chi H}_\infty
\]
\[
  \leq \alpha \bracket*{\frac{C_{\chi,2}}{\e^2}(\tfrac{1}{2}\|H\|\e^2)
    + 2\frac{C_{\chi,1}}{\e}(\|H\|\e) + C_{\chi, 0}\|H\|}
    \leq \alpha \|H\| (C'_{\chi,2} + 2C'_{\chi, 1} + C'_{\chi, 0})
\]
where $C_{\chi,k}$ are the $\sup$ norms of the $k$-th derivatives of
$\chi$.  Thus,
$\sup \norm*{\nabla^2 h(\pi)} \leq C \alpha
\norm*{H}$ for $C = (C'_{\chi,2} + \dots)$, which is a
constant depending only on $\chi$. This proves (4).

We compute the gradient:
\[
  \nabla h(\pi) = \alpha \bracket*{(\nabla \chi_\e) Q_H
    + \chi_\e (\nabla Q_H)}
  = \alpha \bracket*{\frac{1}{\e}\nabla\chi(z) \cdot Q_H(\pi)
    + \chi(z) \cdot H(\pi-q)}
\]
Using the bounds from (4) for $z \in B(0,1)$:
\[
    \|\nabla h(\pi)\| \leq \alpha \bracket*{\frac{1}{\e} C_{\chi,1} (\tfrac{1}{2}\|H\|\e^2) + C_{\chi,0} (\|H\|\e)}
    = \alpha \norm*{H} \cdot \e \cdot (\tfrac{1}{2} C_{\chi,1} + C_{\chi,0})
\]
The gradient norm is
$\sup \|\nabla h\| \leq C' \alpha \norm*{H} \e$.  We are
given $\alpha$ and $\delta$. We can choose $\e$ small enough to
satisfy both the support condition ($B(q,\e)\subset U$) and the
gradient condition:
\[
    \e \leq \frac{\delta}{C' \alpha \norm*{H}}.
\]
(If $\alpha H = 0$, $h=0$ and all conditions are trivial). This proves (2).

Finally, for $F + h$ to be Legendre, we need $\nabla^2(F+h) \succ 0$.
We have $\nabla^2 F(\pi) \succeq \lambda I$. By (4),
$\|\nabla^2 h\|_\infty \leq C \alpha \norm*{H}$.  By
choosing $\alpha$ small enough such that
$C \alpha \norm*{H} < \lambda$, Weyl's inequality ensures
that $\nabla^2 F + \nabla^2 h \succ 0$ on $K$.
\end{proof}

\subsection{Localized approximation of Hessians}

\begin{lemma}[Localized approximation of Hessians]
\label{lem:lusin_refined}
Let $F$ be a Legendre generator with Hessian
$\nabla^2 F \colon \Delta \to \mathbb{S}^d_{+}$.  For every
$\eta > 0$, there exists a measurable partition $\{R_1,\dots,R_m\}$ of
a subset $R \subset \Delta$ with $\mu(\Delta \setminus R) < \eta$ such
that for each $j$ there exists a symmetric matrix
$M_j \in \mathbb{S}^d_{+}$ satisfying
$$
 \sup_{\pi \in R_j} \, \|\nabla^2 F(\pi) - M_j\| \leq \eta.
$$
In other words, up to a set of measure $< \eta$, the Hessian $\nabla^2 F$
can be approximated within $\eta$ by a finite-valued, piecewise-constant
matrix field.
\end{lemma}

\begin{proof}
By Lusin's theorem, for the given $\eta > 0$ there exists a compact set
$K \subset \Delta$ with $\mu(\Delta \setminus K) < \eta/2$ on which
$\nabla^2 F$ is continuous.
Since $K$ is compact and $\nabla^2 F$ is continuous on $K$, it is
uniformly continuous there. Thus, there exists $\delta > 0$ such that
for all $\pi,\pi' \in K$ with $\norm*{\pi - \pi'} \leq \delta$, we have
$\|\nabla^2 F(\pi) - \nabla^2 F(\pi')\| \leq \eta$.

Cover $K$ by finitely many Borel sets $R_1,\dots,R_m$ of diameter at
most $\delta$. For each $R_j$, fix a representative point
$\pi_j \in R_j$ and define $M_j = \nabla^2 F(\pi_j)$. By construction,
for all $\pi \in R_j$ we have
$\|\nabla^2 F(\pi) - M_j\| \leq \eta$. Setting
$R = \bigcup_{j=1}^m R_j$, we have
$\mu(\Delta \setminus R) \leq \mu(\Delta \setminus K) < \eta/2 < \eta$.

This yields the desired partition with piecewise-constant approximation.
\end{proof}

\end{document}